\documentclass[final,1p,times,authoryear]{elsarticle}

\usepackage{amssymb}
\usepackage{amsthm}

\usepackage{listings}
\usepackage{times}
\usepackage{helvet}
\usepackage{courier}
\usepackage{graphicx}
\usepackage{amsmath}
\usepackage{amsfonts}
\usepackage{mathrsfs}
\usepackage{tabularx}
\usepackage{epsfig}
\usepackage{fleqn}
\usepackage{MathLetters}
\usepackage{MyMacros}
\usepackage{natbib}
\usepackage{subfigure}
\usepackage{fancybox}
\usepackage{paralist}
\usepackage{wrapfig}
\usepackage{url}
\usepackage{pbox}
\usepackage{clrscode}
\usepackage{color}

\journal{Artificial Intelligence}

\begin{document}

\begin{frontmatter}



\title{Extracting Lifted Mutual Exclusion Invariants\\ from Temporal Planning Domains}

\author[SB]{Sara Bernardini}
\address[SB]{Department of Computer Science, Royal Holloway University of London, \\
Egham, Surrey, TW20 0EX, UK}
\ead{sara.bernardini@kcl.ac.uk}
\author[FF]{Fabio Fagnani}
\address[FF]{Department of Mathematical Sciences (DISMA), Politecnico di Torino \\
Corso Duca degli Abruzzi, 24, 10129 Torino, Italy}
\ead{fabio.fagnani@polito.it}
\author[DS]{David E. Smith}
\address[DS]{Intelligent Systems Division, NASA Ames Research Center, Moffett Field, CA 94035}
\ead{david.smith@nasa.gov}

\begin{abstract}
We present a technique for automatically extracting mutual exclusion invariants from temporal planning instances. It first identifies a set of invariant templates by inspecting the lifted representation of the domain and then checks these templates against properties that assure invariance. Our technique builds on other approaches to invariant synthesis presented in the literature, but departs from their limited focus on instantaneous actions by addressing temporal domains. To deal with time, we formulate invariance conditions that account for the entire structure of the actions and the possible concurrent interactions between them. As a result, we construct a significantly more comprehensive technique than previous methods, which is able to find not only invariants for temporal domains, but also a broader set of invariants for non-temporal domains. The experimental results reported in this paper provide evidence that identifying a broader set of invariants results in the generation of fewer multi-valued state variables with larger domains. We show that, in turn, this reduction in the number of variables reflects positively on the performance of a number of temporal planners that use a variable/value representation by significantly reducing their running time.
\end{abstract}


\end{frontmatter}


\newcommand{\centergraphics}[1]{\begin{center} \includegraphics#1 \end{center}}

\newcommand{\st}{[}
\newcommand{\ovv}{()}
\newcommand{\en}{]}
\newcommand{\sto}{[)}
\newcommand{\ove}{(]}
\newcommand{\ste}{\st\en}

\newcommand{\Os}[1]{\mbox{$#1_{\st}$}}
\newcommand{\Oo}[1]{\mbox{$#1_{\ovv}$}}
\newcommand{\Oe}[1]{\mbox{$#1_{\en}$}}

\newcommand{\Oso}[1]{\mbox{$#1_{\sto}$}}
\newcommand{\Ooe}[1]{\mbox{$#1_{\ove}$}}
\newcommand{\Ose}[1]{\mbox{$#1_{\ste}$}}

\newcommand{\Ev}[4]{\mbox{$\mbox{\it #1}_{#2}^{#3}(#4)$}}

\newcommand{\Ct}[3]{\Ev{cond}{#1}{#2}{#3}}
\newcommand{\Et}[3]{\Ev{eff}{#1}{#2}{#3}}
\newcommand{\pre}[1]{\Ev{pre}{}{}{#1}}
\newcommand{\eff}[1]{\Et{}{}{#1}}

%
\newcommand{\Cs}[1]{\Ct{\st}{}{#1}}
\newcommand{\Co}[1]{\Ct{\ov}{}{#1}}
\newcommand{\Ce}[1]{\Ct{\en}{}{#1}}
\newcommand{\Cso}[1]{\Ct{\sto}{}{#1}}
\newcommand{\Coe}[1]{\Ct{\ove}{}{#1}}
\newcommand{\Es}[1]{\Et{\st}{}{#1}}
\newcommand{\Ee}[1]{\Et{\en}{}{#1}}

\newcommand{\Mset}{\text{mutex}_\text{set}}

\newtheorem{theorem}{Theorem}
\newtheorem{definition}[theorem]{Definition}
\newtheorem{proposition}[theorem]{Proposition}
\newtheorem{lemma}[theorem]{Lemma}
\newtheorem{corollary}[theorem]{Corollary}
\newtheorem{example}{Example}
\theoremstyle{remark}
\newtheorem{remark}[theorem]{Remark}

\section{Introduction}
\label{sec:intro}
This paper presents a technique for synthesising mutual exclusion invariants from temporal planning domains expressed in PDDL2.1 \citep{FoxLong-03}. A mutual exclusion invariant over a set of ground atoms means that at most one atom in the set is true at any given moment. A set with this property can intuitively be seen as the domain of a multi-valued state variable.\footnote{To be precise, a set of mutually exclusive atoms is the domain of an implicit state variable only when augmented with a catch-all null value, which can be manufactured on demand. This is because a mutual exclusion invariant encodes the concept of ``at most one'', whereas a state variable encodes the concept of ``at least one''. The null value is then used when no atom is true, if these situations exist.} For instance, consider the $\sfs{Floortile}$ domain from the 8th International Planning Competition (IPC8 - see Appendix A)\footnote{We call IPC3-IPC8 the competitions for temporal planning held in 2002, 2004, 2006, 2008, 2011 and 2014.}. A mutual exclusion invariant for this domain states that two ground atoms that indicate the position of a robot can never be true at the same time. Intuitively, this means that a robot cannot be at two different locations simultaneously. To give a concrete case, consider a planning problem for the $\sfs{Floortile}$ domain with one robot \verb+r1+ and three locations, \verb+t1+, \verb+t2+ and \verb+t3+. We can create a state variable that indicates the position of \verb+r1+ with a domain of three values: \verb+robot-at_r1_t1+, \verb+robot-at_r1_t2+ and \verb+robot-at_r1_t3+.

Although a number of approaches to invariant synthesis have been proposed so far \cite[]{GereviniSchubert-00,Rintanen-00,Rintanen-08,FoxLong-98,Helmert-09}, they are limited in scope because they deal with non-temporal domains only. Recently, \cite{Rintanen-14} has proposed a technique for temporal domains, but this technique does not scale to complex problems because it requires to ground the domain. Our approach solves both these problems at the same time. We find invariants for temporal domains by applying an algorithm that works at the lifted level of the representation and, in consequence, is very efficient and scales to large instances.

Our invariant synthesis builds on \cite{Helmert-09}, which presents a technique to translate the non-temporal subset of PDDL2.2 \citep{Edelkamp-04} into FDR, a multi-valued planning task formalism used by Fast Downward \citep{helmert-jair-2006}. Since finding invariants for temporal tasks is much more complex than for tasks with instantaneous actions, a simple generalisation of Helmert's technique to temporal settings does not work. In the temporal case, simultaneity and interference between concurrent actions can occur, hence our algorithm cannot check actions individually against the invariance conditions, but needs to consider the entire set of actions and their possible intertwinements over time. In capturing the temporal case, we formulate invariance conditions that take into account the entire structure of the action schemas as well as the possible interactions between them. As a result, we construct a significantly more comprehensive technique that is able to find not only invariants for temporal domains, but also a broader set of invariants for non-temporal domains. 

Our technique is based on a two-steps approach. First, we provide a general theory at the ground level and propose results that insure invariance under two types of properties: safety conditions for individual instantaneous and durative actions as well as collective conditions that prevent dangerous intertwinements between durative actions. Then, we lift these results at the level of schemas so that all checks needed for verifying invariance can be performed at this higher level, without the need of grounding the domain. Complexity of such checks are of linear or low polynomial order in terms of the number of schemas and literals appearing in the domain.


\subsection{Motivations}
Automated planning is a well-established field of artificial intelligence and, over more than fifty years since its appearance, several paradigms have emerged. One fundamental difference between these paradigms is whether time is treated implicitly or explicitly. While classical planning focuses on the causal evolution of the world, temporal planning is concerned with the temporal properties of the world. In classical planning, actions are considered to be instantaneous, whereas in temporal planning actions have durations and can be executed concurrently. 
Another important difference between planning paradigms relates to whether the world is modelled by adopting a Boolean propositional representation or a representation based on multi-valued state variables. Traditionally, the majority of the work in planning has been devoted to classical planning with domains expressed through propositional languages, and in particular PDDL \citep{McDermott-00} and its successors \citep{FoxLong-03}, the language of the IPC. However, in parallel with the development of classical propositional planning, a number of temporal planning systems have been proposed for coping with practical problems, especially space mission operations \citep{FrankJonsson-03,Chien-00,GhallabLaruelle-94,Muscettola-94,Cesta-08}. Typically, these systems use variable/value representations. 
Table \ref{table:paradigms} shows a classification of several well-known planners based on these different characteristics. 

\begin{table}
\begin{center}
\footnotesize
\begin{tabular}{c|l|l|}
\multicolumn{1}{r}{}
 &  \multicolumn{1}{c}{Propositional}
 & \multicolumn{1}{c}{Variable/Value} \\
\cline{2-3}
Classical & \begin{tabular}[c]{@{}l@{}} HSP \\ \citep{BonetGeffner01}\\FF \\ \citep{Hoffmann01FF}\\YAHSP \\ \citep{vidal-04}\end{tabular} &
\begin{tabular}[c]{@{}l@{}}FD \\ \citep{helmert-jair-2006}\\LAMA \\ \citep{Richter-10}\end{tabular} \\
\cline{2-3}
Temporal & \begin{tabular}[l]{@{}l@{}}LPG \\ \citep{Gerevini-06}\\POPF \\ \citep{coles-10}\end{tabular} & 
\begin{tabular}[l]{@{}l@{}}TFD \\ \citep{Eyerich-09}\\EUROPA2 \\ \citep{FrankJonsson-03}\end{tabular} \\
\cline{2-3}
\end{tabular}
\caption{Examples of planners and their classification based on whether they treat time explicitly or implicitly and whether they use a Boolean propositional representation or a multi-valued state variable representation.}
\label{table:paradigms}
\end{center}
\end{table}

\normalsize

Recently, a few techniques have been proposed for translating propositional representations into variable/value representations \citep{helmert-jair-2006, Bernardini-08-trans, Rintanen-14}. A central task of all these techniques is the generation of state variables from propositions and actions. The basic procedure to do this (which we use as the baseline in our experiments) relies on generating one state variable with two values, true and false, for each proposition. Naturally, such translation produces no performance advantage. A more sophisticated strategy, which produces compact and optimised encodings, rests on extracting mutual exclusion invariants from propositional domains and using such invariants to generate multi-valued state variables. This is the focus on our work.

These translation techniques are important as they allow fair testing of planners developed for variable/value representations on PDDL benchmarks (which are propositional). The practical issue is that planners that permit variable/value representation need this feature to be well-exploited and perform competitively. Since translation between the two different representations can be cheaply automated, there is no reason to reject providing the richer representations to those planners that accept them (if the translation was expensive, one might reasonably argue about the fairness of this process). In consequence, these translating techniques are extremely useful for comparing alternative planning paradigms and for promoting cross-fertilisation of ideas between different planning communities, which is our primary motivation.

However, the importance of these translation techniques goes beyond the engineering of a bridge between different input languages. In transforming propositional representations into state variable representations, they generate new domain knowledge, where \emph{new} means \emph{accessible} in this context. Effectively, these techniques turn into internal mini theorem provers since, rather than merely translating, they firstly selectively explore the deductive closure of the original theory to find theorems that permit optimising the representation and secondly execute those optimisations.\footnote{One might argue that \emph{optimising} is a more technically precise term than \emph{translating}} We will show that the cost of performing these optimisations is worth it because it is very fast and can be amortised over many problems.

Mutual exclusion invariants are also beneficial in pruning the search space for search methods such as symbolic techniques based on SAT \citep{KautzSelman-99,Huang-10} and backward chaining search \citep{BluFur97}. In addition, as invariants help to reduce the number of variables required to encode a domain, they are used in planning systems based on binary decision diagrams (BDDs) \citep{Edelkamp-2001}, constraint programming \citep{Do-2001}, causal graph heuristics \citep{helmert-jair-2006} and pattern databases \citep{Haslum-07}.

Finally, from a knowledge engineering perspective, the invariant synthesis presented in this paper can be used as a powerful tool for debugging temporal planning domains, although we do not focus on this specifically in this paper. As shown in \cite{Cushing-07}, several temporal planning tasks developed for the various IPC editions are bugged with the consequence that the planners take a long time to solve them, when they actually manage to do so. As invariants capture intuitive properties of the physical systems described in the domains, it is easy for a domain expert to identify modelling mistakes by inspecting them. Discrepancies between the invariants found by the automatic technique and those that the expert expects to see for a given domain indicate that the domain does not encode the physical system correctly. In consequence, the expert can revise the domain and repair it. For example, considering the $\sfs{Rover}$ domain, we expect that a \verb+store+ could be either full or empty at any time point. However, the invariant synthesis does not produce an invariant with the atoms \verb+full+ and \verb+empty+. It can be shown that this is because the action \verb+drop+ is not properly modelled. Our technique not only alerts the expert that the system is not properly modelled, but also refers the expert to the action that is not encoded correctly. This is a useful feature to fix modelling errors quickly and safely.

\subsection{Contributions of the paper}
In brief, the contributions of this paper are the following. 

From a theoretical point of view:
\begin{itemize}
\item We give the first formal account of a mutual exclusion invariant synthesis for temporal domains that works at the lifted level of the representation. Our presentation of this topic is rigorous and comprehensive and our theory is general and not tailored around IPC domains as with related techniques.
\item Our technique is based on inferring general properties of the state space by studying the structure of the action schemas and the lifted relations in the domain, without the need to ground it.
This is generally an hard task. Our theoretical framework is sophisticated, but it results in practical tools that have high efficiency and low computational cost.
\end{itemize}

From a practical point of view:
\begin{itemize}
\item We provide a tool for optimising the generation of state variables from propositions and actions. This results in compact encodings that benefit the performance of planners, as we will show in our experimental results (see Section \ref{sec:experiments}).
\item We offer a technique that can be used as a debugging tool for temporal planning domains. As this type of domains are particularly challenging to encode, especially when large and complex, a rigorous debugging process is crucial in producing correct representations of the systems under consideration.
\end{itemize}

\subsection{Organisation of the paper}
This paper is organised as follows. After presenting PDDL and PDDL2.1, our input languages, in Section \ref{sec:pddl}, we formally introduce the notion of invariance in Section \ref{sec:templates}. 

Sections \ref{sec:safeness inst actions}, \ref{sec:safeness sequences} and \ref{sec: cond inv} are devoted to a detailed analysis of actions at the ground level. In particular, Section \ref{sec:safeness inst actions} focuses on instantaneous actions: the fundamental concept of \emph{strong safety} is introduced and analysed and a first sufficient result for invariance, Corollary \ref{cor:inv}, is established. Section \ref{sec:safeness sequences} analyses \emph{sequences} of actions and, in particular, durative actions (seen as a sequence of three instantaneous actions) for which two new concepts of safety are proposed and investigated: \emph{individual} and \emph{simple safety}. Our main technical results are presented in Section \ref{sec: cond inv} and consist of Theorem \ref{theo: *safety}, Theorem \ref{theo: non intertwining} and Corollary \ref{cor: type a}: these results insure invariance under milder safety requirements on the durative actions than Corollary \ref{cor:inv}. This is obtained by adding requirements that prevent the intertwinement of durative actions that are not strongly safe. 

Sections \ref{sec:action} and \ref{sec: dur schemas} are devoted to \emph{lift} the concepts and results obtained in the previous sections to the level of action \emph{schemas}. In particular, Section \ref{sec:action} deals with the problem of lifting the concept of strong safety for instantaneous schemas, while Section \ref{sec: dur schemas} considers durative action schemas and presents the lifted version of our main results, Corollaries \ref{cor: *safety schema}, \ref{cor: non intertwining schemas} and  \ref{cor: type a schemas}. 

These results are the basic ingredients of our algorithm to find invariants, which is proposed in Section \ref{sec:algorithm}. Section \ref{sec:experiments} reports an extensive experimental evaluation of our approach against the domains of the last three IPCs. Sections \ref{sec:related} and \ref{sec:conclusions} conclude the paper with a description of related works and closing remarks. There are three appendices: Appendix A and B contain the specifications of the planning domains used as the running examples in the paper; Appendix C contains all the technical proofs.

\section{Canonical Form of Planning Tasks}
\label{sec:pddl}
In this work, we consider planning instances that are expressed in PDDL2.1 \citep{FoxLong-03}. However, before applying our algorithm to find invariants, we manipulate the domain to enforce a regular structure in the specification of the action schemas. In what follows, we first give an overview of this \emph{canonical} form that we use and then describe how such a form can be obtained starting from a domain expressed in PDDL2.1.

\subsection{PDDL Canonical Form}
\label{sec:canonical}
A planning \emph{instance} is a tuple $\cI = (\cD, \cP)$, where $\cD$ is a planning \emph{domain} and $\cP$ a planning \emph{problem}. The domain $\cD = (\cR, \cA^{i}, \cA^{d}, arity)$ is a 4-tuple consisting of finite sets of relation symbols, instantaneous actions, durative actions, and a function $arity$ mapping all of these symbols to their respective arities. $\cP = (\cO, Init, \cG)$ is a triple consisting of the objects in the domain, the initial logical state and the goal logical state specifications. 
 
The \emph{ground atoms} of the planning instance, $Atms$, are the (finitely many) expressions formed by applying the relations in $\cR$ to the objects in $\cO$ (respecting arities). A \emph{logical state} is any subset of $Atms$ and  $\cS = 2^{Atms}$ denotes the set of logical states. The initial state $Init\in\cS$ and the goal $\cG\subseteq \cS$ is a subset of logical states (typically defined as a conjunction of literals).
A \emph{state} is a tuple in $\bR \times \cS$, where the first value is the time of the state and the second value (\emph{logical state}) is a subset of $Atms$. The initial state for $\cI$ is of the form $(t_0, Init)$.

The set $\cA^i$ is a collection of \emph{instantaneous} action schemas. An instantaneous action schema $\alpha$ is composed of the following sets: 
\begin{itemize}
\item $V_{\alpha}$, i.e. the free schema's arguments;
\item $Pre_{\alpha} = Pre^+_{\alpha}  \cup Pre^-_{\alpha}$, where $Pre^+_{\alpha}$ are the positive preconditions and $Pre^-_{\alpha}$ the negative preconditions;
\item $Eff_{\alpha} = Eff^+_{\alpha} \cup Eff^-_{\alpha}$, where $Eff^+_{\alpha}$ are the add effects and $Eff^-_{\alpha}$ the delete effects. 
\end{itemize}
Preconditions and effects are sets of literals $l$ of the form: $\forall v_{1}, \ldots, v_{k}: q$ where $q$ is a non-quantified positive literal 
of the form $r(v_1', \ldots, v'_{n})$, where $r \in \cR$, $arity(r) = n$, $\{ v_{1}, \ldots, v_{k}\} \subseteq \{v_1', \ldots, v'_{n}\}$ is the set of quantified arguments, $\{v'_1, \ldots, v'_{n}\}\setminus \{ v_{1}, \ldots, v_{k}\} \subseteq V_{\alpha}$ is the set of free arguments. The universal quantification can be trivial  (i.e. quantification over zero arguments) and, in this case, it is omitted. 
We indicate 
the set of the positions of the free and the quantified arguments, respectively, as $\vars[l]$ and $\varsq[l]$, and the pair $\langle r, a \rangle$, where $r$ is the relation symbol that appears in literal $q$ and $a$ is its arity, as $\rela[l]$. Given a position $i$, we indicate the corresponding argument as $\args[i, l]$.

The set $\cA^d$ is a collection of \emph{durative} action schemas. A durative action schema $D\alpha$ is a triple of instantaneous action schemas $D\alpha=(\alpha^{st}, \alpha^{inv}, \alpha^{end})$ such that $V_{\alpha^{st}} = V_{\alpha^{inv}} = V_{\alpha_{end}}$ (this common set is denoted $V_{D\alpha}$). We indicate as $\{ D\alpha \}$ the set $\{ \alpha^{st}, \alpha^{inv}, \alpha^{end}\}$.

We call $\cA$ the set of all the instantaneous actions schemas in the domain, including those induced by durative actions: $\cA = \cA^i \cup \bigcup\limits_{D\alpha \in \cA^d} \{ D\alpha \}$. Given any two action schemas $\alpha_1$ and $\alpha_2$ in $\cA$ such that it does not exist a durative action $D\alpha$ with both $\alpha_1$ and $\alpha_2$ in $\{ D\alpha \}$, we assume that the free arguments of $\alpha_1$ and $\alpha_2$ are disjoint sets, i.e. $V_{\alpha_{1}} \cap V_{\alpha_{2}} = \emptyset$.

Given an action schema $\alpha \in \cA^i$ with free arguments $V_{\alpha}$,
consider an injective \emph{grounding} function $gr: V_{\alpha} \rightarrow \cO$ that maps the free arguments in $\alpha$ to objects $\cO$ of the problem. The function $gr$ induces a function on the literals in $\alpha$ as follows. Given a literal $l$ that appears in $\alpha$, we call $\tilde{gr}(l)$ the literal that is obtained from $l$ by grounding its free arguments according to $gr$ and $gr(l)$ the set of ground atoms obtained from $\tilde{gr}(l)$ by substituting objects in $\cO$ for each quantified argument in $l$ in all possible ways. Note that, when there are no quantified arguments, $\tilde{gr}(l) = gr(l)$ and they are singletons. Given a set $L$ containing literals $l_1, \ldots l_n$, we call $\tilde{gr}(L) = \tilde{gr}(l_1) \cup \ldots \cup \tilde{gr}(l_n)$ and $gr(L) = gr(l_1) \cup \ldots \cup  gr(l_n)$. We call $\tilde{gr}(\alpha)$ the action schema obtained from $\alpha$ by grounding each literal $l$ that appears in $\alpha$ according to $gr$ and $gr(\alpha)$ the ground action that is obtained from $\tilde{gr}(\alpha)$ by replacing the quantified arguments with the set of ground atoms formed by substituting objects in $\cO$ for the quantified arguments in all possible ways. 

Given a durative action schema $D\alpha \in \cA^d$ and a grounding function $gr$, the ground durative action $gr(D\alpha)$ is obtained by applying $gr$ to the instantaneous fragments of $D\alpha$: $gr(D\alpha)=(
gr(\alpha^{st}), gr(\alpha^{inv}), gr(\alpha^{end}))$. Note that we cannot apply different grounding functions to different parts of a durative action schema. 

Given a ground action $a$, we indicate its positive and negative preconditions as $Pre^{\pm}_{a}$ and its add and delete effects as $Eff^{\pm}_{a}$. We call $\cG\cA^i$, $\cG\cA^d$, respectively, the set of instantaneous, and durative ground actions. Finally, we call $\cG\cA$ the set of all ground actions in $\cI$ (obtained from grounding all schemas in $\cA$).

A ground action $a$ is \emph{applicable} in a logical state $s$ if $Pre^+_a \subseteq s$ and $Pre^-_a \cap s = \emptyset$. The \emph{result} of applying $a$ in $s$ is the state $s'$ such that $s' = (s \setminus Eff^-_a) \cup Eff^+_a$. We call $\xi$ this transition function: $s' = \xi(s,a)$. 

The transition function $\xi$ can be generalised to a set of ground actions $A=\{a_1, \ldots,$ $a_n\}$ to be executed concurrently: $s' = \xi(s, A)$. However, in order to handle concurrent actions, we need to introduce the so-called \emph{no moving targets} rule: no two actions can simultaneously make use of a value if one of the two is accessing the value to update it. The value is a \emph{moving target} for the other action to access. This rule avoids conflicting effects, but also applies to the preconditions of an action: no concurrent actions can affect the parts of the state relevant to the precondition tests of other actions in the set (regardless of whether those effects might be harmful or not). In formula, two ground actions $a$ and $b$ are \emph{non-interfering} if: $$Pre_a \cap (Eff^+_b \cup Eff^-_b) = Pre_b \cap (Eff^+_a \cup Eff^-_a) = \emptyset$$ $$Eff^+_a \cap Eff^-_b = Eff^+_b \cap Eff^-_a = \emptyset$$ If two actions are not non-interfering, they are \emph{mutex}.

In this work, whenever we consider a set of concurrent actions $A=\{a_1, \ldots, a_n\}$, we implicitly assume that the component actions are pairwise non-interfering. In this case, 
given a state $s$ such that each $a_i\in A$ is applicable in $s$, the transition function $s' = \xi(s, A)$ is defined as follows: $$s' = (s \setminus \bigcup_{a \in A} Eff^-_a) \cup \bigcup_{a \in A} Eff^+_a$$
The following useful result shows that the application of a set of actions can always be serialised.

\begin{proposition}[Serialisability] 
\label{prop: ser} Given a set of actions $A=\{a_1, \ldots, a_n\}$ and a state $s$ in which $A$ is applicable, consider the sequence of states recursively defined as  $s_0 = s$ and $s_k = \xi(s_{k-1}, a_k)$ for $k=1, \dots , n$. Then,
\begin{enumerate}[(i)]
\item The sequence $s_k$ is well defined: $a_k$ is applicable in $s_k$ for every $k=1, \dots , n$;
\item $s_n= \xi(s, A)$.
\end{enumerate}
\end{proposition}

An \emph{instantaneous timed action} has the following syntactic form: $(t, a)$, where $t$ is a positive rational number in floating point syntax and $a$ is a ground instantaneous action. A \emph{durative timed action} has the following syntactic form: $(t, Da[t'])$, where $t$ is a rational valued time, $Da$ is a ground durative action and $t'$ is a non-negative rational-valued duration. It is possible for multiple timed actions to be given the same time stamp, indicating that they should be executed concurrently.

A \emph{simple plan} $\pi$ for an instance $\cI$ is a finite collection of instantaneous timed actions and a \emph{plan} $\Pi$ consists of a finite collection of (instantaneous and durative) timed actions. The \emph{happening time sequence} $\{t_i\}_{i = 0, \ldots ,\bar k}$ for a plan $\Pi$ is: 
$$\{ t_{0} \} \cup \{ t | (t, a) \in \Pi \ or \ (t, Da[t']) \in \Pi \ or \ (t - t', Da[t']) \in \Pi \}$$ 
Note that the last disjunct allows the time corresponding to the end of execution of a durative action to be included as a happening time.

Given a plan $\Pi$, the \emph{induced} simple plan for $\Pi$ is the set of pairs $\pi$ containing:
\begin{itemize}[(i)]
\item $(t, a)$ for each $(t, a) \in \Pi$, where $a$ is an instantaneous ground action;
\item $(t, a^{st})$ and $(t + t', a^{end})$ for all pairs $(t, Da[t']) \in \Pi$, where $Da$ is a durative ground action; and
\item $((t_i + t_{i+1})/2, a^{inv})$ for each pair $(t, Da[t']) \in \Pi$ and for each $i$ such that $t \leq t_i < (t + t')$, where $t_i$ and $t_{i+1}$ are in the happening sequence for $\Pi$.
\end{itemize}

The \emph{happening} at time $t$ of the plan $\pi$ is defined as $A_t=\{a\in\cG\cA\,|\, \exists t\; (t,a)\in\pi\}$. Note that in $\pi$ we have formally lost the coupling among the start and end fragments of durative actions. Since in certain cases this information is necessary, we set a definition: a durative action $Da$ is said to happen in $\pi$ in the time interval $[t, t+t']$ whenever this holds true in the original plan $\Pi$, namely when $(t, Da[t'])\in\Pi$.

A simple plan $\pi$ for a planning instance $\cI$ is \emph{executable} if it defines a happening sequence $\{ t_i \}_{i = 0 \ldots \bar{k}}$ and there is a sequence of logical states $\{ s_i \}_{i = 0 \ldots \bar{k}}$ such that $s_0 = Init$ and for each $i = 0, \ldots ,\bar{k}$, $s_{i+1}$ is the result of executing the happening at time $t_i$ in $\pi$. Formally, we have that $A_{t_{i+1}}$ is applicable in $s_{i}$ and $s_{i+1} = \xi(s_i, A_{t_{i+1}})$. The state $s_{\bar{k}}$ is called the final logical state produced by $\pi$. The sequence of times and states $\{S_i = (t_i, s_i)_{i = 0 \ldots \bar{k}} \}$ is called the (unique) trace of $\pi$, $trace(\pi)$. Two simple plans are said to be \emph{equivalent} if they give rise to the same trace.  

The following result holds, from the definition of mutex, induced plan and executability:
\begin{remark}
\label{rem:mut}
Given a ground durative action $Da=(a^{st}, a^{inv}, a^{end})$ and a ground instantaneous action $a'$, if $a'$ and $a^{inv}$ are mutex, then there is no executable simple plan that contains the timed actions $(t_1, a^{st})$, $(t_2, a')$ and $(t_3, a^{end})$ with $t_1 < t_2 <t_3$.
\end{remark}

A simple plan for a planning instance $\cI$ is \emph{valid} if it is executable and produces a final state $s_{\bar{k}}\in\cG$.

We call $Plans$ all the valid (induced and not) simple plans for $\cI$ and $\cS_r$ the union of all the logical states that appear in the traces associated with the plans in $Plans$: $\cS_r = \{ s \ | \ \exists \ \pi \in \ Plans \ \textrm{and} \ (t, s) \in trace(\pi)\}$. Note that $\cS_r \subseteq \cS$. We call the states in $\cS_r$ \emph{reachable states}. 

\subsection{From PDDL2.1 to Canonical PDDL}
We build the canonical form described above starting from PDDL2.1 instances, which are characterised by metric and temporal information \citep{FoxLong-03}. Numeric variables can be seen as already in the variable/value form and so we do not handle them. We could potentially exploit metric information in order to find additional state variables, but currently we do not do that. Instead, we assume that numeric variables are already in the right form and ignore them and numeric constraints when we look for logical state variables. 

Temporal information are handled in PDDL2.1 by means of \emph{durative} actions. They can be either discretised or continuous, but we focus on discretised durative actions only here. They have a duration field and temporally annotated conditions and effects. The duration field contains temporal constraints involving terms composed of arithmetic expressions and the dedicated variable $duration$. The annotation of a condition makes explicit whether the associated proposition must hold at the \emph{start} of the interval (the point at which the action is applied), the \emph{end} of the interval (the point at which the final effects are asserted) or \emph{over all} the interval (open at both ends) from the start to the end (invariant over the duration of the action). The annotation of an effect makes explicit whether the effect is immediate (it happens at the start of the interval) or delayed (it happens at the end of the interval). No other time points are accessible. Logical changes are considered to be instantaneous and can only happen at the accessible points. To build our canonical form, we transform durative actions into triples of instantaneous actions. We do this in such a way that we do not change the set of plans that can be obtained for any goal specification. Plans with durative actions, in fact, are always given a semantics in terms of the semantics of simple plans \citep{FoxLong-03}, as explained in the previous section.

Let us see now in more detail how we obtain the PDDL canonical form from PDDL2.1 instances. 

A PDDL2.1 instance looks the same as a canonical instance, except for the set of action schemas in the domain. In particular, in a PDDL2.1 domain, in place of the sets $\cA^i$ and $\cA^d$, we find a set $\cA^a$ that contains both instantaneous and durative action schemas, which have the following characteristics. Durative action schemas have temporally annotated conditions and effects, which we indicate as $Pre^{x}$ and $Eff^{x}$, where $x$ is in the set $\{st, inv, end \}$. Given an action schema in $\cA^a$ (durative or not), the condition formula can be a relation, a negation, a conjunction or disjunction of relations or a quantified formula on relations. The effect formula can be a relation, a negation or a conjunction of relations, a universally quantified formula on relations or a conditional effect formula, which is a tuple formed by a precondition formula and and effect formula. We manipulate the action schemas in $\cA^a$ to obtain $\cA^i$ and $\cA^d$, where each action schema in these sets has the canonical form described in Section \ref{sec:canonical}. 

First, we eliminate conditional effects and existentially quantified formulae through an operation referred to as \emph{flattening} (see \cite{FoxLong-03} for details). Since these features are syntactic sugar, they can be eliminated by applying simple syntactic transformations. The resulting schemas are equivalent to the original ones. 

Given a flatten action schema $\alpha$, we take the formulas (temporally annotated or not) in its conditions and effects and \emph{normalise} them by using the algorithm introduced by \cite{Helmert-09} (we refer the interested reader to this paper for a full description of the normalisation process\footnote{Our normalisation differs from  \cite{Helmert-09} only in that we preliminarily eliminate conditional effects by applying the flattening operation before normalisation and we keep universal quantification in the preconditions. We also apply normalisation not only to formulas appearing in instantaneous actions as in  \cite{Helmert-09}, but also to temporally annotated formulas in durative actions. We normalise the formulas and leave the temporal annotation unchanged.}). After normalisation, all action schema conditions and effects become sets of universally quantified first-order literals $l$ of the form $\forall v_{1}, \ldots, v_{k}: q$, where $q$ is a non-quantified literal and the universal quantification can be trivial. We indicate by $Pre_{\alpha}^{+}$ and $Eff_{\alpha}^{+}$ the set of positive literals that appear positive in $\alpha$ and by $Pre_{\alpha}^{-}$ and $Eff_{\alpha}^{-}$ the set of positive literals that appear negative in $\alpha$.

Note that we consider illegal durative action schemas $D\alpha$ such that it exists a literal $l$ that satisfies one of the following conditions: 
\begin{itemize}
\item $l \in Pre_{D\alpha}^{-inv}$ and $l \in (Pre_{D\alpha}^{+st} \setminus Eff_{D\alpha}^{-st}) \cup Eff_{D\alpha}^{+st}$; 
\item $l \in Pre_{D\alpha}^{+inv}$ and $l \in (Pre_{D\alpha}^{-st} \setminus Eff_{D\alpha}^{+st}) \cup Eff_{D\alpha}^{-st}$; 
\item $l \in Pre_{D\alpha}^{+inv}$ and $l \in Pre_{D\alpha}^{-end}$;
\item $l \in Pre_{D\alpha}^{-inv}$ and $l \in Pre_{D\alpha}^{+end}$
\end{itemize}

\noindent We assume that $\cA^a$ contains no durative action schemas of such types.

After flattening and normalisation, we transform the durative action schemas in $\cA^a$ in triples of instantaneous action schemas. Given a durative action $D\alpha \in A^a$, we create two instantaneous action schemas that correspond to the end points of $D\alpha$, $\alpha^{st}$ and $\alpha^{end}$, and one that corresponds to the invariant conditions that must hold over that duration of $D\alpha$, $\alpha^{inv}$. 
More formally, given a durative action schema $D\alpha$ we create $\alpha^{st}$, $\alpha^{inv}$  and $\alpha^{end}$ as follows:

\begin{table}[htp]
\begin{center}
\begin{tabular}{|c|c|c|}
\hline
$\alpha^{st}$ & $\alpha^{inv}$ & $\alpha^{end}$\\
\hline
$Pre^{+}_{\alpha^{st}} = Pre^{+st}_{D\alpha}$ & $Pre^+_{\alpha^{inv}} = Pre^{+inv}_{D\alpha}$  & $Pre^+_{\alpha^{end}} = Pre^{+end}_{D\alpha}$\\
$Pre^{-}_{\alpha^{st}} = Pre^{-st}_{D\alpha}$ & $Pre^-_{\alpha^{inv}} = Pre^{-inv}_{D\alpha}$  & $Pre^-_{\alpha^{end}} = Pre^{-end}_{D\alpha}$ \\
$Eff^{+}_{\alpha^{st}} = Eff^{+st}_{D\alpha}$ & $Eff^+_{\alpha^{inv}} = \emptyset$ & $Eff^+_{\alpha^{end}} = Eff^{+end}_{D\alpha}$ \\
$Eff^-_{\alpha^{st}} = Eff^{-st}_{D\alpha}$ & $Eff^-_{\alpha^{inv}} = \emptyset$ & $Eff^-_{\alpha^{end}} = Eff^{-end}_{D\alpha}$\\
\hline
\end{tabular}
\end{center}
\label{tab:ia}
\caption{Transformation of durative action schemas in triples of instantaneous action schemas.}
\end{table}%

At this point, we are ready to construct $\cA^i$ and $\cA^d$ from $\cA^a$. We add each flatten and normalised instantaneous action in $\cA^a$ to $\cA^i$. For each durative action $D\alpha \in A^a$, after applying flattening and normalisation, we create the corresponding tuple $(\alpha^{st}, \alpha^{inv}, \alpha^{end})$ and add it to $\cA^d$. 

Given a planning instance $\cI$ in canonical form obtained from a PDDL2.1 instance $\cI'$ and a valid plan $\Pi$ for $\cI$, $\Pi$ can be converted into an equivalent valid plan $\Pi'$ for $\cI'$.

\subsection{Running Example: the \sfs{Floortile} Domain}
We use the \sfs{Floortile} domain as our running example. It has been introduced in the IPC-2014 and then reused in 2015. The full PDDL2.1 specification is available in Appendix A. The domain describes a set of robots that use different colours to paint patterns in floor tiles. The robots can move around the floor tiles in four directions (up, down, left and right). Robots paint with one color at a time, but can change their spray guns to any available color. Robots can only paint the tile that is in front (up) and behind (down) them, and once a tile has been painted no robot can stand on it. 

We have the following relations in this domain: 
$\cR = \{$ {\tt{up, down, right, left, robot-at, robot-has, painted,}
\tt{clear, available-color}} $\}$. They have arity two, except for the last two, which have arity one. {\tt {clear}} indicates whether a tile is still unpainted, {\tt {available-color}} whether a color gun is available to be picked by a robot and {\tt { up, down, right, left}} indicate the respective positions of two tiles. 

The set of instantaneous action schemas $\cA^i$  is empty, while the set of durative action schemas $\cA^d$ is the following:
$\cA^d = \{${\tt {change-color, paint-up, paint-down, up, down, right, left}} $\}$.  

As an example, the durative action schema {\tt {paint-up}} corresponds to the following triple: 

\noindent ({\tt {paint-up}}$^{st}$, {\tt {paint-up}}$^{inv}$, {\tt {paint-up}}$^{end}$), where the single instantaneous action schemas have the following specifications:
\begin{table}[htp]
\begin{center}
\begin{tabular}{|c|c|c|c|}
\hline
$\alpha$ & ${\tt {paint-up}}^{st}$ & ${\tt {paint-up}}^{inv}$ & ${\tt {paint-up}}^{end}$\\
\hline
$V_{\alpha}$ & $\{ {\tt{r, y, x, c }} \}$ & $\{ {\tt{r, y, x, c}} \}$ & $\{ {\tt{r, y, x, c}} \}$\\
\hline
$Pre^+_{\alpha}$ & $\{{\tt{robot-at(r, x)}}$ &  $\{{\tt{robot-has(r, c)}}$ & $\emptyset$\\
& ${\tt{clear(y)}} \}$ & ${\tt{up(y, x)}} \}$ & \\
\hline
$Pre^-_{\alpha}$ & $\emptyset$ & $\emptyset$ & $\emptyset$\\
\hline
$Eff^+_{\alpha}$ & $\emptyset$ & $\emptyset$ & $\{{\tt{painted(y, c)}}\}$\\
\hline
$Eff^-_{\alpha}$ & $\{{\tt{clear(y)}} \}$ & $\emptyset$ & $\emptyset$\\
\hline
\end{tabular}
\end{center}
\caption{Durative action schema {\tt {paint-up}} seen as a triple of instantaneous action schemas.}
\label{tb:paintup}
\end{table}%

Note that the triple of single instantaneous action schemas in canonical form is obtained from the following PDDL2.1 specification:

\lstset{language=Lisp}         
\begin{lstlisting}[frame=single,basicstyle=\footnotesize]
(:durative-action paint-up
  :parameters (?r - robot ?y - tile ?x - tile ?c - color)
  :duration (= ?duration 2)
  :condition (and (over all (robot-has ?r ?c))
                  (at start (robot-at ?r ?x))
                  (over all (up ?y ?x))
                  (at start (clear ?y)))
  :effect (and (at start (not (clear ?y)))
               (at end (painted ?y ?c))))
\end{lstlisting}


\section{Mutual Exclusion Invariants and Templates}
\label{sec:templates}
In this section, we formally introduce the concepts of invariant and mutual exclusion invariant and give examples of them. 

\begin{definition}[Invariant]
\label{def:invariant}
An \emph{invariant} of a PDDL2.1 planning instance is a property of the world states such that when it is satisfied in the initial state $Init$, it is satisfied in all reachable states $\cS_r$. 
\end{definition}

For example, given the $\sfs{Floortile}$ domain, a trivial invariant says that for each object $x$, if $x$ is a robot, then $x$ is not a tile. Similar invariants hold for each type defined in the domain. A more interesting invariant says that, for any two objects $x$ and $y$, if \verb+up(x,y)+ holds, then \verb+down(y,x)+ holds too, but \verb+down(x,y)+ does not. It is possible to identify several invariants for the $\sfs{Floortile}$ domain, ranging from trivial invariants such as those involving type predicates to very complex invariants.

In this paper, we focus on \emph{mutual exclusion} invariants, which state that a set of ground atoms can never be true at the same time. 

\begin{example}[\sfs{Floortile} domain]
\label{ex:template}
A mutual exclusion invariant for this domain states that two ground atoms indicating the position of a robot identified as \verb+rbt1+, such as \url{robot-at(rbt1,tile1)} and \verb+robot-at(rbt1,tile2)+, can never be true at the same time. Intuitively, this means that \verb+rbt1+ cannot be in two different positions simultaneously. Another more complex invariant states that, given a tile \verb+tile1+, a robot \verb+rbt1+ and a colour \verb+clr1+, atoms of the form  \verb+clear(tile1)+,  \verb+robot-at(rbt1,+ \verb+tile1)+ and \url{painted(tile1,clr1)} can never be true at the same time. This means that a tile can be in one of three possible states: not yet painted (clear), occupied by a robot that is painting it or already painted. 
\end{example}

\begin{definition}[Mutual Exclusion Invariant]
\label{def:mei}
Given a planning instance $\cI$, let $Z$ be a set of ground atoms in $2^{Atms}$. A \emph{mutual exclusion invariant} is an invariant stating that at most one element of $Z$ is true in any reachable state. We refer to a set $Z$ with this property as a \emph{mutual exclusion invariant set}.
\end{definition}

In what follows, we refer to mutual exclusion invariants and mutual exclusion invariant sets as simply invariants and invariant sets for the sake of brevity.

Although we aim to find sets of mutually exclusive ground atoms, we often work with relations and action schemas to control complexity. A convenient and compact way for indicating several invariant sets at the same time involves using \emph{invariant templates}, which are defined below, after introducing a few preliminary definitions. 

\begin{definition}[Component]
\label{def:comp}
A \emph{component} $c$ is a tuple $\langle r, a, p \rangle$, where $r$ is a relation symbol in $\cR$, $a$ is a number that represents the arity of $r$, i.e. $a=arity(r)$, and $p \in \{0, \ldots, a \}$ is a number that represents the position of one of the arguments of $r$, which is called the \emph{counted} argument. We put $p=a$ if there are no counted arguments. The set of the labelled \emph{fixed} arguments of $c$ is $F_c = \{(c, i) \, | \, i = 0, \ldots, (a-1);  i \not = p \}$.
\end{definition}

Given a set of components $\cC = \{ c_1, c_2, \ldots, c_n \}$, we define the set of fixed arguments of $\cC$ as $F_{\cC} = \bigcup\limits_{c \in \cC} F_c$.

\begin{definition}[Admissible Partition]
\label{def:part}
Given a set of components $\cC$ and a set of fixed arguments $F_{\cC}$, an \emph{admissible partition} of $F_{\cC}$ is a partition $\cF_{\cC} = \{ G_1, \ldots, G_k\}$ such that $| G_j \cap F_c | = 1$ for each $c \in \cC$.
\end{definition}

Given two elements $(c_1, i)$ and $(c_2, j)$ of $F_{\cC}$ that belong to the same set of the partition $\cF_{\cC}$ we use the notation: $(c_1, i) \sim_ {\cF_{\cC}}(c_2, j)$. 

\begin{remark}
\label{rem:boh}
Note that the existence of an admissible partition of $F_{\cC}$ implies that all the components in $\cC$ have the same number of fixed arguments, which is also the number of the sets in the partition. In the special case in which the number of fixed arguments in each component is equal to one, there is just one admissible (trivial) partition $\cF_{\cC} = \{F_{\cC}\}$.
\end{remark}

\begin{definition}[Template]
\label{def:temp}
A \emph{template} $\cT$ is a pair $( \cC, \cF_{\cC})$ such that $\cC$ is a set of components and $\cF_{\cC}$ is an admissible partition of $F_{\cC}$. We simply write $\cT = ( \cC)$ when the partition is trivial, i.e. $\cF_{\cC} = \{F_{\cC}\}$.
\end{definition}

By previous considerations, the set of relations appearing in the set of components $\cC$ of a template, up to a permutation of the position of the arguments, will always have the following form:
$$\{r_i(x_1^i, \dots , x_k^i, v^i)\,i=1,\dots ,n_1\}\cup\{r_i(x_1^i, \dots , x_k^i)\,i=n_1+1,\dots ,n_1+n_2\}$$ 
where $v^i$s indicate the counted arguments, $x_l^i$'s the fixed arguments and $x_l^i\sim_ {\cF_{\cC}}x_l^j$ for every $l, i, j$.

\begin{definition}[Template's Instance]
\label{def:temp-instance}
Given a template $\cT$, an \emph{instance} $\gamma$ of $\cT$ is a function that maps the elements in $F_{\cC}$ to the objects $\cO$ of the problem $\cP$ such that $\gamma(c_1, i) = \gamma(c_2, j)$ if and only if $(c_1, i) \sim_ {\cF_{\cC}}(c_2, j)$.
\end{definition}

\begin{definition}[Template's Instantiation]
\label{def:temp-instantiation}
The \emph{instantiation} of $\cT$ according to instance $\gamma$, $\gamma(\cT)$, is the set of ground atoms in $2^{Atms}$ obtained as follows: for each component $c = \langle r, a, p \rangle$ of $\cT$, take the relation symbol $r$, for each element $(c, i) \in F_{\cC}$ bind the argument in position $i$ according to $\gamma((c, i))$ and the counted argument in position $p$ to all the objects $\cO$ of the problem $\cP$.
\end{definition}

Given how we construct $\gamma(\cT)$, it is easy to see that its elements will look as follows:
$\gamma(\cT) = \{r_i(\gamma(x_1^i), \dots , \gamma(x_k^i), v^i)\,i=1,\dots ,n_1\}\cup\{r_i(\gamma(x_1^i), \dots , \gamma(x_k^i))\,i=n_1+1,\dots ,n_1+n_2\}$.

Instances are interesting because they can be used to reason about their (exponentially larger) instantiations without, in fact, constructing those instantiations.

Given a template $\cT$ and an instance $\gamma$, if the ground atoms in the instantiation of $\cT$ according to $\gamma$ are mutually exclusive in the initial state $Init$ and remain such in any state reachable $s \in \cS_r$, then $\gamma(\cT)$ is (by definition) a mutual exclusion invariant set. A template with this property for each possible instantiation $\gamma$ is called \emph{invariant template}.

\begin{definition}[Invariant Template]
\label{def:inv-temp}
A template $\cT$ is an \emph{invariant template} if, for each instance $\gamma$, the instantiation of $\cT$ according to $\gamma$ is a mutual exclusion invariant set.
\end{definition}

Given an invariant template $\cT$, we can create one state variable for each of its instances. The domains of these variables are the corresponding mutual exclusion invariant sets with an additional \emph{null} value, which is used when no element in the mutual exclusion invariant set is true. 

Before describing in what situations we can feasibly prove that a template is invariant, we introduce a final concept:

\begin{definition}[Template Instance's Weight]
\label{def:temp-weight}
Given an instance $\gamma$ of a template $\cT$, its \emph{weight} in a state $s$, $w(\cT, \gamma, s)$, is the number of ground atoms in its instantiation $\gamma(\cT)$ that are true in $s$:
$$w(\cT, \gamma,s) =  | s \cap \gamma(\cT) |$$
\end{definition}

\begin{proposition}
\label{rem:wei}
A template $\cT$ is an invariant template if and only if, for each instance $\gamma$ and each state $s \in \cS_r$, $w(\cT, \gamma,s) \leq 1$.
\end{proposition}
\begin{proof}
It follows from Definitions \ref{def:inv-temp} and \ref{def:temp-weight}.
\end{proof}

\begin{example}[\sfs{Floortile} domain]
\label{ex:long}
A template for this domain is $\cT_{ft} = (\{c_1, c_2, c_3 \})$, where:
\begin{itemize}
\item $c_1 =  \langle \verb+robot-at+, 2, 0 \rangle $ is the first component.  It includes the relation \url{robot-at} that has an arity of two (i.e. the relation \verb+robot-at(robot,tile)+ has two arguments) and the argument in position zero, i.e. \verb+robot+, is the counted argument. The remaining argument, \verb+tile+, which is in position one, is the fixed argument: $F_{c_{1}} = \{(c_1, 1)\}$.
\item $c_2 = \langle \verb+painted+, 2, 1 \rangle$ is the second component with $F_{c_{2}} = \{(c_2, 0)\}$. 
\item  $c_3 =  \langle \verb+clear+, 1, 1 \rangle $ is the last component with  $F_{c_{3}} = \{(c_3, 0)\}$.
\end{itemize}

Note that, since the three components have one fixed argument, all components are in the same equivalent class (trivial partition). 

Assume that we have a problem $\cP$ with two robots \verb+rbt1+ and \verb+rbt2+, three tiles, \verb+tile1+, \verb+tile2+ and \verb+tile3+ and one colour \verb+black+. Consider one possible instance $\gamma_1$ such that  $\gamma_1((c_1, 1)) = \gamma_{1}((c_2, 0)) = \gamma_{1}((c_3, 0)) =  \verb+tile1+$. The instantiation of the template $\cT_{ft}$ according to $\gamma_{1}$ is: $\gamma_{1}(\cT_{ft}) = \{$ \verb+robot-at(rbt1,tile1)+, \url+robot-at(rbt2,tile1)+, \url+painted(tile1,black)+, \verb+clear(tile1)+ $\}$. 


The weight of the instance $\gamma_{1}$ in a state $s$ is the number of ground instantiations of \url+robot-at+ \url+(robot,tile)+, \url+painted(tile,colour)+ and \url+clear(tile)+ that are true in $s$, where the variable $tile$ has been instantiated as \verb+tile1+. If we have a state $s$ in which no instantiations of \verb+robot-at(robot,tile)+ and \url+painted(tile,colour)+ are true, but \verb+clear(tile1)+ is true, the weight in $s$ is one.
 

We will see that we can actually prove that $\cT_{ft}$ is an invariant, which states that a tile can be clear, already painted or in the process of being painted by a robot.  Hence, for the problem $\cP$, we can create a state variable that represents each of the three tiles, whose values are the possible configurations of such tiles and the null value. It can also be proved that at least and at most one element of the mutual exclusion invariant set need to be true in any reachable state and, in consequence, the \verb+null+ value can be removed from the domain of the state variable. Hence we have:
$\cS\cV_{tile1} = \{$ \verb+robot-at(rbt1,tile1)+, \verb+robot-at(rbt2,tile1)+, \url{painted} \url{(tile1,} \url{black)}, \verb+clear(tile1)+ $\}$ and similarly for $\cS\cV_{tile2}$ and $\cS\cV_{tile3}$.

\end{example}

\section{Safe Instantaneous Ground Actions}
\label{sec:safeness inst actions}
In this and in the following sections, given a planning instance $\cI = (\cD, \cP)$ and a template $\cT$, we discuss the conditions that $\cT$ needs to satisfy to be an invariant. We make the standing assumption that the initial state $Init$ satisfies the weight condition $w(\cT,\gamma, Init)\leq 1$ for every instance $\gamma$ and determine \emph{sufficient} conditions on the families of instantaneous and durative actions in $\cD$ that ensure that $\cT$ is an invariant. In this section, in particular, we work out a concept of \emph{safety} for instantaneous actions that guarantees that, when a safe action is executed, the weight bound is not violated. 

\subsection{Safe instantaneous ground actions}
We assume a template $\cT$ to be fixed as well an instance $\gamma$.
Consider a set of concurrent pairwise non-interfering ground actions $A\subseteq \cG\cA$. The set of states $s\in\cS$ on which $A$ is applicable is denoted by $\cS_A$. We start with the following definition:

\begin{definition}[Strongly safe actions] 
The set of actions $A$ is \emph{strongly} $\gamma$-\emph{safe} if, for each $s \in \cS_A$ such that $w(\cT, \gamma,s) \leq 1$, the successor state $s' = \xi(s, A)$ is such that $w(\cT, \gamma,s') \leq 1$.
\end{definition}

The study of strong $\gamma$-safety for an action set $A$ can be reduced to the study of the state dynamics on the template instantiation $\gamma(\cT)$. This is intuitive and is formalised below.

\begin{remark}\label{remark: split}
Given an action $a\in \cG\cA$, define $a_{\gamma}$ and $a_{\neg\gamma}$ as the actions, respectively, specified by
$$
\begin{array}{ll}Pre^{\pm}_{a_{\gamma}}=Pre^{\pm}_{a}\cap\gamma(\cT),\quad &Eff^{\pm}_{a_{\gamma}}=Eff^{\pm}_{a}\cap\gamma(\cT)\\
Pre^{\pm}_{a_{\neg\gamma}}=Pre^{\pm}_{a}\cap\gamma(\cT)^c,\quad &Eff^{\pm}_{a_{\neg\gamma}}=Eff^{\pm}_{a}\cap\gamma(\cT)^c
\end{array}
$$
(where $A^c$ denotes the set complement of $A$).
Accordingly, we define, given an action set $A$, the action sets $A_\gamma=\{a_{\gamma}\,|\, a \in A\}$ and $A_{\neg\gamma}=\{a_{\neg\gamma}\,|\, a \in A\}$.
Split now the states in a similar way: given $s\in \cS$, put $s_{\gamma}=s\cap\gamma(\cT)$ and $s_{\neg\gamma}=s\cap\gamma(\cT)^c$. Then, it is immediate to see that given a state $s$ we have that $s\in\cS_A$ if and only if $s_{\gamma}\in\cS_{A_{\gamma}}$ and$s_{\neg\gamma}\in\cS_{A_{\neg\gamma}}$ and it holds that:
\begin{equation}\label{eq: state split}s'=\xi(s, A)\; \Leftrightarrow\; \left\{\begin{array}{rcl}s'_{\gamma}&=&\xi(s_{\gamma}, A_{\gamma})\\ s'_{\neg\gamma}&=&\xi(s_{\neg\gamma}, A_{\neg\gamma})\end{array}\right.\end{equation}
\end{remark}

This leads to the following simple but useful result.

\begin{proposition}\label{prop: equiv safe} 
For a set of actions $A$, the following conditions are equivalent:
\begin{enumerate}[(i)]
\item $A$ is strongly $\gamma$-safe;
\item $A_{\gamma}$ is strongly $\gamma$-safe;
\item For every $s\in\cS_{A_{\gamma}}$ such that $s\subseteq \gamma(\cT)$ and $w(\cT,\gamma, s)\leq 1$, it holds that the successor state $s' = \xi(s, A_{\gamma})$ is such that $w(\cT, \gamma,s') \leq 1$.
\end{enumerate}
\end{proposition}
%

\begin{definition}[Classification of Ground Actions]
\label{def: classification}
A set of ground actions $A$ is: 
\begin{itemize}
\item $\gamma$-\emph{unreachable} if $|Pre^+_{A_{\gamma}}|\geq 2$;
\item $\gamma$-\emph{heavy} if $|Pre^+_{A_{\gamma}}|\leq 1$ and $|Eff^+_{A_{\gamma}}|\geq 2$;
\item $\gamma$-\emph{irrelevant}  if $|Pre^+_{A_{\gamma}}|\leq 1$ and $|Eff^+_{A_{\gamma}}|=0$;
\item $\gamma$-\emph{relevant} for $\cT$ if $|Pre^+_{A_{\gamma}}|\leq 1$ and $|Eff^+_{A_{\gamma}}|=1$.
\end{itemize}
\end{definition}

It is immediate to see that each $A\subseteq \cG\cA$ belongs to one and just one of the above four disjoint classes. The following result clarifies their relation with strong safety. 

\begin{theorem}
\label{prop:weight class} 
Let $A$ be a set of ground actions. Then, 
\begin{enumerate}[1.]
\item if $A$ is $\gamma$-unreachable or $\gamma$-irrelevant, $A$ is strongly $\gamma$-safe;
\item if $A$ is heavy, $A$ is not strongly $\gamma$-safe.
\end{enumerate}
\end{theorem}
%
%
%

As the next example shows, relevant action sets may be strongly safe or not. 


\begin{example} Consider a template $\cT$ and an instance $\gamma$ such that $\gamma(\cT) = \{q, q', q''\}$, where $q$, $q'$ and $q''$ are three ground atoms, and $A = \{ a\}$, where $a$ is a ground action such that $Eff^+_a = \{ q \}$. Since $|Eff^+_a \cap \gamma(\cT)|=1$, $a$ is $\gamma$-relevant. Consider a state $s \in \cS_A$ such that $w(\cT, \gamma,s) \leq 1$.
\begin{itemize}
\item Suppose that $Pre^+_a = \{ q'\}$ and $Eff^-_a = \{ q' \}$ as shown in Figure \ref{fig:ex1}, left. In this case, $a$ is strongly $\gamma$-safe. In fact, $q' \in s$ and in consequence $w(\cT, \gamma,s) = 1$. Given $s' = \xi(s, a)$, $q' \not \in s'$, but $q \in s'$ and therefore $w(\cT, \gamma,s') = 1$. 
\item  Supposed that $Pre^+_a = \emptyset$ and $Eff^-_a = \{ q' \}$ as shown in Figure \ref{fig:ex1}, left. In this case, $a$ is $\gamma$-relevant, but is not strongly $\gamma$-safe. In fact, suppose that $q'' \in s$ and therefore $w(\cT, \gamma,s) = 1$. Since $q'' \not \in Eff^-_a$ and $q \in Eff^+_a$ , $s' = \xi(s, a)$ is such that $w(\cT, \gamma,s') = 2$.
\end{itemize}
 \end{example}

\begin{figure}[tbh]
\centering\includegraphics[width=100mm]{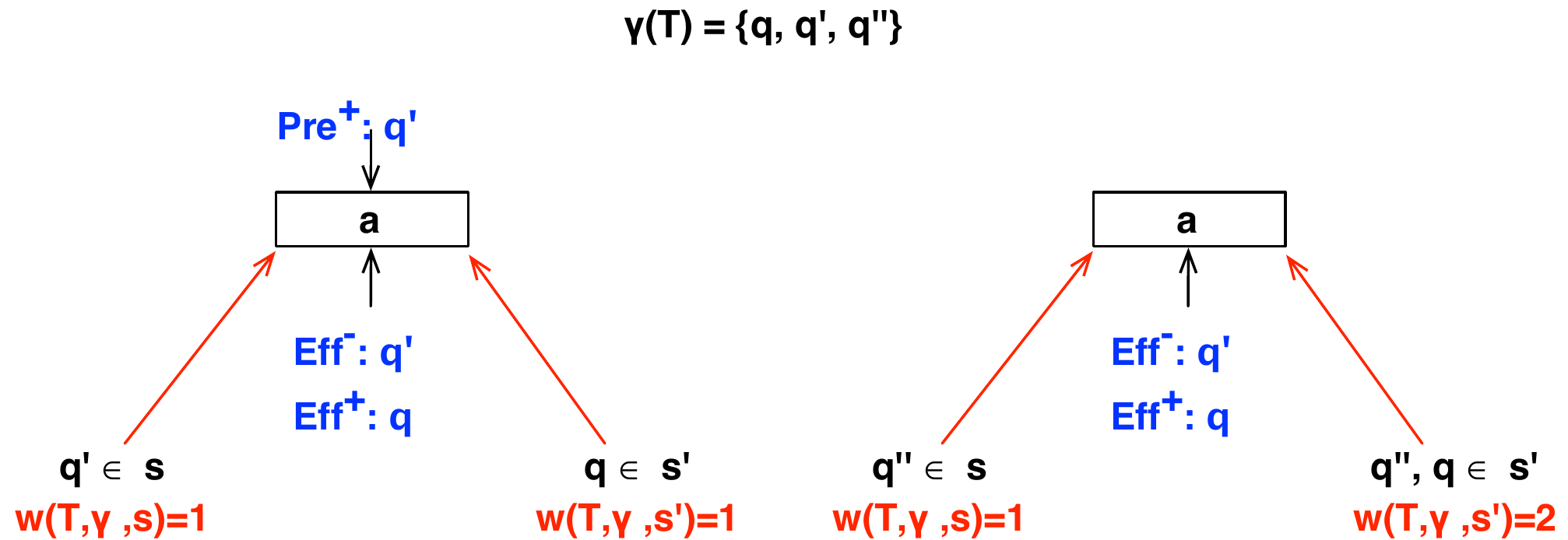}
\caption{Right: action $a$ is $\gamma$-relevant and strongly $\gamma$-safe. Left: action $a$ is essentially $\gamma$-relevant. Note that, in action specifications, only preconditions and effects different from the empty sets are represented.}
\label{fig:ex1}
\end{figure}

We now propose the following classification of relevant actions.

\begin{definition}[Classification of Relevant Actions] 
\label{def:relevant action sets} 
A $\gamma$-relevant set of ground actions $A$ is:
\begin{itemize}
\item \emph{balanced}  if $|Pre^+_{A_{\gamma}}|= 1$ and $Pre^+_{A_{\gamma}}\subseteq Eff^+_{A_{\gamma}}\cup Eff^-_{A_{\gamma}}$;
\item \emph{unbalanced} if $|Pre^+_{A_{\gamma}}|= 1$ and $Pre^+_{A_{\gamma}}\cap ( Eff^+_{A_{\gamma}}\cup Eff^-_{A_{\gamma}})=\emptyset$;
\item \emph{bounded} if $|Pre^+_{A_{\gamma}}|= 0$ and $Pre_{A_{\gamma}}\cup Eff_{A_{\gamma}}=\gamma(\cT)$;
\item \emph{unbounded} if $|Pre^+_{A_{\gamma}}|= 0$ and $Pre_{A_{\gamma}}\cup Eff_{A_{\gamma}}\neq \gamma(\cT)$.
\end{itemize}
\end{definition}

Again it is obvious that every relevant set of ground actions $A$ belongs to one and just one of the above four disjoint classes. The following results completes the analysis of strong safety.

\begin{theorem}
\label{prop:relevant actions} 
Let $A$ be a $\gamma$-relevant set of ground actions. Then, 
\begin{enumerate}[1.]
\item  if $A$ is balanced or bounded, $A$ is strongly $\gamma$-safe;
\item if $A$ is unbalanced or unbounded, $A$ is not strongly $\gamma$-safe.
\end{enumerate}
\end{theorem}

Immediate consequence of Theorems \ref{prop:weight class} and \ref{prop:relevant actions} is the following result:

\begin{corollary}
\label{cor: strongly safe} 
Let $A$ be a set of ground actions. Then,
\begin{enumerate}[1.]
\item if $A$ is either $\gamma$-unreachable, $\gamma$-irrelevant, $\gamma$-relevant balanced, or $\gamma$-relevant bounded, $A$ is strongly $\gamma$-safe;
\item if $A$ is either $\gamma$-heavy, $\gamma$-relevant unbalanced, or $\gamma$-relevant unbounded, $A$ is not strongly $\gamma$-safe.
\end{enumerate}
\end{corollary}

Finally, this result shows that strong $\gamma$-safety can always be checked at the level of single actions.

\begin{proposition}
\label{rem: strongly-safe-action} 
Let $A$ be a set of actions. Then,
$A$ is strongly $\gamma$-safe if $a$ is strongly $\gamma$-safe for all $a \in A$.
\end{proposition}

\begin{example}[\sfs{Floortile} domain]
Consider a template $\cT = (\{c\}, \{ \{ (c, 0)\} \})$, where $c = \langle \verb+painted+,$ $ 2, 1 \rangle$. Take two instances $\gamma_1(c, 0) =$ \url{tile1} and $\gamma_2(c, 0) =$ \url{tile2} and the ground action $a=$\url{paint} \url{-up(rbt1, tile1, tile3, red)} (Table \ref{tb:paintup}). This action is strongly $\gamma_2-$safe since it is irrelevant, but it is $\gamma_1$-relevant and not strongly $\gamma_1-$safe. This is because, given a state $s \in \cS_A$ such that, for example, \url{painted(tile4, black)}$ \in s$, $s' = \xi(s, a)$ is such that \url{painted(tile4, black)}, \url{painted(tile1, black)} $ \in s'$, with  $w(\cT, \gamma_1,s') = 2$.
\end{example}

We conclude now with a definition and a first result that expresses a sufficient condition for a template to be invariant.

\begin{definition}
\label{def: strongly-safe-action}
Given a template $\cT$, a set of actions $A \subseteq \cG\cA$ is \emph{strongly safe} if it is strongly $\gamma$-\emph{safe} for every instance $\gamma$.
\end{definition}

We have the following result:

\begin{corollary}
\label{cor:inv}
Given a template $\cT$, $\cT$ is invariant if for each $a \in \cG\cA$, $a$ is strongly safe.
\end{corollary}
\begin{proof} It follows from Remark \ref{rem: strongly-safe-action}, Definition \ref{def: strongly-safe-action} and Proposition \ref{rem:wei}.
\end{proof}

The condition expressed in Corollary \ref{cor:inv} cannot be inverted in general. Indeed, a template can be invariant even if not all actions are strongly safe. We will see when this happens in the following section. 

\section{Safe action sequences and safe durative actions}
\label{sec:safeness sequences}
A template can be invariant even if not all ground actions are strongly safe. This happens for two reasons. On the one hand, since the set of reachable states $\cS_r$ is in general smaller than $\cS$, it may be that all the states that are responsible for the lack of strong safety are unreachable, i.e. they are not in $\cS_r$. On the other hand, in domains with durative actions, some instantaneous actions are temporally coupled because they are the start and end fragments of the same durative action. This coupling imposes constraints on the states where the end part can be applied, which might prove helpful to establish that a template is invariant. 
While in this paper we will not analyse the first case as it would require an analysis of the set of reachable states $\cS_r$, which is practically unfeasible, we now elaborate suitable simple concepts of safety for durative actions, which are weaker than strong safety. This extension is of great importance to apply our technique to real-world planning domains. In fact, they often present durative actions that have a non strongly safe end fragment, but that nonetheless never violate the weight condition when appearing in a plan. We propose a definition of safety for durative actions that captures this case. However, given that in a plan a durative action may intertwine with other actions that happen in between its start and end points, we need to work out a concept of safety for more general sequence of actions than just durative ones.

Below, we consider general sequences of ground action sets ${\bf A}:=(A^1,A^2,\dots ,A^n)$. Note that any valid simple plan $\pi$ naturally induces such a sequence. Indeed, if $trace(\pi)=\{S_i=(t_i,s_i)_{i=0,\dots , \bar k}\}$ and $A_{t_i}$ are the relative happenings, we can consider the so called \emph{happening sequence} of $\pi$: ${\bf A}_{\pi}=(A_{t_0}, \dots , A_{t_{\bar k}})$. $\bf A_{\pi}$ contains all the information on the plan $\pi$ except the time values at which the various actions happen. 

To study the invariance of a template, we break the happening sequence of each plan into subsequences determined by the happenings of durative actions. More precisely, we consider sequences ${\bf A}:=(A^1,A^2,\dots ,A^n)$ where, for some durative action $Da=(a^{st}, a^{inv}, a^{end})$, we have that $a^{st}\in A^1$ and $a^{end}\in A^n$. $A^2, \dots, A^{n-1}$, as well as $A^1$ and $A^n$, possibly contain other actions that are executed over the duration of $Da$. 
However, in the first instance, it is convenient to consider general sequences of ground actions ${\bf A}:=(A^1,A^2,\dots ,A^n)$ without referring to plans or durative actions. Hence, in this section, we first propose a definition of safety for $\bf A$ such that, when $\bf A$ is effectively executed serially in any valid plan $\pi$, the weight constraint is not violated in any intermediate step and at the end of the sequence, if it is not violated in the state where the sequence is initially applied. For single action sets (sequences of length $n=1$), such concept coincides with the notion of strong safety.

We then consider a slightly stronger notion of safety which is \emph{robust} to the insertion, between elements of the sequence, of other ground actions whose positive effects have no intersection with the template. To do this, it is necessary to introduce a number of auxiliary concepts relating to the state dynamics induced by the execution of $\bf A$. This general theory will be then applied to sequences constructed from durative actions. 


\subsection{Safe ground action sequences}

Given a sequence of ground action sets ${\bf A}:=(A^1,A^2,\dots ,A^n)$, we denote with ${\bf S}_{\bf A}$ the set of state sequences $(s^0,\dots ,s^n)\in \cS^{n+1}$ such that 
$$A^i\;\hbox{\rm  is applicable in}\; s^{i-1}\;{\rm and}\; s^i=\xi(s^{i-1}, A_i)\;\forall i=1,\dots ,n$$
If $(s^0,\dots ,s^n)\in {\bf S}_{\bf A}$, we say that $(s^0,\dots ,s^n)$ is a state sequence \emph{compatible} with ${\bf A}$. Given an instance $\gamma$, we also define ${\bf S}_{\bf A}(\gamma)$ as the set of compatible state sequences $(s^0,\dots ,s^n)$ such that $w(\cT,\gamma, s^0)\leq 1$. We use the following notation for subsequences of ${\bf A}$: ${\bf A}_h^k=(A^h,A^{h+1},\dots ,A^k)$.

We now fix a template $\cT$ and an instance $\gamma$ and propose the following natural definition of safety for a sequence.

\begin{definition}[Individually safe actions]
\label{def: gamma ind safe}
A sequence of ground action sets ${\bf A}:=(A^1,A^2,\dots ,A^n)$ is \emph{individually $\gamma$-safe} if for every sequence of states $(s^0,\dots ,s^n)$ $\in {\bf S}_{\bf A}$ we have that 
$$w(\cT, \gamma, s^0)\leq 1\;\Rightarrow\; w(\cT, \gamma, s^i)\leq 1\,\forall i=1,\dots , n$$
\end{definition}

The invariance of a template can now be expressed in terms of individual safety for the happening sequences. 

\begin{proposition}\label{prop: invariance safety} Let $\cT$ be a template. Suppose that for every valid simple plan $\pi$, the sequence ${\bf A}_{\pi}$ is individually $\gamma$-safe for every instance $\gamma$. Then, $\cT$ is invariant.
\end{proposition}


Below are elementary properties of individual $\gamma$-safety for subsequences of ${\bf A}$. 

\begin{proposition}\label{prop: ind safe} 
Consider a sequence of ground action sets ${\bf A}:=(A^1,A^2,\dots ,A^n)$. The following properties hold:
\begin{enumerate}[(i)]
\item if, for some $k$ and $h$ such that $k\geq h-1$, ${\bf A}_1^k=(A^1,A^2,\dots ,A^k)$ and ${\bf A}_h^n=(A^{h},\dots ,A^n)$ are both individually $\gamma$-safe, then also ${\bf A}$ is individually $\gamma$-safe;
\item if ${\bf A}$ is individually $\gamma$-safe and $A^k$ and $A^{k+1}$ are non-interfering, then ${\bf A}'=(A^1,A^2,\dots ,A^k\cup A^{k+1},\dots ,A^n)$ is individually $\gamma$-safe;
\item if ${\bf A}$ is individually $\gamma$-safe and $B^j$, for $j=1,\dots , n$ are action sets such that $Eff_{B^j}=\emptyset$, then, ${\bf A}'=(A^1,  B^1, A^2,\dots , B^n, A^n)$ and ${\bf A}''=(A^1\cup  B^1,\dots , A^n\cup B^n)$ are individually $\gamma$-safe.  
\end{enumerate}
\end{proposition}
%
%

The following is a useful consequence of the previous results: it asserts that if individual safety holds locally in a sequence, then it also holds globally.

\begin{corollary}\label{cor: ind safe} 
For a sequence of ground action sets ${\bf A}:=(A^1,A^2,\dots ,A^n)$, the following conditions are equivalent:
\begin{enumerate}
\item[(i)]  the sequence ${\bf A}$ is individually $\gamma$-safe;
\item[(ii)]  for each $j=1,\dots , n$, there exists a subsequence ${\bf A}_{j-r}^{j+s}$, with $r, s \geq 0$, that is individually $\gamma$-safe.
\end{enumerate}
\end{corollary}
\begin{proof}
(i)$\Rightarrow$(ii) is trivial and (ii)$\Rightarrow$(i) follows from an iterative use of (i) of Proposition \ref{prop: ind safe}.
\end{proof}

Individual safety is a weak property since it is not robust with respect to the insertion of other actions, even when these actions are irrelevant but possess delete effects. This is connected to the fact that, while individual safety has this nice local to global feature illustrated in Corollary \ref{cor: ind safe}, it does not possess the opposite feature: subsequences of individual safe sequences may not be individual safe. The following example shows both these phenomena.

\begin{example}
\label{ex: safe}
Consider a template $\cT$ and an instance $\gamma$ such that $\gamma(\cT) = \{q, q'\}$. The set of state sequences compatible with ${\bf A}:=(a^1,a^2)$ (Figure \ref{fig:ind} - top diagram) is: ${\bf S}_{\bf A} = \{ (s^0, s^1, s^2) | q \not\in s^0, s^1=s^0 \cup \{q'\}, s^2=s^1\}$. Note that $q \not\in s^0$ because, by hypothesis, $a^2$ is applicable in $s^1$ and $s^1=s^0 \cup \{q'\}$. ${\bf A}$ is individually  $\gamma$-safe since $w(\cT, \gamma, s^i) \leq 1$ for every state $s^i$ that appears in ${\bf S}_{\bf A}$. Note that $a^1$ is $\gamma$-relevant unbounded and thus not strongly $\gamma$-safe.

Now consider the sequence ${\bf \tilde A}:=(a^1, b, a^2)$ (Figure \ref{fig:ind} - bottom diagram) where a $\gamma$-irrelevant action $b$ is inserted between $a^1$ and $a^2$. The new set of state sequences compatible with ${\bf \tilde A}$ is: ${\bf S}_{\bf \tilde A} = \{ (s^0, s^1, s^2, s^3) | s^1=s^0 \cup \{q'\}, s^2 = s^1 \setminus \{q \}, s^3=s^2\}$. Note that now $q$ can be in $s^0$ since it is the action $b$ that ensures the applicability of $a^2$. If $q\in s^0$, since $a^1$ adds $q'$ to $s^0$, $w(\cT, \gamma, s^1) = 2$. Clearly, this new sequence is not individually $\gamma$-safe. The insertion of a $\gamma$-irrelevant action has failed the individual $\gamma$-safety of the sequence ${\bf A}$.
\end{example}

\begin{figure}[tbh]
\centering\includegraphics[width=60mm]{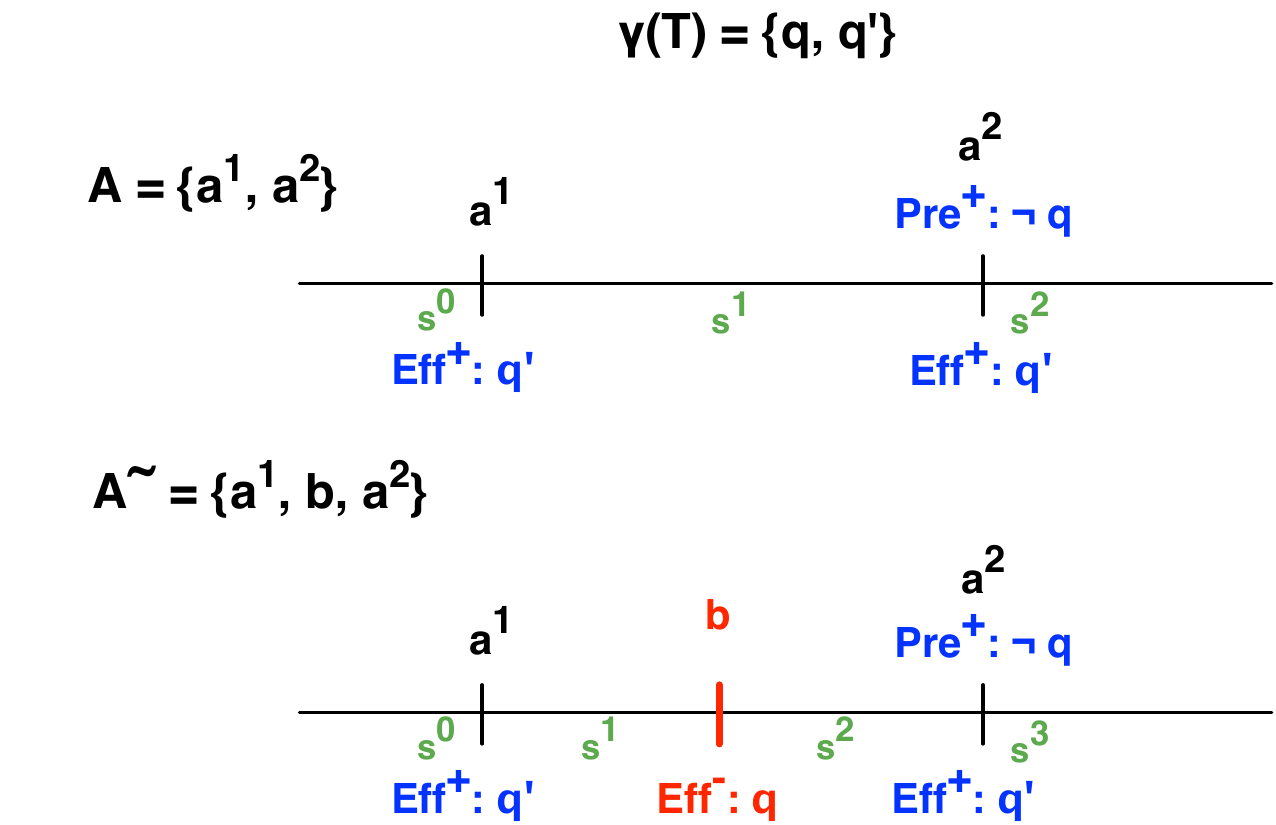}
\caption{The insertion of the $\gamma$-irrelevant action $b$ fails the individual $\gamma$-safety of the sequence ${\bf A}$.}
\label{fig:ind}
\end{figure}

For proving some of our results, the concept of individual safety is not sufficient. Below we present a stronger definition of safety for an action sequence that is robust with respect to the insertion of irrelevant actions in it. First, we define the simple concepts of executable and reachable sequences.

\begin{definition}[Executable and reachable actions]
\label{def: executable} The sequence ${\bf A}=(A^1,A^2,\dots,$ $A^n)$ is called:
\begin{itemize}
\item \emph{executable} if ${\bf S}_{\bf A}\neq \emptyset$;
\item $\gamma$-\emph{(un)reachable} if ${\bf S}_{\bf A}(\gamma)\neq \emptyset$ (${\bf S}_{\bf A}(\gamma)= \emptyset$).
\end{itemize}
\end{definition}

\begin{remark}\label{rem: exec reach safe}
Note the following chain of implications
$${\rm non-executable}\;\Rightarrow\; \gamma{\rm -unreachable} \;\Rightarrow\; {\rm individually}\, \gamma{\rm -safe}$$
Note that if $\pi$ is a valid simple plan with happening sequence ${\bf A}_{\pi}$, then ${\bf A}_{\pi}$ is $\gamma$-reachable for every $\gamma$ due to the standing assumption that $w(\cT, \gamma, Init)\leq 1$ for every $\gamma$. Moreover, every subsequence ${\bf A}$ of ${\bf A}_{\pi}$ is executable. If a subsequence ${\bf A}$ of ${\bf A}_{\pi}$ is $\gamma$-unreachable, the weight will surely exceed $2$ at some point of the plan $\pi$ and thus the template $\cT$ will not be invariant. 
\end{remark}

In the special case of a sequence of length $2$, executability and reachability admit  very simple characterisations. We report them below as we will need them later.
First define, for a generic set of actions $A$, the subsets
$$\Gamma^+_A:=(Pre^+_A\setminus Eff^-_A)\cup Eff^+_A,\quad \Gamma^-_A:=(Pre^-_A\setminus Eff^+_A)\cup Eff^-_A$$
We have the following result:
\begin{proposition}\label{prop: exec two} Given a sequence of two ground action sets ${\bf A}=(A^1,A^2)$, the following conditions are equivalent:
\begin{enumerate}
\item[(i)] $\bf A$ is executable;
\item[(ii)] $\Gamma^+_{A^1}\cap Pre^-_{A^2}=\emptyset = \Gamma^-_{A^1}\cap Pre^+_{A^2}$.
\end{enumerate}
\end{proposition}
%
%
%

\begin{proposition}\label{prop: reach two} Given a sequence of two ground action sets ${\bf A}=(A^1,A^2)$, the following conditions are equivalent:
\begin{enumerate}
\item[(i)] $\bf A$ is $\gamma$-reachable;
\item[(ii)] $\bf A$ is executable and $|Pre^+_{A_{\gamma}^1}\cup (Pre^+_{A_{\gamma}^2}\setminus Eff^+_{A_{\gamma}^1})|\leq 1$.
\end{enumerate}
\end{proposition}
%

The following are immediate properties of executability and unreachability:

\begin{proposition}\label{prop: exec reach} 
Consider a sequence ${\bf A}=(A^1,A^2,\dots ,A^n)$ that is executable or $\gamma$-reachable. Then,
\begin{enumerate}
\item[(i)] if $B^j\subseteq A^j$ are such that $Eff_{B^j}=\emptyset $ for every $j=1, \dots , n-1$, then 
also ${\bf A}'=(A^1\setminus B^1,A^2\setminus B^2,\dots ,A^n\setminus B^n)$ is, respectively, executable or $\gamma$-reachable. 
\item[(ii)] if $A^j=A'^j\cup A''^j$ for some $j=1,\dots , n$, then also 
${\bf A}'=(A^1,A^2,\dots A'^j, A''^j,$ $\dots ,A^n)$ is, respectively, executable or $\gamma$-reachable.
\end{enumerate}
\end{proposition}
%

Here is our stronger notion of safety:

\begin{definition}[Safe actions]
\label{def: gamma str safe}
A sequence of ground action sets ${\bf A}:=(A^1,A^2,$ $\dots A^n)$ is \emph{$\gamma$-safe} if it is executable and ${\bf A}_1^k$ is individually $\gamma$-safe for every $k=1,\dots ,n$.
\end{definition}

Note how the sequence ${\bf A}:=(a^1,a^2)$ considered in Example \ref{ex: safe} is indeed not $\gamma$-safe, since $a^1$ is not individually $\gamma$-safe. The next example shows instead the reason why executability is required. 

\begin{example}
\label{ex: exe}
Consider a template $\cT$ and an instance $\gamma$ such that $\gamma(\cT) = \{q, q', q''\}$. The sequence ${\bf A}:=(a^1,a^2)$ (Figure \ref{fig:notexe} - top diagram) is individually $\gamma-$safe because ${\bf S}_{\bf A} = \emptyset$ ($\neg q''$ is required false by $a^2$, but it is asserted true by $a^1$). 

Now consider the sequence ${\bf \tilde A}:=(a^1, b, a^2)$ (Figure \ref{fig:notexe} - bottom diagram) where a $\gamma$-irrelevant action $b$ is inserted between $a^1$ and $a^2$. This insertion makes ${\bf S}_{\bf A} \neq \emptyset$. Since $q, q' \in s^3$, $w(\cT, \gamma, s^3) = 2$ and therefore ${\bf A}$ is not individually $\gamma-$safe.
\end{example}

\begin{figure}[tbh]
\centering\includegraphics[width=80mm]{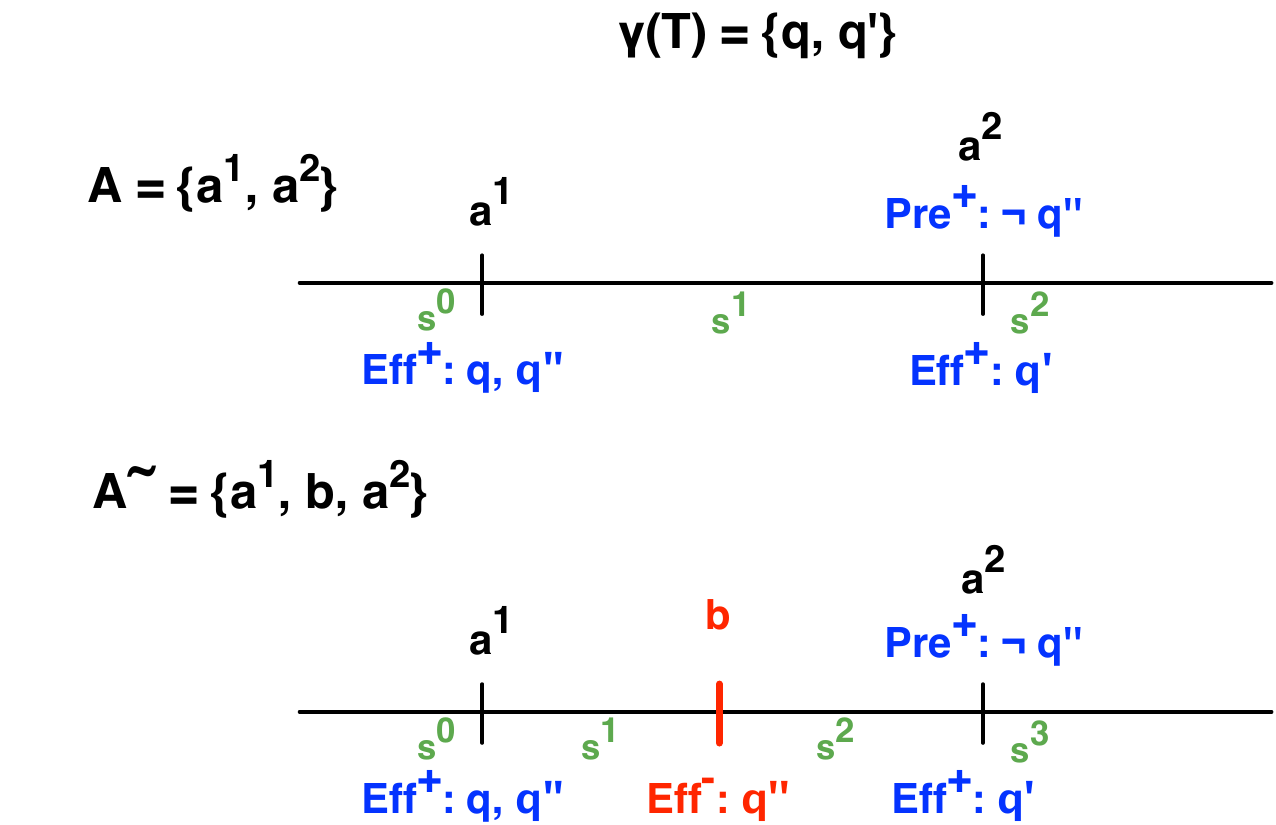}
\caption{Lack of robustness for individually $\gamma$-safe actions.}
\label{fig:notexe}
\end{figure}

\begin{remark}
If  ${\bf A}=(A^1,A^2,\dots A^n)$ is $\gamma$-safe, the first action set $A^1$ must necessarily be strongly $\gamma$-safe. On the other hand, if ${\bf A}$ is executable and every $A^j$ for $j=1,\dots , n$ is strongly $\gamma$-safe then, $\bf A$ is $\gamma$-safe. 
\end{remark}

This motivates the following definition.

\begin{definition}[Strongly and simply safe actions]
\label{def: gamma simply safe}
A sequence of ground action sets ${\bf A}=(A^1,A^2,$ $\dots ,A^n)$ is:
\begin{itemize}
\item \emph{strongly $\gamma$-safe} if it is executable and every $A^j$ for $j=1,\dots , n$ is strongly $\gamma$-safe;
\item \emph{simply $\gamma$-safe} if it is $\gamma$-safe but not strongly $\gamma$-safe.
\end{itemize}
\end{definition}

The following result shows that heavy or relevant unbalanced actions cannot be part of safe reachable sequences.
\begin{proposition}\label{prop: heavy presence} 
Suppose ${\bf A}=(A^1,A^2,\dots ,A^n)$ is a $\gamma$-safe and $\gamma$-reachable sequence of ground action sets. Then, for every $j=1,\dots ,n$, $A^j$ is not $\gamma$-heavy and is no $\gamma$-relevant unbalanced.
\end{proposition}

The property of $\gamma$-reachability is necessary for the previous result to hold, as the following example shows.

\begin{example}
Consider a template $\cT$ and an instance $\gamma$ such that $\gamma(\cT) = \{q, q'\}$. The sequence ${\bf A}:=(a^1,a^2)$ (Figure \ref{fig:ex7}) is $\gamma$-safe because it is executable and the subsequences $(a^1)$ and $(a^1,a^2)$ are both individually safe given that $a^1$ is $\gamma$-unreachable. In this case, Proposition \ref{prop: heavy presence} does not hold since $a^2$ is $\gamma$-heavy. 
\end{example}

\begin{figure}[tbh]
\centering\includegraphics[width=70mm]{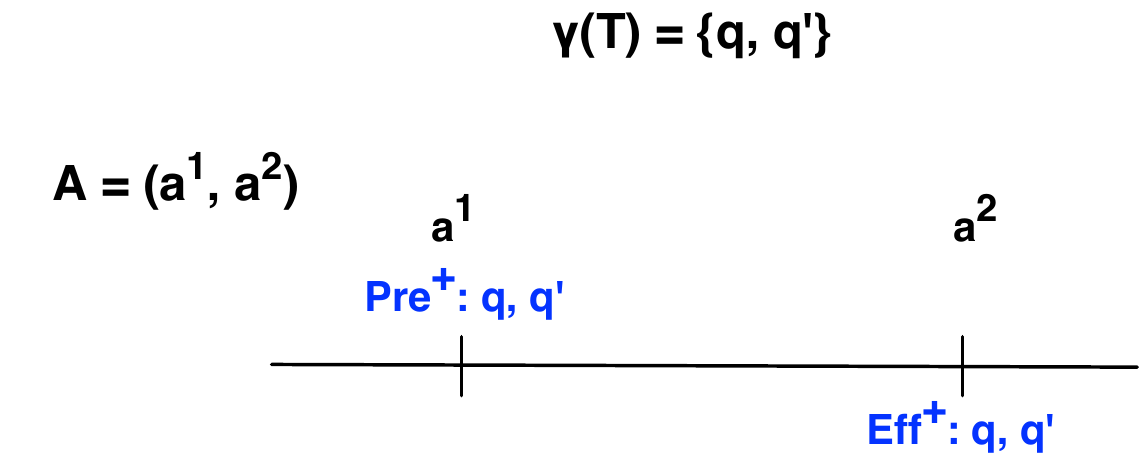}
\caption{The sequence ${\bf A}:=(a^1,a^2)$ is $\gamma$-safe.}
\label{fig:ex7}
\end{figure}

In studying the two safety properties for a sequence $\bf A$ introduced so far, we can essentially restrict ourselves to study the state dynamics on the template instantiation $\gamma(\cT)$ as we did for strong $\gamma$-safety of instantaneous actions (see Remark \ref{remark: split}).
 
Given the sequence ${\bf A}:=(A^1,A^2,\dots A^n)$, we denote by ${\bf A}_{\gamma}:=(A^1_{\gamma},A^2_{\gamma},\dots A^n_{\gamma})$ and ${\bf A}_{\neg\gamma}:=(A^1_{\neg\gamma}, A^2_{\neg\gamma},\dots A^n_{\neg\gamma})$ the corresponding restricted sequences. We have the following result.

\begin{proposition}\label{prop: safe split} Given the sequence ${\bf A}:=(A^1,A^2,\dots A^n)$,
\begin{enumerate}[1.]
\item $\bf A$ is executable if and only if ${\bf A}_{\gamma}$ and ${\bf A}_{\neg\gamma}$ are both executable;
\item $\bf A$ is $\gamma$-reachable if and only if ${\bf A}_{\gamma}$ is $\gamma$-reachable and ${\bf A}_{\neg\gamma}$ is executable;
\item $\bf A$ is individually $\gamma$-safe if and only if ${\bf A}_{\gamma}$ is individually $\gamma$-safe.
\item $\bf A$ is $\gamma$-safe if and only if ${\bf A}_{\gamma}$ is $\gamma$-safe and ${\bf A}_{\neg\gamma}$ is executable.
\end{enumerate}
\end{proposition}
%

We are now ready to state and prove the following fundamental result, which ensures that the concept of safe sequence is robust to the insertion of irrelevant actions.

\begin{theorem}\label{theo: safe sequences} 
Consider a $\gamma$-safe sequence ${\bf A}:=(A^1,A^2)$ and $\gamma$-irrelevant ground action sets $B^1, B^2,\dots , B^{n}$. Then, the sequence
$\tilde{\bf A}:=(A^1,B^1,\dots , B^{n}, A^2)$ is either non executable or $\gamma$-safe.
\end{theorem}

We conclude this section with a last definition:

\begin{definition}
\label{def: sequence safe}
Given a template $\cT$, a sequence of ground action sets ${\bf A}$ is, respectively, \emph{safe} or \emph{strongly safe} if it is, respectively, $\gamma$-safe or strongly $\gamma$-safe, for every instance $\gamma$. It is \emph{simply safe} if it is safe but not strongly safe.
\end{definition}

\subsection{Safe ground durative actions}
\label{subsec:sgda}
We now restrict our attention to ground durative actions $Da=(a^{st}, a^{inv}, a^{end})$. If we interpret $Da$ as a sequence of three actions, we can consider for it the properties defined for general sequences such as $\gamma$-safety and strong $\gamma$-safety. We propose an explicit characterisation of these properties in this case, which will be useful later on.

First, let us focus on the specific way in which durative actions appear in the happening sequence of a plan. 
Consider a simple induced plan $\pi$ having $trace(\pi)=\{S_i=(t_i,s_i)_{i=0,\dots , \bar k}\}$ and happenings $A_{t_i}$. Let ${\bf A}_{\pi}$ be the corresponding happening sequence. If a durative action $Da$ happens in $\pi$ in the time interval $[t_{i+1}, t_j]$, we clearly have that $a^{st}\in A_{t_{i+1}}$ and $a^{end}\in A_{t_j}$. Moreover, by the way $\pi$ is constructed from the original plan, we also have that $j-i$ is odd and for every even $h=2,4,\dots ,j-i-1$, $A_{t_{i+h}}$ consists of $\{a^{inv}\}$ and, possibly, preconditions of other durative actions happening in the original plan $\Pi$ simultaneously or intertwined with $Da$. This motivates the following definition.

\begin{definition}[Admissible actions] A sequence ${\bf A}:=(A^1,A^2,\dots ,A^n)$ is:
\begin{itemize}
\item \emph{admissible} if, for any durative action $Da'$, it holds
$$a'^{st}\in A^{i}\;\Rightarrow\; a'^{inv}\in A^{i+1}\quad a'^{end}\in A^{j}\;\Rightarrow\; a'^{inv}\in A^{j-1}$$
\item $Da$-\emph{admissible}, for some durative action $Da$, if it is admissible and the following conditions are satisfied:
\begin{enumerate}[(a)]
\item $a^{st}\in A^1$ and $a^{end}\in A^n$;
\item $n$ is odd and for every $j=2,4,\dots ,n-1$, $A^j$ consists of $\{a^{inv}\}$ and, possibly, preconditions of other durative actions.
\end{enumerate}
\end{itemize}
\end{definition}

Any subsequence of the happening sequence of a simple plan is admissible and if its starting and its ending coincide with, respectively, the start and the end of a durative action $Da$, it is $Da$-admissible. 

To study the safety of a $Da$-{admissible} sequence,  we can, in many cases, reduce the analysis of the durative action $Da$ to the analysis of an auxiliary sequence of just two actions $Da_*=(a^{st}_*, a^{end}_*)$, where $a_*^{st}$ and $a_*^{end}$ are instantaneous actions such that:
$$\begin{array}{ll} Eff^{\pm}_{a_*^{st}}=Eff^{\pm}_{a^{st}},\quad &Pre^{\pm}_{a_*^{st}}=Pre^{\pm}_{a^{st}}\cup (Pre^{\pm}_{a^{inv}}\setminus Eff^{\pm}_{a^{st}})\\
Eff^{\pm}_{a_*^{end}}=Eff^{\pm}_{a^{end}},\quad &Pre^{\pm}_{a_*^{end}}=Pre^{\pm}_{a^{end}}\cup Pre^{\pm}_{a^{inv}}
\end{array}
$$
The relation between the two sequences $Da$ and $Da_*$ is clarified by the following result. Assume, as always, that a template $\cT$ and an instance $\gamma$ have been fixed.

\begin{proposition}\label{prop: aux durative 1} The following facts hold true:
\begin{enumerate}
\item[(i)] $(s^0, s^1, s^2)\in {\bf S}_{(a^{st}, a^{inv})}$ if and only if $s^1=s^2$ and $(s^0, s^1)\in {\bf S}_{a^{st}_*}$;
\item[(ii)] $(s^0, s^1, s^2)\in {\bf S}_{(a^{inv}, a^{end})}$ if and only if $s^0=s^1$ and $(s^1, s^2)\in {\bf S}_{a^{end}_*}$;
\item[(iii)] $(s^0, s^1, s^2, s^3)\in {\bf S}_{Da}$ if and only if $s^1=s^2$ and $(s^0, s^1, s^3)\in {\bf S}_{Da_*}$;
\item[(iv)] $(a^{st}, a^{inv})$  is individually $\gamma$-safe if and only if $a^{st}_*$  is strongly $\gamma$-safe;
\item[(v)] $(a^{inv}, a^{end})$  is individually $\gamma$-safe if and only if $a^{end}_*$  is strongly $\gamma$-safe;
\item[(vi)] $Da$ is individually $\gamma$-safe if and only if $Da_*$  is individually $\gamma$-safe.
\end{enumerate}
\end{proposition}

The next result studies the effect of exchanging the start and end of a durative action $Da$ with those of the auxiliary sequence $Da_*$ in a $Da$-admissible sequence.

\begin{proposition}\label{prop: aux durative 2}  Consider a durative action $Da=(a^{st}, a^{inv}, a^{end})$ and a $Da$-admissible sequence of actions ${\bf A}=(\{a^{st}\}, A^2,  \dots, A^{n-1}, \{a^{end}\}).$
Put ${\bf A}_*=(\{a^{st}_*\}, A^2,  \dots, A^{n-1}, \{a^{end}_*\})$.
Then ${\bf S}_{\bf A}={\bf S}_{\bf A_*}$. In particular,  $\bf A$ is individually $\gamma$-safe if and only if $\bf A_*$ is individually $\gamma$-safe.
\end{proposition}

The last proposition implies that, in analysing the state dynamics in a valid plan, we can replace the start and end of each durative action $Da$ with the corresponding ones of the auxiliary sequence $Da_*$, if such start and end happen isolated from other actions. This is useful for two reasons. On the one hand, there are cases in which $Da_*$ is strongly safe even if $Da$ is not. On the other hand, we can directly apply Theorem \ref{theo: safe sequences} to $Da_*$ since it is of length $2$.

As we shall see later, our sufficient results for the invariance of a template always require safety (strong or simple) of the auxiliary actions $Da_*=(a^{st}_*, a^{end}_*)$. The check for strong safety can be done by considering the single components of $Da_*$ and referring back to the analysis that we carried out in previous chapter. Below, we propose a full characterisation of simple safety for auxiliary actions. 

Note first that if $Da_*=(a^{st}_*, a^{end}_*)$ is simply $\gamma$-safe (Definition \ref{def: gamma simply safe}), necessarily $Da_*$ is executable, $a^{st}_*$ is strongly $\gamma$-safe and $a^{end}_*$ is not strongly $\gamma$-safe. If, besides these three properties, $Da_*$ is $\gamma$-unreachable, then, $Da_*$ is simply $\gamma$-safe because of Remark \ref{rem: exec reach safe}. If we instead assume that $Da_*$ is simply $\gamma$-safe and $\gamma$-reachable, then, because of Proposition \ref{prop: heavy presence}, we have that $a^{end}_*$ is $\gamma$-relevant unbounded. The following result completely characterises simple $\gamma$-safety for such actions.

\begin{proposition}
\label{prop: aux simple safe} Assume that  $Da_*=(a^{st}_*, a^{end}_*)$ is a $\gamma$-reachable sequence such that $a^{st}_*$ is strongly $\gamma$-safe and $a^{end}_*$ is relevant unbounded.
Then, $Da_*$ is simply $\gamma$-safe if and only if one of the following mutually exclusive conditions are satisfied:
\begin{enumerate}[(a)]
\item $a^{st}_{*}$ $\gamma$-irrelevant, $|Pre^+_{a^{st}_{*\gamma}}|=1$, $Pre^+_{a^{st}_{*\gamma}}\subseteq Eff^-_{a^{st}_{*\gamma}}$;
\item $a^{st}_{*}$ $\gamma$-irrelevant, $|Pre^+_{a^{st}_{*\gamma}}|=1$, $Pre^+_{a^{st}_{*\gamma}}\not\subseteq Eff^-_{a^{st}_{*\gamma}}$, $Pre^+_{a^{st}_{*\gamma}}\subseteq Eff^-_{a^{end}_{*\gamma}}\cup Eff^+_{a^{end}_{*\gamma}}$;
\item $a^{st}_{*}$ $\gamma$-irrelevant, $|Pre^+_{a^{st}_{*\gamma}}|=0$, $Pre^-_{a^{st}_{*\gamma}}\cup Eff^-_{a^{st}_{*\gamma}}\cup Eff^-_{a^{end}_{*\gamma}}\cup Eff^+_{a^{end}_{*\gamma}}=\gamma(\cT)$;
\item $a^{st}_{*}$ $\gamma$-relevant, $Eff^+_{a^{st}_{*\gamma}}\subseteq Eff^-_{a^{st}_{*\gamma}}\cup Eff^+_{a^{st}_{*\gamma}}$.
\end{enumerate}
\end{proposition}

\begin{remark}\label{rem: type a} If Condition (a) of Proposition \ref{prop: aux simple safe} holds, this implies that the same conditions needs to be satisfied by $a^{st}$, namely it holds: $|Pre^+_{a^{st}_{\gamma}}|=1$, $Pre^+_{a^{st}_{\gamma}}\subseteq Eff^-_{a^{st}_{\gamma}}$.
\end{remark}

\begin{definition}[Safe durative actions]
\label{def: durative action type}
We say that $Da_*$ is \emph{simply $\gamma$-safe of type (x)} where $x\in\{a,b,c,d\}$ if it is $\gamma$-reachable, $a^{st}_{*}$ is strongly $\gamma$-safe, $a^{end}_{*}$ is $\gamma$-relevant unbounded, and, finally,
$Da_*$ satisfies the condition (x) of Proposition \ref{prop: aux simple safe}.
\end{definition}

\begin{example}
Consider a template $\cT$ and an instance $\gamma$ such that $\gamma(\cT) = \{q, q'\}$. Figure \ref{fig:notexe} shows possible instances of actions of types (a)-(d).
\end{example}

\begin{figure}[tbh]
\centering\includegraphics[width=80mm]{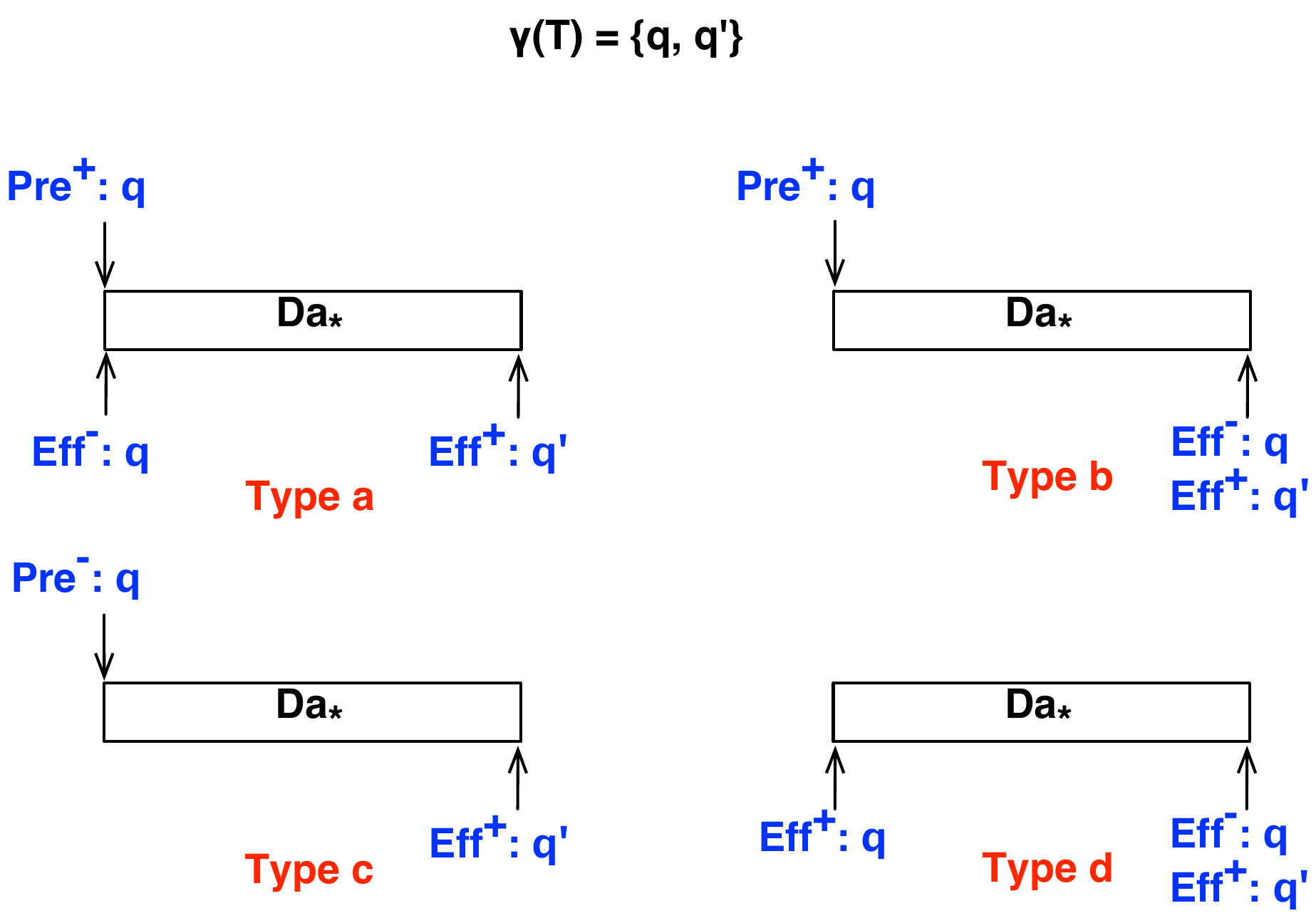}
\caption{Examples of actions of types (a)-(d).}
\label{fig:notexe}
\end{figure}

When the start or the end of a durative action $Da$ happen simultaneously with other actions, the reduction of $Da$ to $Da_*$ cannot be performed in general as shown in the following example.

\begin{example}
Consider a template $\cT$ and an instance $\gamma$ such that $\gamma(\cT) = \{q, q', q''\}$. Figure \ref{fig:simu} shows that, when the durative actions $Da$ and $Da'$ are considered in isolation, both $a^{st}_{*}$ and $a'^{st}_{*}$ are strongly safe since they are $\gamma$-unreachable. Since $a^{end}_{*}$ and $a'^{end}_{*}$ are irrelevant, $Da$ and $Da'$ are strongly safe. However, if we now consider the case in which $Da$ and $Da'$ happen simultaneously, giving rise to the sequence ${\bf A} = (A^1=\{a^{st}, a'^{st}\}, A^2=\{a^{inv}, a'^{inv}\})$, we see that ${\bf A}$ is not individually $\gamma$-safe. In fact, if we put $s^0=\{q''\}$ with $w(\cT, \gamma, s^0) = 1$, we have that $s^1=\xi(s^0,A^1) = \{q, q'', q''' \}$ with $w(\cT, \gamma, s^1) = 3$, which violates the definition of individual $\gamma$-safety.
\end{example}
\begin{figure}[tbh]
\centering\includegraphics[width=90mm]{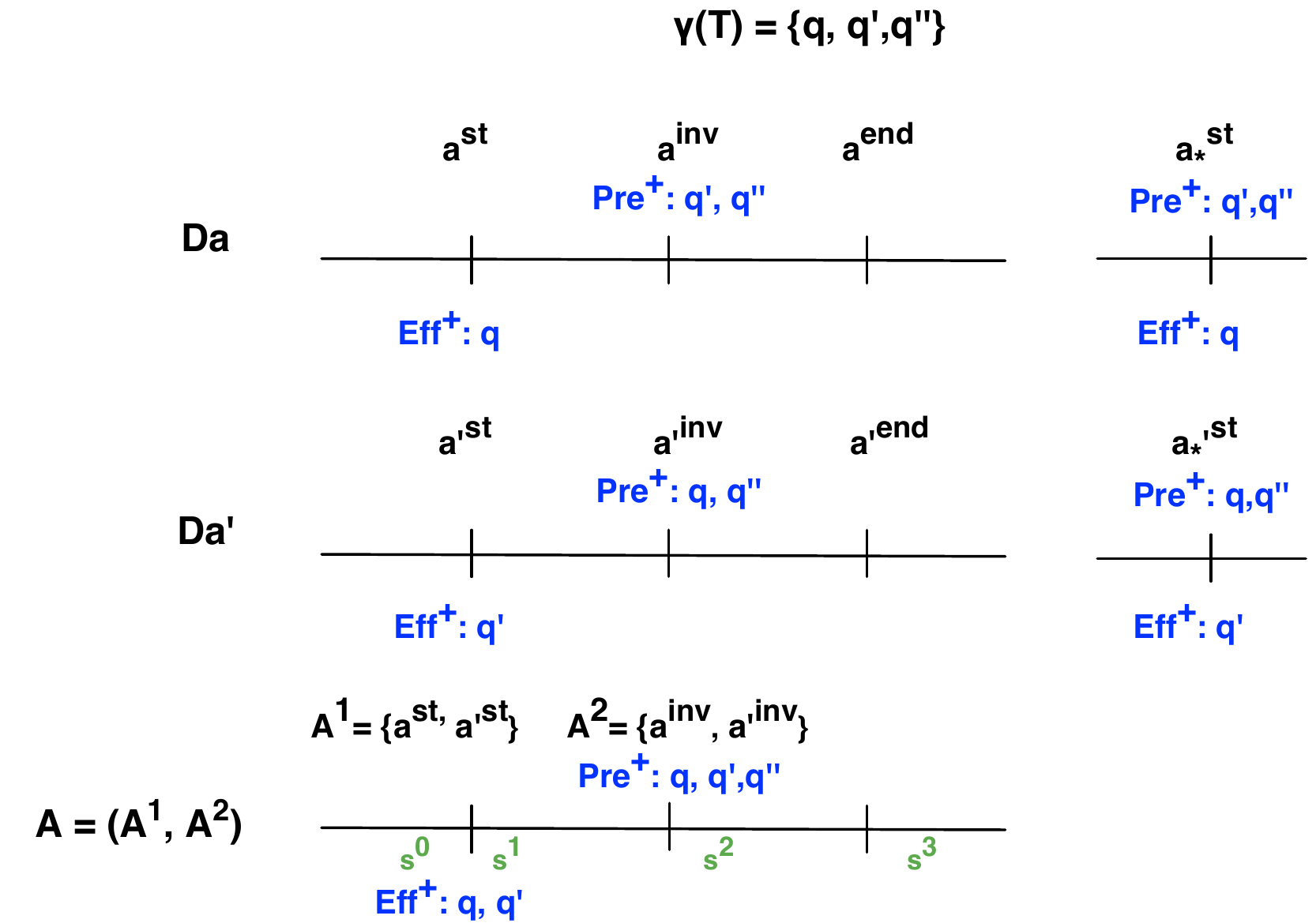}
\caption{The sequence ${\bf A}$ is not individually $\gamma$-safe.}
\label{fig:simu}
\end{figure}

Note that, in the previous example, the two durative actions are $\gamma$-unreachable. The following result shows that such pathological phenomena can only happen in that case and will be instrumental for the results of the next section.

\begin{proposition}\label{prop: start  * safe} Let $Da$ be a $\gamma$-reachable durative action such that $a^{st}$ is not strongly $\gamma$-safe, while $a^{st}_*$ is strongly $\gamma$-safe. Then,
\begin{enumerate}
\item[(i)] $a^{st}_*$ is $\gamma$-relevant bounded;
\item[(ii)] for every ground action sets $A^1$ such that $(\{a^{st}\}\cup A^1, a^{inv})$ is executable, $(\{a^{st}\}\cup A^1, a^{inv})$ is individually $\gamma$-safe.
\end{enumerate}
\end{proposition}

No similar results hold for the end parts of durative actions as next example shows. 

\begin{example}
Consider a template $\cT$ and an instance $\gamma$ such that $\gamma(\cT) = \{q, q'\}$. When the durative actions $Da$ and $Da'$ (Figure \ref{fig:ex10}) are considered in isolation, both $a^{end}_{*}$ and $a'^{end}_{*}$ are strongly $\gamma$-safe since they are $\gamma$-bounded. Since $a^{st}_{*}$ and $a'^{st}_{*}$ are irrelevant, $Da$ and $Da'$ are strongly safe. However, if $Da$ and $Da'$ happen simultaneously, giving rise to the sequence ${\bf A} = (A^1=\{a^{inv}, a'^{inv}\}, A^2=\{a^{end}, a'^{end}\})$, ${\bf A}$ is not individually $\gamma$-safe. If we put $s^0=\emptyset$ with $w(\cT, \gamma, s^0) = 0$, we have that $s^1=\xi(s^0,A^1) = \emptyset$ and $s^2=\xi(s^1,A^2) = \{q, q'\}$ with $w(\cT, \gamma, s^2) = 2$, which violates the definition of individual $\gamma$-safety.
\end{example}
\begin{figure}[tbh]
\centering\includegraphics[width=80mm]{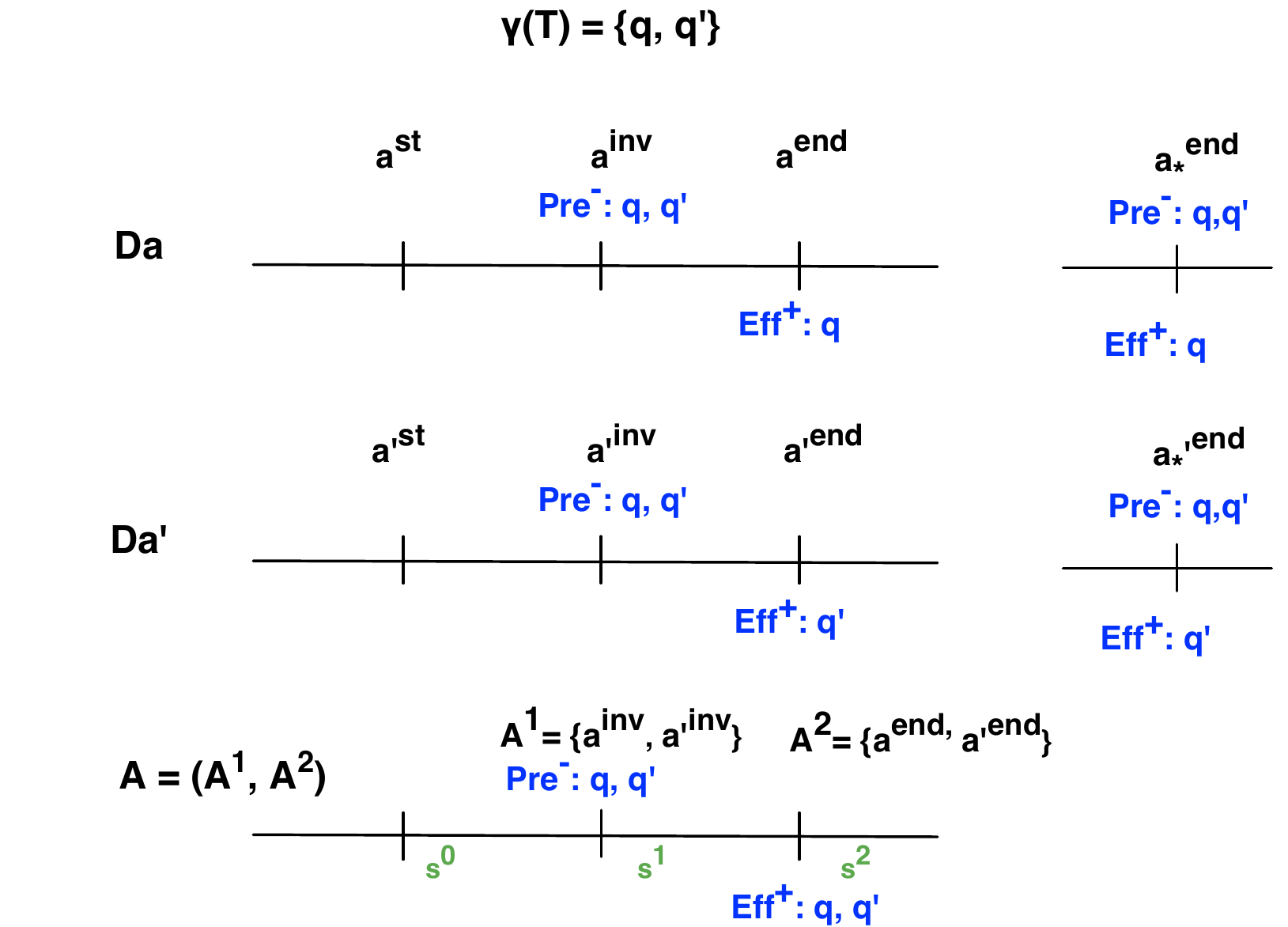}
\label{fig:ex10}
\caption{The sequence ${\bf A}$ is not individually $\gamma$-safe.}

\label{fig:endparts}
\end{figure}

\section{Conditions for the invariance of a template}
\label{sec: cond inv}
Any plan $\pi$ where all instantaneous ground actions are strongly safe, all durative ground actions are safe and take place in isolation, i.e. with no other actions happening in between them, yields a safe happening sequence ${\bf A}_{\pi}$, as a consequence of Corollary \ref{cor: ind safe}. The difficulty, in general, is that durative actions can in principle start or end together and be intertwined with other instantaneous or durative actions. 
Safety of durative actions must therefore be accompanied by suitable hypothesis guaranteeing that dangerous intertwinements or simultaneous happenings cannot take place in valid plans. In this way, we can work out sufficient conditions for the invariance of a template, which will be useful in analysing concrete examples.

In this section, we present two results that give \emph{sufficient} conditions for the invariance of a template. The first deals with the particular case when all instantaneous actions are strongly safe and all durative actions $Da$ are such that $Da_*$ is strongly safe. The second result considers a more general case when there are durative actions $Da$ for which $Da_*$ is only simply safe. We recall our standing assumption that $w(\cT, \gamma, Init)\leq 1$ for every $\gamma$.

Given a template $\cT$ and an instance $\gamma$, we denote by $\cG\cA^d(\gamma)$ the collection of durative actions which are not strongly $\gamma$-safe and with $\cG\cA^{st}(\gamma)$ and $\cG\cA^{end}(\gamma)$, respectively, the collection of their start and end fragments. The following property prevents the simultaneous end of durative actions that could yield unsafe phenomena.

\begin{definition}[Relevant right isolated actions]\label{def: relevant right isolated}
Given a template $\cT$, the set of ground durative actions $\cG\cA^d$ is said to be \emph{relevant right isolated} if, for every instance $\gamma$ and for every $Da^1, Da^2\in\cG\cA^d(\gamma)$, one of the following conditions is satisfied:
\begin{enumerate}[(i)]
\item $|Eff^+_{a^{1end}_\gamma}\cup Eff^+_{a^{2end}_\gamma}|\leq 1$;
\item at least one of the two pairs $\{a^{1end}, a^{2end}\}$ or $\{a^{1inv}, a^{2inv}\}$ is mutex;
\item if they are both non-interfering, $(\{a^{1inv}, a^{2inv}\}, \{a^{1end}, a^{2end}\})$ is $\gamma$-unreachable.
\end{enumerate}
\end{definition}

\begin{theorem}
\label{theo: *safety} 
Consider a template $\cT$ and suppose that the set of instantaneous actions $\cG\cA^i$ and that of durative actions $\cG\cA^d$ satisfy the following properties:
\begin{enumerate}[(i)]
\item every $a\in\cG\cA^i$ is strongly safe;
\item for every instance $\gamma$ and every $Da\in\cG\cA^d(\gamma)$,  $Da_*$ is $\gamma$-reachable and strongly $\gamma$-safe;
\item $\cG\cA^d$ is relevant right isolated.
\end{enumerate}
Then, $\cT$ is invariant.
\end{theorem}

Note that assumption (iii) in the statement of Theorem \ref{theo: *safety} is to exclude the simultaneous end of durative actions; if such phenomena can be excluded a-priori, the hypothesis can be removed.

When there are durative ground actions $Da$ for which $Da_*$ is not strongly $\gamma$-safe, further hypotheses are needed in order to guarantee that the template $\cT$ is invariant. The main point is that, in this case, not only simultaneity can be harmful, but also any intertwinement between such a durative action and other actions. The following examples show the type of phenomena that can happen and that any theorem extending Theorem \ref{theo: *safety} needs to prevent.

\begin{example}
Consider a template $\cT$ and an instance $\gamma$ such that $\gamma(\cT) = \{q, q', q''\}$. Both the durative actions $Da$ and $Da'$ (Figure \ref{fig:ex11}) are $\gamma$-safe. However, they can intertwine in such a way to give rise to a sequence that is individually unsafe: ${\bf A} = (A^1=\{a^{st}\}, A^2=\{a'^{st}\}, A^3=\{a^{end}\}, A^4=\{a'^{end}\})$. If we put $s^0=\{q\}$ with $w(\cT,\gamma,s^0)=1$, we have that $s^4=\{q',q''\}$ with $w(\cT,\gamma,s^4)=2$.
\end{example}
\begin{figure}[tbh]
\centering\includegraphics[width=100mm]{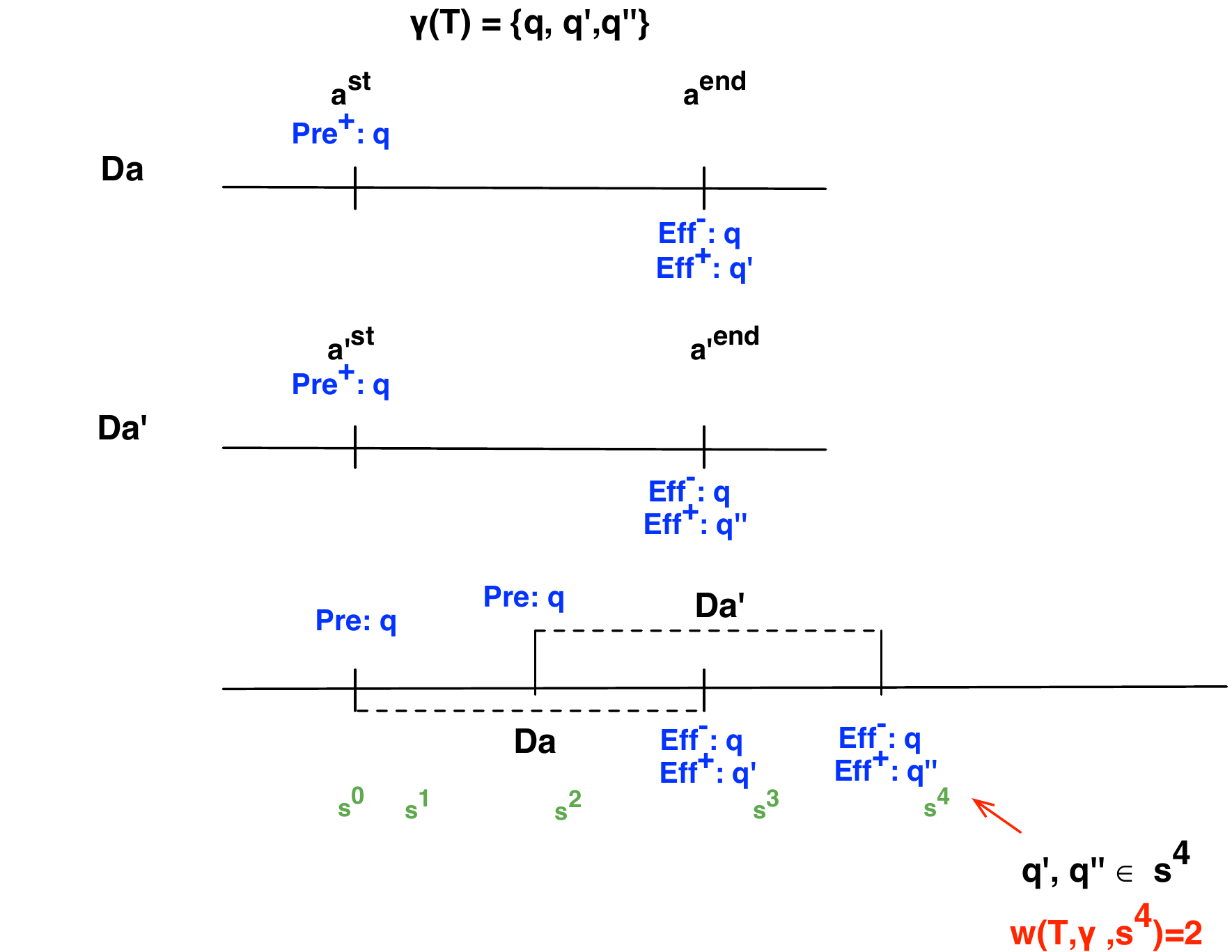}
\caption{Schemas $Da$ and $Da'$ can intertwine in such a way to give rise to a sequence that is individually unsafe.}
\label{fig:ex11}
\end{figure}

The following definition describes a set of durative actions for which such phenomena cannot take place. It consists of three requirements acting, for each instantiation $\gamma$, on the subset of dangerous durative actions $\cG\cA^d(\gamma)$. The first prevents the simultaneous happening of two start fragments of such durative actions. The second states that, between two successive start fragments of durative actions in $\cG\cA^d(\gamma)$, there must be the end of a third action also in $\cG\cA^d(\gamma)$. Finally, the third requirement prevents $\gamma$-relevant actions to happen in between a durative action in $\cG\cA^d(\gamma)$.

\begin{definition}[Relevant non intertwining actions]\label{def: non intertwining}
Given a template $\cT$, the set of ground durative actions $\cG\cA^d$ is said to be \emph{relevant non intertwining} if, for every instance $\gamma$, every $Da\in\cG\cA^d(\gamma)$ and for every $\gamma$-reachable $Da$-admissible sequence of actions
\begin{equation}\label{def: rel isol}{\bf A}=(\{a^{st}\}\cup A^1, A^2, \dots , A^{n-1}, \{a^{end}\}\cup A^n)\,,\end{equation}
the following conditions are satisfied:
\begin{enumerate}[(i)]
\item $A^1\cap \cG\cA^{st}(\gamma)=\emptyset$;
\item If $A^1=\emptyset$ and $b\in A^j\cap \cG\cA^{st}(\gamma)$ for $j<n$, then there exists $b'\in A^{j'}\cap \cG\cA^{end}(\gamma)$ for some $0<j'\leq j$;
\item If $A^1=\emptyset$ and $A^j\cap (\cG\cA^{st}(\gamma)\cup \cG\cA^{end}(\gamma))=\emptyset$ for every $j=2,\dots ,{n-1}$, then each $A^j$ is $\gamma$-irrelevant for $j=2,\dots ,{n-1}$.
\end{enumerate}
\end{definition}

We are now ready to state and prove the main result of this section, that expresses a sufficient condition for a template to be invariant. 

\begin{theorem} 
\label{theo: non intertwining} 
Consider a template $\cT$ and suppose that the set of instantaneous actions $\cG\cA^i$ and that of durative actions $\cG\cA^d$ satisfy the following properties:
\begin{enumerate}[(i)]
\item every $a\in\cG\cA^i$ is strongly safe;
\item for every $Da\in\cG\cA^d$, $Da_*$ is safe;
\item the set $\cG\cA^d$ is relevant non-intertwining.
\end{enumerate}
Then, $\cT$ is invariant.
\end{theorem}

\begin{figure}[tbh]
\centering\includegraphics[width=70mm]{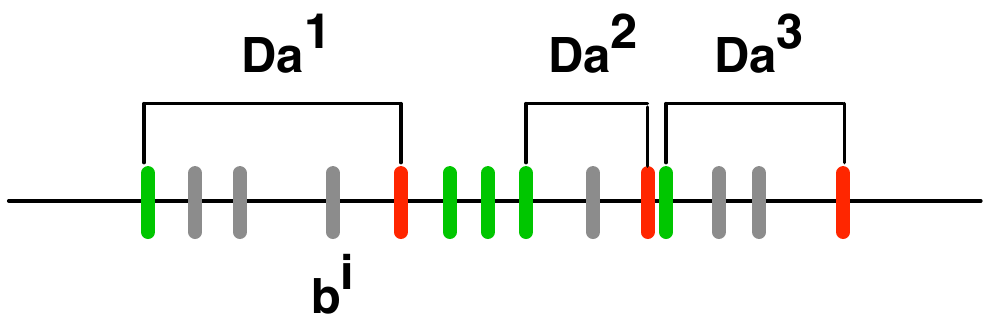}
\caption{Structure of a plan $\tilde\pi$ as constructed in Theorem \ref{theo: non intertwining}. Strongly safe actions are indicated in green, relevant in red and irrelevant in grey.}
\label{fig:planstruct}
\end{figure}


The properties that the set of durative actions $\cG\cA^d$ needs to satisfy to be relevant non intertwining, which are expressed in Definition \ref{def: non intertwining}, are in general difficult to check as they require to consider sequences of actions of possibly any length. Below we propose a sufficient condition that guarantees such properties to hold, which is much simpler and suitable to be later analysed at the lifted level of action schemas.

We start with two definitions. The first is a left version of the relevant right isolated property. It prevents dangerous durative actions to start simultaneously. It is needed to insure condition (i) of Definition \ref{def: non intertwining} of relevant non intertwining. The second definition allows us to reformulate conditions (ii) and (iii) of Definition \ref{def: non intertwining}.

\begin{definition}[Relevant left isolated actions]\label{def: relevant left isolated}
Given a template $\cT$, the set of ground durative actions $\cG\cA^d$ is said to be \emph{relevant left isolated} if, for every instance $\gamma$ and for every $Da^1, Da^2\in\cG\cA^d(\gamma)$, one of the following conditions is satisfied:
\begin{enumerate}[(i)]
\item at least one of the two pairs  $\{a^{1st}, a^{2st}\}$ or  $\{a^{1inv}, a^{2inv}\}$ is mutex;
\item if they are both non-interfering, $(\{a^{1st}, a^{2st}\}, \{a^{1inv}, a^{2inv}\})$ is $\gamma$-unreachable.
\end{enumerate}
\end{definition}

\begin{definition}[Irrelevant unreachable actions]
\label{def: irrelevant unreachable} Consider a template $\cT$ and an instance $\gamma$. A pair of actions $(a,a')$ is $\gamma$\emph{-irrelevant unreachable} if any sequence of actions
$${\bf A}=(\{a\}, A^2, \dots , A^{n-1}, \{a'\})$$
such that $A^2,\dots ,A^{n-1}$ are $\gamma$-irrelevant, is $\gamma$-unreachable.
\end{definition}

The next result expresses a sufficient condition for the set $\cG\cA^d$ to be relevant non intertwining.

\begin{proposition}\label{prop: relevant non inter} Consider a template $\cT$.  The set $\cG\cA^d$ is relevant non intertwining if the following conditions are satisfied:
\begin{enumerate}[(i)]
\item  $\cG\cA^d$ is relevant left isolated;
\item  for every instance $\gamma$, for every $Da\in \cG\cA^{d}(\gamma)$, and for every $a'\in \cG\cA\setminus \cG\cA^{end}(\gamma)$ that is not $\gamma$-irrelevant or $a'\in \cG\cA^{st}(\gamma)$, $\{a^{inv}, a'\}$ is mutex or the pair $(a^{st}, a')$ is $\gamma$-irrelevant unreachable.
\end{enumerate}
\end{proposition}

The property of $\gamma$-irrelevant unreachable, though conceptually simpler than the original properties required in the definition of relevant non-intertwining actions, is still too complex for practical implementation, as it requires to verify properties over sequences of undefined length. We now propose a stronger version of it that is instead of simple computational complexity (linear in the number of ground actions). 

\begin{definition}[Strongly irrelevant unreachable actions]
\label{def: strongly del inter} 
A pair of actions $(a,a')$ is \emph{strongly} $\gamma$\emph{-irrelevant unreachable} if any of the following conditions are satisfied:
\begin{enumerate}[(i)]
\item there exists $q\in \Gamma_a^+\cap Pre_{a'}^-$ such that, for every $a''$ that is $\gamma$-irrelevant, $q\not\in Eff^-_{a''}$;
\item there exists $q\in \Gamma_a^-\cap Pre_{a'}^+$ such that, for every $a''$ that is $\gamma$-irrelevant, $q\not\in Eff^+_{a''}$;
\item $|Pre^+_{a_{\gamma}}\cup (Pre^+_{a'_{\gamma}}\setminus Eff^+_{a_\gamma})|>1$.
\end{enumerate}
\end{definition}
The first condition essentially says that the application of the action $a$ leads to a state containing a ground atom $q$ that needs to be false in order to then apply $a'$ and that there is no $\gamma$-irrelevant action that can make this atom false. The second condition is the analogous of the first, but exchanges the role of true and false atoms. Finally, the third condition is equivalent to require that $(a,a')$ is a $\gamma$-unreachable pair, assuming that it is executable. 

\begin{proposition}\label{prop: irr unr} 
If a pair of actions $(a,a')$ is strongly $\gamma${-irrelevant unreachable}, it is also $\gamma${-irrelevant unreachable}.
\end{proposition}
%
%
%

Based on the previous results, we conclude with a simple sufficient condition for the invariance, which is very useful in analysing concrete cases.

\begin{corollary}\label{cor: type a} Consider a template $\cT$ and suppose that, for every instance $\gamma$,
\begin{itemize}
\item every $Da\in \cG\cA^d(\gamma)$ is such that $Da_*$ is simply $\gamma$-safe of type (a);
\item every $a\in\cG\cA\setminus (\cG\cA^{st}(\gamma)\cup \cG\cA^{end}(\gamma))$ is either $\gamma$-irrelevant or $\gamma$-relevant balanced.
\end{itemize}
Then, $\cT$ is invariant.
\end{corollary}

\section{Safety of Action Schemas for a Template}
\label{sec:action}

In Section \ref{sec: cond inv}, we have established two results guaranteeing the invariance of a template, Theorems \ref {theo: *safety} and \ref{theo: non intertwining}. To be applied, they both need to check that all instantaneous and durative ground actions satisfy a safety condition as well as that other extra properties, which prevent potentially dangerous simultaneous happenings or intertwinements among actions, hold true. Since we aim to find invariants off-line quickly and efficiently, our algorithm does not work at the level of ground actions. Instead, it reasons at the lifted level and uses the structure of the action schemas, i.e. their conditions and effects, to decide whether the ground instantiations of these schemas are safe or not. Our main goal in this section is to obtain lifted versions of Theorems \ref {theo: *safety} and \ref{theo: non intertwining} and Corollary \ref{cor: type a}. 

In general, we call \emph{liftable} a property P of ground actions if, given an action schema $\alpha$, if one instantiation $a^*=gr^*(\alpha)$ satisfies P, then all instantiations $a=gr(\alpha)$ satisfy P. In this case, we say that the action schema $\alpha$ satisfies property P.

The results presented in this and the next sections achieve two main goals. On the one hand, they show that the properties of safety introduced for instantaneous and durative ground actions in Sections \ref{sec:safeness inst actions} and \ref{sec:safeness sequences} are liftable as well the non-intertwining properties, even if in a weaker sense, behind the formulation of Theorems \ref {theo: *safety} and \ref{theo: non intertwining}. On the other hand, they will give efficient characterisations of such properties at the lifted level, which we use in our algorithmic implementation (see Section \ref{sec:algorithm}).

In the remaining part of this section, we analyse instantaneous action schemas and their ground instantiations. We show that strong safety is liftable and work out a complete characterisation of this property at the lifted level. Next section is devoted to lifting properties for durative actions.

\subsection{Structure and properties of action schemas}

We start with the following definition that introduces the key concept of matching. It couples an action schema to a template and allows us to understand if, in the ground world, a ground literal appearing in an action schema is or is not in $\gamma(\cT)$.

\begin{definition}[Matching]
\label{def:match}
Given a template $\cT = ( \cC, \cF_{\cC})$ and an action schema $\alpha \in \cA$, a literal $l$ that appears in $\alpha$ such that it exists a template's component $c = \langle r, a, p \rangle \in \cC$ with $\rela[l] = \langle r, a \rangle$ and, if $l$ is universally quantified, $\varsq[l] = \{ p \}$ is said to \emph{match} $\cT$ via the component $c$. Given two literals $l$ and $l'$, we say that they are $\cT$-\emph{coupled} (and we write $l\sim_{\cT}l'$)  if the following two conditions hold:
\begin{enumerate}[(i)]
\item $l$ and $l'$ individually match $\cT$ via the components $c$ and $c'$; 
\item if $(c, i) \sim_ {\cF_{\cC}}(c', j)$, $\args[i,l] = \args[j, l']$.
\end{enumerate}
\end{definition}

We now fix a template $\cT$ and an action schema $\alpha$ and study the properties of the relation $\sim_{\cT}$ on the literals of $\alpha$ that match $\cT$, introduced above. First, it is useful to work out more concrete representations for literals: this is the content of next Remark.

\begin{remark}
\label{rem:matching}
Suppose that $l$ is a literal in the action schema $\alpha$ that matches the template $\cT$ via the component $c$. The corresponding relation $r$ will necessarily have the structure $r(x_1,\dots x_k, v)$ where $x_j$'s denote the fixed arguments, $v$ the counted argument (which could also be absent) and $l=r(a_1,\dots a_k, v)$ or $l=r(a_1,\dots a_k, a_{k+1})$ depending if, respectively, $l$ is universally quantified or simple, and where $a_1,\dots ,a_k, a_{k+1}$ are free arguments. 
Suppose now that $l_1$ and $l_2$ are two literals in the action schema $\alpha$ that match the template $\cT$ via $c_1$ and $c_2$, respectively. Up to a permutation of the position of the fixed arguments, the corresponding relations $r_1$ and $r_2$ can be written as,  $r_i(x_1^i,\dots x_k^i, v^i)$ for $i=1,2$ where the fixed arguments satisfy the relations $x_j^1\sim_ {\cF_{\cC}}x_j^2$ for every $j$. 
If, moreover, $l_1\sim_{\cT}l_2$, the two literals will have the form $l_i=r_i(a_1,\dots a_k, a_{k+1})$ (or $l_i= \forall v:\;  r_i(a_1,\dots a_k, v)$), where $a_1,\dots ,a_k, a_{k+1}$ are the free arguments. 
\end{remark}

\begin{proposition}  Given a template $\cT$ and an action schema $\alpha$, $\sim_{\cT}$ is an equivalence relation.
\end{proposition}
\begin{proof} The only property to be checked is transitivity and this is evident from the equivalent description of the relation $\sim_{\cT}$ proposed in Remark \ref{rem:matching}.
\end{proof}

An equivalence class of literals with respect to $\sim_{\cT}$ is called a $\cT$-\emph{class}.

We now consider a grounding function $gr$ for $\alpha$ and an instance $\gamma$ for $\cT$. We recall the standing assumption that both $gr$ and $\gamma$ are injective maps (this will be often used in what follows). If $l$ is a literal in $\alpha$ that matches $\cT$,
the subset of ground atoms $gr(l)$ is either a subset of $\gamma(\cT)$ or it must have empty intersection with $\gamma(\cT)$. This is simply because, both sets are closed under any modification of the assignment of the counted argument. This motivates the following definition:

\begin{definition}[Coherence] $gr$ and $\gamma$ are coherent over $l$ if $gr(l)\subseteq \gamma(\cT)$.
\end{definition}

Coherence is more concretely described in the following Remark.

\begin{remark}
\label{rem:coherent} Suppose that $l$ matches $\cT$ via the component $c$ whose relation is $r$. It follows from the considerations in Remark \ref{rem:matching} that, depending if $r$ posses a counted variable or not and if $l$ is simple or universally quantified, $r$, $l$ and $gr(l)$ take the following forms:
$$\begin{array}{lll}r(x_1,\dots x_k)\quad &l=r(a_1,\dots a_k)\quad &gr(l)=\{r(gr(a_1),\dots ,gr(a_k))\}\\
r(x_1,\dots x_k, v)\quad &l=r(a_1,\dots a_k, a_{k+1})\quad &gr(l)=\{r(gr(a_1),\dots ,gr(a_k),gr(a_{k+1}))\}\\
r(x_1,\dots x_k, v)\quad &l=\forall v:\; r(a_1,\dots a_k, v)\quad &gr(l)=\{r(gr(a_1),\dots gr(a_k), o),\,|\, o\in \cO\}\end{array}
$$
Note that, in all cases, the coherence condition $gr(l)\subseteq \gamma(\cT)$ is equivalent to require that:
\begin{equation}\label{coherent}gr(a_j)=\gamma(x_j),\;\forall j=1,\dots ,k\end{equation}
\end{remark}

The following result is immediate from the conditions (\ref{coherent}):
\begin{proposition} Let $l$ be a literal of the action schema $\alpha$. Then, for every grounding function $gr$, it is possible to find an instance $\gamma$ such that $gr$ and $\gamma$ are coherent over $l$ and viceversa.
\end{proposition}

\begin{lemma}
\label{lemma:g-i}
Assume that $gr$ and $\gamma$ are coherent over a literal $l_1$ of $\alpha$ and let $l_2$ be another literal in $\alpha$ that matches $\gamma$. Then,
$gr$ and $\gamma$ are coherent over $l_2$ if and only if $l_2\sim_{\cT}l_1$.
\end{lemma}

The following result immediately follows from the definition of coherence and Lemma \ref{lemma:g-i}. 
\begin{proposition}
\label{cor:g-i}
Suppose that $M$ is a subset of literals appearing in $\alpha$. Then, $gr(M)\cap \gamma(\cT) =gr(M\cap L)$ where $L$ is the $\cT$-class of literals of $\alpha$ on which $gr$ and $\gamma$ are coherent.
\end{proposition}

Proposition \ref{cor:g-i} has an important practical consequence. Once $gr$ and $\gamma$ have been fixed, only the part of $\alpha$ made of literals in the class $L$ where $gr$ and $\gamma$ are coherent affect the part of state dynamics concerning the set $\gamma(\cT)$. Precisely, if $a=gr(\alpha)$, it follows from the definition of $a_{\gamma}$ (see Remark \ref{remark: split}) that: 
$$Pre^{\pm}_{a_{\gamma}}=gr(Pre^{\pm}_{\alpha}\cap L),\quad Eff^{\pm}_{a_{\gamma}}=gr(Eff^{\pm}_{\alpha}\cap L)$$
Considering that, by Proposition \ref{prop: equiv safe}, $a$ is strongly $\gamma$-safe if and only if $a_\gamma$ is also strongly safe, the property of strong safety of an action schema $\alpha$ does not depend on the literals in $\alpha$ that do not match $\cT$. Hence, in principle, such a property should be analysed by studying the restrictions of $\alpha$ to the different $\cT$-classes $L$ of matching literals. This intuition leads to the following definition.

\begin{definition}[Pure Action Schemas]
\label{def:pureactionschema}
Given a template $\cT$, an action schema $\alpha$ and a $\cT$-class $L$ of literals in $\alpha$, we define $\alpha_L$ to be the action schema where we only consider literals belonging to $L$. More precisely, $\alpha_L$ is the action schema such that
$$Pre^{\pm}_{\alpha_L} =Pre^{\pm}_{\alpha}\cap L,\quad Eff^{\pm}_{\alpha_L} =Eff^{\pm}_{\alpha}\cap L$$
We call $\alpha_L$ a \emph{pure action schema}. 
\end{definition}

\begin{example}[\sfs{Floortile} domain]
\label{ex:pure}
Consider the template $\cT_{ft}$ given in Example \ref{ex:long} and the action schema $\alpha=${\tt paint-up}$^{st}$: $Pre_{\alpha}^+=\{$\url{robot-at(r,x)}, \url{clear(y)}$\}$, $Eff_{\alpha}^-=\{$\url{clear(y)}$\}$.

Note that both literals \url{robot-at(r,x)} and \url{clear(y)} in $\alpha$ match $\cT_{ft}$ and form two different $\cT$-classes because they do not satisfy condition (ii) in Definition \ref{def:match}: $L_1 = \{$\url{robot-at(r,x)}$\}$ and $L_2 = \{$\url{clear(y)}$\}$.

Consider the instance $\gamma_1$ that associates \url{tile1} to each fixed argument in the components of $\cT_{ft}$ and grounding function $gr({\tt r})$=\url{rbt1}, $gr({\tt x})$ = \url{tile1} and $gr({\tt y})$ = \url{tile2}. In this case, $gr$ and $\gamma_1$ are coherent on the $\cT$-class $L_1$.

We have two pure action schemas corresponding to $\alpha$: $\alpha_{L_{1}}$ and $\alpha_{L_2}$. $\alpha_{L_{1}}$ has the following specification: $Pre_{\alpha_{L_1}}^+=\{{\tt robot-at(r,x)}\}$ and $\alpha_{L_{2}}$: $Pre_{\alpha_{L_2}}^+=\{ {\tt clear(y)}\}, Eff_{\alpha_{L_2}}^-=\{{\tt clear(y)}\}$.
\end{example}

\subsection{Pure Action Schema Classification} \label{Classification}

We now carry on a detailed analysis of pure action schemas, showing in particular how the check for strong safety for a ground action $a=gr(\alpha)$ can be efficiently performed at the lifted level working with the different pure action schemas $\alpha_L$.

We fix an action schema $\alpha$ and a $\cT$-class $L$ of its literals.
First, we introduce a concept of weight at the level of literals in $L$ that allows us to distinguish between simple and universally quantified literals. Precisely, given $l\in L$, we put $w_l=1$ if $l$ is simple, while $w_l=\omega$ if $l$ is universally quantified and where $\omega =|\cO|$. Given a subset $A\subseteq L$, we define $w(A)=\sum_{l\in A}w_l$. Note that $w(\cdot )$ simply coincides with the notion of cardinality in case all literals in $L$ are simple. If we consider a grounding function $gr$ for $\alpha$, then for every subset $A\subseteq L$, it holds:
\begin{equation}\label{weight eq1} |gr(A)|=w(A)\end{equation}
Similarly, if $c$ is a component of $\cT$, we define $w_c$ equal to $1$ or to $\omega$ if $c$, respectively, does not have or does have a counted variable.

We need a last concept:

\begin{definition}[Coverage] Given a component $c\in \cT$, we let $L_c$ to be the subset of literals in $L$ that match $\cT$ through the component $c$. A subset of literals $M\subseteq L$ is said to \emph{cover} the component $c$, if $w(M\cap L_c)\geq w_c$. $M$ is said to \emph{cover} $\cT$, if $M$ covers every component $c\in\cT$.
\end{definition}

\begin{remark}\label{rem: coverage}
If we consider a component $c\in\cT$, we have that all ground atoms generated by $c$ are in $gr(M)$ if and only if $M$ covers $c$. In particular, $\gamma(\cT)=gr(M)$ if and only if $M$ covers $\cT$.
\end{remark}

We now propose a classification of the pure action schemas $\alpha_L$, formally analogous to the one introduced for action sets in Definitions \ref{def: classification} and \ref{def:relevant action sets}: we simply replace preconditions and effects of $a_{\gamma}$ with those of $\alpha_L$ and the concept of cardinality with that of weight.

\begin{definition}[Classification of Pure Action Schemas]
\label{def:weight class}
The pure action schema $\alpha_L$ is:
\begin{itemize}
\item \emph{unreachable} for $\cT$ if $w(Pre^+_{\alpha_L})\geq 2 $;
\item \emph{heavy} for $\cT$ if $w(Pre^+_{\alpha_L})\leq 1$ and $w(Eff^+_{\alpha_L})\geq 2$;
\item \emph{irrelevant} for $\cT$ if $w(Pre^+_{\alpha_L})\leq 1$ and $w(Eff^+_{\alpha_L})=0$;
\item \emph{relevant} for $\cT$ if $w(Pre^+_{\alpha_L})\leq 1$ and $w(Eff^+_{\alpha_L})=1$.
\end{itemize}
\end{definition}

\begin{definition}[Classification of Relevant Action Schemas]
\label{def:relevant actions}
The pure relevant action schema $\alpha_L$ is:
\begin{itemize}
\item \emph{balanced} for $\cT$ if $w(Pre^+_{\alpha_L})= 1$ and $Pre^+_{\alpha}\subseteq Eff^+_{\alpha_L}\cup Eff^-_{\alpha_L}$;
\item \emph{unbalanced} for $\cT$ if $w(Pre^+_{\alpha_L})= 1$ and $Pre^+_{\alpha}\cap ( Eff^+_{\alpha_L}\cup Eff^-_{\alpha_L})=\emptyset$;
\item \emph{bounded} for $\cT$ if $w(Pre^+_{\alpha_L})= 0$ and $L$ covers $\cT$;
\item \emph{unbounded} for $\cT$ if $w(Pre^+_{\alpha_L})= 0$ and $L$ does not cover $\cT$.
\end{itemize}
\end{definition}

The following result clarifies the relation with the corresponding ground actions.

\begin{proposition}\label{prop: pure schema} Consider an action schema $\alpha$, a $\cT$-class $L$ of its literals, a grounding function $gr$ and an instance $\gamma$ coherent over $L$. Put $a=gr(\alpha)$. Then, $\alpha_L$ satisfies any of the properties expressed in Definitions \ref{def:weight class} and \ref{def:relevant actions} if and only if $a$ satisfies the corresponding $\gamma$-property as defined in Definitions \ref{def: classification} and \ref{def:relevant action sets}.
\end{proposition}
\begin{proof}
Immediate consequence of the fact that $a_{\gamma}=gr(\alpha_L)$, of equation (\ref{weight eq1}), and of Remark \ref{rem: coverage}.
\end{proof}

We are now ready to propose the following final result concerning strong safety of general action schemas. It shows how strong safety can be seen as a property of an action schema and can be studied by analysing its pure parts. 

\begin{corollary}
\label{cor: strongly safe final} Strong safety is a liftable property. Moreover, an action schema $\alpha$  is strongly safe if and only if, for every $\cT$-class of literals $L$ of $\alpha$, $\alpha_L$ is unreachable, irrelevant, relevant balanced or relevant bounded.
\end{corollary}
%

\begin{example}[\sfs{Floortile} domain]
Consider the template $\cT_{ft}$ and the action schema $\alpha=$ {\tt paint-up}$^{st}$ given in Example \ref{ex:pure}.
The two pure action schemas $\alpha_{L_{1}}$ and $\alpha_{L_2}$ are both irrelevant and hence strongly safe. Hence,  $\alpha$ is strongly safe.

Now consider the action schema $\alpha'=${\tt paint-up}$^{end}$ with specification: $Eff_{\alpha'}^+=\{$\url{painted(y, c)}$\}$. This is a pure action schema. It is relevant unbounded and thus not strongly safe.
\end{example}

An immediate consequence of Corollary \ref{cor:inv} is:
\begin{corollary}
\label{cor: all strongly safe}
Given a template $\cT$, $\cT$ is invariant if for each $\alpha \in \cA$, $\alpha$ is strongly safe.
\end{corollary}

\section{Durative action schemas}
\label{sec: dur schemas}
Our goal now is to work out proper lifted versions of the properties of durative actions given in Section \ref{sec:safeness sequences}, in particular those involved in the statement of our main results, Theorems \ref{theo: *safety} and \ref{theo: non intertwining}. Some of these properties concern just one durative action (e.g. safety), while others involve more actions (e.g. non-interfering, irrelevant unreachable).
We start analysing the first type of properties, presenting, in particular, an explicit characterisation of safety for durative actions at the lifted level.

We always use the following notation. Given a durative action schema $D\alpha=(\alpha^{st}, \alpha^{inv}, \alpha^{end})$ and a grounding function $gr$ for $D\alpha$, we put 
$Da=gr(D\alpha)$, where $Da=(a^{st}, a^{inv}, a^{end})$ with $a^{st}=gr(\alpha^{st})$, $a^{inv}=gr(\alpha^{inv})$, and $a^{end}=gr(\alpha^{end})$. Also, we define the auxiliary durative action schema $D\alpha_*=(\alpha^{st}_*, \alpha^{end}_*)$
where $\alpha_*^{st}$ and $\alpha^{end}_*$ are the instantaneous action schema such that:
$$\begin{array}{ll} Eff^{\pm}_{\alpha_*^{st}}=Eff^{\pm}_{\alpha^{st}},\quad &Pre^{\pm}_{\alpha_*^{st}}=Pre^{\pm}_{\alpha^{st}}\cup (Pre^{\pm}_{\alpha^{inv}}\setminus Eff^{\pm}_{\alpha^{st}})\\
Eff^{\pm}_{\alpha_*^{end}}=Eff^{\pm}_{\alpha^{end}},\quad &Pre^{\pm}_{\alpha_*^{end}}=Pre^{\pm}_{\alpha^{end}}\cup Pre^{\pm}_{\alpha^{inv}}
\end{array}
$$
$Da_*=gr(D\alpha_*)$ is the corresponding ground auxiliary action already defined in Section \ref{subsec:sgda}.

\subsection{Safety of durative action schemas}

We now fix a template $\cT$ and start to analyse safety. We consider a durative action schema $D\alpha$, its auxiliary action schema $D\alpha^*$ and its groundings $Da=gr(D\alpha)$ and $Da_*= gr(D\alpha_*)$.
Strong safety for durative actions reduces to check strong safety of its components and it is thus a liftable property. We can thus talk about the strong safety of $D\alpha$ or $D\alpha_*$: this is equivalent to the strong safety of all its groundings, $Da=gr(D\alpha)$ or, respectively, $Da_*= gr(D\alpha_*)$. Check of such property at the lifted level can be done using Corollary \ref{cor: strongly safe final} for the starting and ending fragments.

We now want to characterise simple safety of the auxiliary durative action $Da_*=gr(D\alpha_*)$ at the lifted level. First we consider executability.

Define, for a generic action schema $\alpha$, the subsets
$$\Gamma^+_\alpha:=(Pre^+_\alpha\setminus Eff^-_\alpha)\cup Eff^+_\alpha,\quad \Gamma^-_\alpha:=(Pre^-_\alpha\setminus Eff^+_\alpha)\cup Eff^-_\alpha$$
We have the following result:
\begin{proposition}\label{prop: exec two schema} Executability of auxiliary durative actions is a lifted property. Precisely, $D\alpha_*$ is executable if and only if 
\begin{equation}\label{eq: exec}\Gamma^+_{\alpha^{st}_*}\cap Pre^-_{\alpha^{end}_*}=\emptyset = \Gamma^-_{\alpha^{st}_*}\cap Pre^+_{\alpha^{end}_*}\end{equation}
\end{proposition}
\begin{proof} Immediate consequence of Proposition \ref{prop: exec two}.
\end{proof}

Assume now $D\alpha_*$ to be executable.
Fix an instance $\gamma$ and let $L$ be the the $\cT$-class of literals in $D\alpha$ on which $gr$ and $\gamma$ are coherent. 
Put $D\alpha_L=(\alpha^{st}_L, \alpha^{inv}_L, \alpha^{end}_L)$ and $D\alpha_{*L}=(\alpha^{st}_{*L}, \alpha^{end}_{*L})$.
Note that $Da_{*\gamma}=gr(D\alpha_{*L})$. Therefore, since simple $\gamma$-safety of $Da_*$ only depends on $Da_{*\gamma}$ (since executability has already been assumed), we expect that such property can be formulated in terms of the pure auxiliary durative action schema $D\alpha_{*L}$.
To this aim, we now propose, for such durative schemas, the same classification introduced for ground durative actions in Definition \ref{def: durative action type}. First, we need a further concept:

\begin{definition}[Reachable action schemas] $D\alpha_{*L}$ is said to be \emph{reachable} if it is executable and 
$$w(Pre^+_{\alpha^{st}_{*L}}\cup (Pre^+_{\alpha^{end}_{*L}}\setminus Eff^+_{\alpha^{st}_{*L}}))\leq 1$$
\end{definition}

\begin{proposition}\label{prop: reach two schema} If $gr$ and $\gamma$ are coherent over $L$ and $Da_*=gr(D\alpha_*)$, we have that $Da_{*\gamma}$ is $\gamma$-reachable if and only if $D\alpha_{*L}$ is reachable. 
\end{proposition}
\begin{proof} Immediate consequence of Propositions \ref{prop: reach two} and \ref{cor:g-i} and of equation (\ref{weight eq1}).
\end{proof}

\begin{definition}[Safe durative action schemas]\label{def:classdur} When $D\alpha_{*L}$ is such that
\begin{enumerate}[(i)]
\item $D\alpha_{*L}$ is reachable;
\item $\alpha^{st}_{*L}$ is strongly safe;
\item $\alpha^{end}_{*L}$ is relevant unbounded;
\item $D\alpha_{*L}$ satisfies any of the conditions below:
\begin{enumerate}[(a)]
\item $\alpha^{st}_{*L}$ irrelevant, $w(Pre^+_{\alpha^{st}_{*L}})=1$, $Pre^+_{\alpha^{st}_{*L}}\subseteq Eff^-_{\alpha^{st}_{*L}}$;
\item $\alpha^{st}_{*L}$ irrelevant, $w(Pre^+_{\alpha^{st}_{*L}})=1$, $Pre^+_{\alpha^{st}_{*L}}\not\subseteq Eff^-_{\alpha^{st}_{*L}}$, $Pre^+_{\alpha^{st}_{*L}}\subseteq Eff^-_{\alpha^{end}_L}\cup Eff^+_{\alpha^{end}_L}$;
\item $\alpha^{st}_{*L}$ irrelevant, $w(Pre^+_{\alpha^{st}_{*L}})=0$, $Pre^-_{\alpha^{st}_{*L}}\cup Eff^-_{\alpha^{st}_{*L}}\cup Eff^-_{\alpha^{end}_L}\cup Eff^+_{\alpha^{end}_L}$ covers $\cT$;
\item $\alpha^{st}_{*L}$ relevant, $Eff^+_{\alpha^{st}_L}\subseteq Eff^-_{\alpha^{end}_L}\cup Eff^+_{\alpha^{end}_L}$.
\end{enumerate}
\end{enumerate}
we say that $D\alpha_{*L}$ is \emph{simply safe of type (x)} where $x\in\{a,b,c,d\}$.
\end{definition}

\begin{corollary} Safety for durative auxiliary actions is a liftable property. $Da_*=gr(D\alpha_*)$ is safe if and only if:
\begin{itemize}
\item $D\alpha_*$ is executable;
\item For every $\cT$-class $L$  of literals in $D\alpha$, one of the following conditions hold:
\begin{itemize} 
\item $D\alpha_{*L}$ is strongly safe;
\item $\alpha^{st}_{*L}$ is strongly safe and $D\alpha_{*L}$ is unreachable;
\item $D\alpha_{*L}$ is simply safe of type (x) where $x\in\{a,b,c,d\}$.
\end{itemize}
\end{itemize}
\end{corollary}
\begin{proof}
Immediate consequence of previous definitions and Proposition \ref{prop: aux simple safe}.
\end{proof}

\begin{example}[\sfs{Floortile} domain]
\label{ex:class}
Consider our usual template: $$\cT_{ft} = (\{\langle \verb+robot-at+, 2, 0 \rangle, \langle \verb+painted+, 2, 1 \rangle,  \langle \verb+clear+, 1, 1 \rangle \}$$ 
and the action schema: 
$$D\alpha={\tt paint-up}: ({\tt {paint-up}}^{st}, {\tt {paint-up}}^{inv}, {\tt {paint-up}}^{end})$$

\noindent where the single instantaneous action schemas have the specifications as in Table \ref{tb:paintup}.  
\begin{table}[htp]
\begin{center}
\begin{tabular}{|c|c|c|c|}
\hline
$\alpha$ & ${\tt {paint-up}}^{st}$ & ${\tt {paint-up}}^{inv}$ & ${\tt {paint-up}}^{end}$\\
\hline
$Pre^+_{\alpha}$ & $\{{\tt{robot-at(r, x)}}$ &  $\{{\tt{robot-has(r, c)}}$ & $\emptyset$\\
& ${\tt{clear(y)}} \}$ & ${\tt{up(y, x)}} \}$ & \\
\hline
$Eff^+_{\alpha}$ & $\emptyset$ & $\emptyset$ & $\{{\tt{painted(y, c)}}\}$\\
\hline
$Eff^-_{\alpha}$ & $\{{\tt{clear(y)}} \}$ & $\emptyset$ & $\emptyset$\\
\hline
\end{tabular}
\end{center}
\caption{Durative action schema {\tt {paint-up}} (abbreviated specification).}
\label{tb:paintup}
\end{table}%

In this action schema, we have three literals that match $\cT_{ft}$:  \url{robot-at(r,x)}, \url{clear(y)} and  \url{painted(y,c)}. They form two $\cT$-classes: $L_1 = \{$\url{robot-at(r,x)}$\}$ and $L_2 = \{$\url{clear(y)},  \url{painted(y,c)}$\}$. Note that in this case {\tt {paint-up}}$_{L_i}^{st}$ is equal to {\tt {paint-up}}$_{*L_{i}}^{st}$ for $i=1,2$ and the same holds for {\tt {paint-up}}$_{L_i}^{end}$.

The pure action schemas {\tt {paint-up}}$_{L_1}^{st}$, {\tt {paint-up}}$_{L_2}^{st}$ and {\tt {paint-up}}$_{L_1}^{end}$ are strongly safe because they are irrelevant. The pure schema {\tt {paint-up}}$_{L_2}^{end}$ is relevant unbounded.

The pure durative action schema {\tt paint-up}$_{L_1}$ is strongly safe because {\tt {paint-up}}$_{L_1}^{st}$ and {\tt {paint-up}}$_{L_1}^{end}$ are strongly safe since they are irrelevant.

The pure schema {\tt paint-up}$_{L_2}$ is simply safe of type (a) since: 
\begin{itemize}
\item {\tt paint-up}$_{L_2}$ is reachable because {\tt {paint-up}}$_{L_2}^{end}$ does not contain preconditions
\item {\tt {paint-up}}$_{L_2}^{st}$ is strongly safe since it is irrelevant;
\item {\tt {paint-up}}$_{L_2}^{end}$ is relevant unbounded;
\item $w(Pre^+_{ {\tt {paint-up}}_{L_2}^{st}}) = 1$ because the preconditions at start consist in \url{clear(y)};
\item $Pre^+_{ {\tt {paint-up}}_{L_2}^{st}} \subseteq Eff^-_{ {\tt {paint-up}}_{L_2}^{st}}$ because the delete effects at start also contain \url{clear(y)}.
\end{itemize}
\end{example}

\subsection{Lifting properties of multiple actions}
In this section, we study how properties that involve more than one action (e.g. mutex) can be lifted. This requires to work simultaneously with different groundings and, for this reason, additional concepts are need.

Consider two action schemas $\alpha^1$ and $\alpha^2$ (instantaneous or durative) with set of free arguments $V_{\alpha^1}$ and $V_{\alpha^2}$, respectively. Whenever we consider two groundings $gr^1$ and $gr^2$ for $\alpha^1$ and $\alpha^2$, respectively, the pairwise properties of the two actions $a^i=gr^i(\alpha^i)$ (e.g. properties regarding the sequence $(a^1, a^2)$ or the set $\{a^1, a^2\}$) are non liftable, as in general they may depend on the specific groundings chosen. A key aspect is the possible presence, in the two action schemas, of pairs of free arguments $v^i\in V^{if}$ such that $gr^1(v^1)=gr^2(v^2)$: this may cause the same ground atom to appear in the two actions $a^1$ and $a^2$, which in principle can affect the validity of certain properties, such as non-interference.
To cope with this complexity at the lifted level, we introduce a concept of \emph{reduced union} of the two sets $V_{\alpha^1}$ and $V_{\alpha^2}$ to be used as a common set of free arguments for the two schemas.

We define a \emph{matching} between $\alpha^1$ and $\alpha^2$ as any subset $\cM\subseteq V_{\alpha^1}\times V_{\alpha^2}$ such that:
\begin{itemize}
\item If $(v^1, v^2),\, (w^1, v^2)\in\cM$, then $v^1=w^1$;
\item If $(v^1, v^2),\, (v^1, w^2)\in\cM$, then $v^2=w^2$.
\end{itemize}
We now define the set $V_{\alpha^1}\sqcup_{\cM}V_{\alpha^2}$ obtained by $V_{\alpha^1}\cup V_{\alpha^2}$ by reducing each pair of arguments $v^1\in V_{\alpha^1}$ and $v^2\in V_{\alpha^2}$ such that $(v^1, v^2)\in\cM$ to a new argument, denoted as $v^1v^2$. Note that in the case when $\cM=\emptyset$, no reduction takes place and $V_{\alpha^1}\sqcup_{\emptyset}V_{\alpha^2}=V_{\alpha^1}\cup V_{\alpha^2}$.

Given a matching $\cM$, we have natural maps $\pi^i_{\cM}: V^{if} \to V_{\alpha^1}\sqcup_{\cM}V_{\alpha^2}$ associating to each argument $v^i$, $v^i$ itself or the new reduced argument $v^iv^j$ in case $(v^i, v^j)\in\cM$.
The two schemas $\alpha^1$ and $\alpha^2$ can thus be rewritten in this new alphabet by formally substituting each free argument $v^i\in V^{if}$ in their literals with $\pi^i_{\cM}(v^i)$. If $l^i$ is a literal of $\alpha^i$, we denote by $\pi^i_{\cM}(l^i)$ the literal obtained with this substitution. Similarly, if $A^i$ is a set of literals of $\alpha^i$, we put $\pi^i_{\cM}(A^i)=\{\pi^i_{\cM}(l^i)\,|\, l^i\in A^i\}$. 

On the literals of the two schemas, expressed in the common argument set $V_{\alpha^1}\sqcup_{\cM}V_{\alpha^2}$, we can jointly apply set theoretic operators. If $l^i$ is a literal of $\alpha^i$ and $A^i$ is a set of literals of $\alpha^i$, for $i=1,2$, we will use the notation 
$l^1=_{\cM}l^2$ for $\pi^1_{\cM}(l^1)=\pi^2_{\cM}(l^2)$ and $l^1\in_{\cM} A^2$ for $\pi^1_{\cM}(l^1)\in \pi^2_{\cM}(A^2)$. Similarly, we put $A^1*_{\cM}A^2=\pi^1_{\cM}(A^1)* \pi^2_{\cM}(A^2)$ where $*\in \{\cup, \cap, \setminus\}$.

We now investigate the relation between matchings and specific groundings of the two schemas.

\begin{definition}[Coherent grounding functions] Consider two action schemas $\alpha^1$ and $\alpha^2$ and a matching $\cM$ between them.
Two grounding functions $gr^1$ and $gr^2$ for $\alpha^1$ and $\alpha^2$, respectively, are said to be $\cM$\emph{-adapted} if given $v^i\in V^{if}$ for $i=1,2$, it holds $gr^1(v^1)=gr^2(v^2)$ if and only if $(v^1, v^2)\in \cM$.
\end{definition}

\begin{remark}\label{rem: ground adapted} Note that, given two groundings $gr^1$ and $gr^2$, if we consider
 $\cM =\{(v^1, v^2)\,|\, gr^1(v^1)=gr^2(v^2)\}$ we clearly have that $\cM$ is a matching (recall that maps $gr^i$ are injective) and $gr^1$ and $gr^2$ are $\cM$-adapted. 
 \end{remark}

Coherent groundings can clearly be factorised through the reduced set $V_{\alpha^1}\sqcup_{\cM}V_{\alpha^2}$:

\begin{proposition}\label{prop: adapted groundings} Consider two action schemas $\alpha^1$ and $\alpha^2$, a matching $\cM$ between them, and grounding functions $gr^i$ for $\alpha^i$, $i=1,2$.
The following conditions are equivalent:
\begin{enumerate}[(i)]
\item $gr^1$ and $gr^2$ are $\cM$-adapted;
\item there exists an injective function $gr: V_{\alpha^1}\sqcup_{\cM}V_{\alpha^2}\to \cO$ such that, $gr^i=gr\circ\pi^i_{\cM}$ for $i=1,2$.
\end{enumerate}
\end{proposition}

Suppose that $gr^1$ and $gr^2$ are two $\cM$-adapted groundings of $\alpha^1$ and $\alpha^2$. If $A^i$ is a set of literals of $\alpha^i$, for $i=1,2$, for any set theoretic operation $*\in\{\cup, \cap, \setminus\}$ it holds that:
\begin{equation}\label{eq:gr}gr^1(A^1)*gr^2(A^2)=gr(\pi^1_{\cM}(A^1))* gr(\pi^2_{\cM}(A^2))=gr(A^1*_{\cM}A^2)\end{equation}
This follows from Proposition \ref{prop: adapted groundings} and the fact that $gr$ is injective.
An iterative use of (\ref{eq:gr}) shows that any set theoretic expression on the two ground actions $gr^i(\alpha^i)$ is in bijection (through $gr$) with a corresponding expression on the two schemas $\alpha^i$ expressed in the common reduced set $V_{\alpha^1}\sqcup_{\cM}V_{\alpha^2}$. As a consequence, any property of ground actions (with the standing assumption of $\cM$-adapted groundings) that can be expressed by set theoretic operations on their literals can be reformulated by rewriting these literals in the new alphabet $V_{\alpha^1}\sqcup_{\cM}V_{\alpha^2}$. This is the key observation in order to lift properties of pairs of actions.
To be more concrete, we consider the example of non-interfering actions, which will be needed in what follows.

\begin{definition}[Mutex simple action schemas]
\label{schemas non int} 
We say that two action schemas $\alpha^1$ and $\alpha^2$ are $\cM$ \emph{non-interfering} if for $i\neq j$
$$Eff^{+}(\alpha^i)\cap_{\cM} Eff^{-}(\alpha^j)=\emptyset$$
$$[Pre^{+}(\alpha^i)\cup_{\cM} Pre^{-}(\alpha^i)]\cap_{\cM} [Eff^{+}(\alpha^j)\cup_{\cM} Eff^{-}(\alpha^j)]=\emptyset$$
If $\alpha^1$ and $\alpha^2$ are not $\cM$ non-interfering, they are called \emph{$\cM$-mutex}.
\end{definition}

\begin{proposition}\label{prop: mutex1} 
Suppose that $\alpha^1$ and $\alpha^2$ are $\cM$-mutex and suppose that $gr^1$ and $gr^2$ are two $\cM$-adapted grounding functions for $\alpha^1$ and $\alpha^2$, respectively. Then, the two ground actions $a^i=gr^i(\alpha^i)$ are mutex.
\end{proposition}
\begin{proof} Immediate consequence of (\ref{eq:gr}).
\end{proof}

\begin{remark}\label{rem monotonic} Note that certain properties that depend on the matching $\cM$ have a monotonic behaviour, i.e. if they are true for a matching $\cM$, they remain true for a larger matching $\cM'\supseteq \cM$. This is the case, for instance, of properties that can be expressed in terms of identities between literals of type $l^1=_{\cM} l^2$, such as the $\cM$-mutex property.
\end{remark}

To cope with properties related to a template and its instantiations, it is useful to introduce a family of matchings induced by the presence of literals in the two schemas matching in a template.
Precisely, consider now a template $\cT = ( \cC, \cF_{\cC})$ and two action schemas $\alpha^1$ and $\alpha^2$.
Consider $\cT$-classes $L^i$ of literals of $\alpha^i$ for $i=1,2$. There is a natural way to associate a matching to $L^1$ and $L^2$ as follows. Pick literals $l^i\in L^i$ for $i=1,2$ and consider components $c^i\in\cC$ such that $l^i$ matches $\cT$ through $c^i$. 
Put 
\begin{equation}\label{matching classes}\cM_{L^1,L^2}:=\{(\args[h,l^1], \args[k, l^2])\,|\, (c^1, h) \sim_ {\cF_{\cC}}(c^2, k)\}\end{equation}
It immediately follows from the definition of $\cT$-coupled pairs of literals (Definition \ref{def:match}) that $\cM_{L^1,L^2}$ does not depend on the particular literals $l^i$ chosen, but only on the $\cT$-classes $L^i$. 

Essentially, in $\cM_{L^1,L^2}$, we are rewriting arguments in the literals of $L^1$ and $L^2$ that correspond to $\cF_{\cC}$-equivalent variables in the template $\cT$. The next proposition shows the role played by such a matching.

\begin{proposition}\label{prop: equiv adapted} Consider two groundings $gr^1$ and $gr^2$ for $\alpha^1$ and $\alpha^2$, respectively, which are $\cM$-adapted. Then the following facts hold:
\begin{enumerate}[(i)] 
\item given an instance $\gamma$ for $\cT$, if  $L^i$ are the $\cT$-classes of literals of $\alpha^i$ on which $gr^i$ and $\gamma$ are coherent. Then, $\cM_{L^1, L^2}\subseteq \cM$;
\item given $\cT$-classes  of literals $L^i$ of $\alpha^i$, if $\cM_{L^1, L^2}\subseteq \cM$, there exists just one instance $\gamma$ of $\cT$ such that  $gr^i$ and $\gamma$ are coherent on $L^i$.
\end{enumerate}
\end{proposition}
\begin{proof}
Immediate consequence of the definition (\ref{matching classes}) and of Remark \ref{rem:coherent}.
\end{proof}

We are now ready to lift the properties used in Section \ref{sec: cond inv}. We start with unreachability.

\begin{definition}[Unreachable durative action schemas]
\label{schemas non int} 
Given two durative action schemas $D\alpha^1$, $D\alpha^2$ and corresponding $\cT$-classes of literals $L^1$ and $L^2$,
we say that $(\{\alpha^{1inv}, \alpha^{2inv}\}, \{\alpha^{1end}, \alpha^{2end}\})$ is $(L^1, L^2)$-\emph{unreachable} if, denoted $\cM=\cM_{L^1, L^2}$, at least one of the following conditions is satisfied
\begin{enumerate}[(i)] 
\item $Pre^+_{\alpha^{1inv}}\cap_{\cM} Pre^-_{\alpha^{2end}}\neq\emptyset$;
\item $Pre^-_{\alpha^{1inv}}\cap_{\cM} Pre^+_{\alpha^{2end}}\neq\emptyset$;
\item $w((Pre^+_{\alpha^{1inv}_{L^1}}\cup Pre^+_{\alpha^{1end}_{L^1}})\cup_{\cM} (Pre^+_{\alpha^{2inv}_{L^2}}\cup Pre^+_{\alpha^{2end}_{L^2}}))\geq 2$.
\end{enumerate} 
\end{definition}

\begin{proposition}\label{prop: unreachable schemas} 
Suppose that $D\alpha^1$, $D\alpha^2$ are two durative action schemas and $gr^1$, $gr^2$ two corresponding grounding functions. Put $Da^i=gr(D\alpha^i)$ and consider an instance $\gamma$. Let $L^i$ be the $\cT$-class of literals of $D\alpha^i$ on which $gr^i$ and $\gamma$ are coherent. If 
$(\{\alpha^{1inv}, \alpha^{2inv}\}, \{\alpha^{1end}, \alpha^{2end}\})$ is $(L^1, L^2)$-unreachable, then  $(\{a^{1inv}, a^{2inv}\}, \{a^{1end}, a^{2end}\})$
is $\gamma$-unreachable.
\end{proposition}

We now propose the lifted version of relevant right isolated.

\begin{definition}[Relevant right isolated schemas]\label{def: relevant right isolated schema}
Given a template $\cT$, the set of durative action schemas $\cA^d$ is said to be \emph{relevant right isolated} if, for every $D\alpha^1, D\alpha^2\in\cA^d$, corresponding $\cT$-classes $L^1, L^2$ of literals of each of them such that $D\alpha^i_{L^i}$ are both not strongly safe, one of the following conditions is satisfied (we use the notation $\cM=\cM_{L^1, L^2}$):
\begin{enumerate}[(i)]
\item $|Eff^+_{\alpha^{1end}_{L^1}}\cup_{\cM} Eff^+_{\alpha^{2end}_{L^2}}|\leq 1$;
\item at least one of the two pairs $\{\alpha^{1end}, \alpha^{2end}\}$ or $\{\alpha^{1inv}, \alpha^{2inv}\}$ is $\cM$-mutex;
\item $(\{\alpha^{1inv}, \alpha^{2inv}\}, \{\alpha^{1end}, \alpha^{2end}\})$ is $(L^1, L^2)$-unreachable.
\end{enumerate}
\end{definition}

\begin{proposition}\label{prop: relevant right isolated schemas} 
Given a template $\cT$, suppose that the set of durative action schemas $\cA^d$ is {relevant right isolated}. Then $\cG\cA^d$ is also relevant right isolated.
\end{proposition}

Similarly, we can lift the property of relevant left isolated expressed in Definition \ref{def: relevant left isolated} by analogously defining:

\begin{definition}[Relevant left isolated schemas]\label{def: relevant left isolated schema}
Given a template $\cT$, the set of durative action schemas $\cA^d$ is said to be \emph{relevant left isolated} if, for every $D\alpha^1, D\alpha^2\in\cA^d$, corresponding $\cT$-classes $L^1, L^2$ of literals of each of them such that $D\alpha^i_{L^i}$ are both not strongly safe, one of the following conditions is satisfied (we use the notation $\cM=\cM_{L^1, L^2}$):
\begin{enumerate}[(i)]
\item at least one of the two pairs $\{\alpha^{1st}, \alpha^{2st}\}$ or $\{\alpha^{1inv}, \alpha^{2inv}\}$ is $\cM$-mutex;
\item $(\{\alpha^{1st}, \alpha^{2st}\}, \{\alpha^{1inv}, \alpha^{2inv}\})$ is $(L^1, L^2)$-unreachable.
\end{enumerate}
\end{definition}

Similarly, it holds:

\begin{proposition}\label{prop: relevant left isolated schemas} 
Given a template $\cT$, suppose that the set of durative action schemas $\cA^d$ is {relevant left isolated}. Then $\cG\cA^d$ is also relevant left isolated.
\end{proposition}

The last property we want to lift is that of strong irrelevant unreachability expressed in Definition \ref{def: strongly del inter}. This is important since it can be efficiently implemented at the algorithmic level.

\begin{definition}[Strongly irrelevant unreachable schemas]\label{def: strongly irr unreach schemas} Consider a template $\cT$, a pair of action schemas $\alpha^1,\alpha^2\in\cA$, and corresponding $\cT$-classes of literals $L^1$ and $L^2$.
We say that the pair $(\alpha^1,\alpha^2)$ is strongly $(L^1,L^2)$ \emph{-irrelevant unreachable} if, denoted $\cM=\cM_{L^1, L^2}$, any of the following conditions is satisfied:
\begin{enumerate}[(i)]
\item there exist $l^1\in \Gamma_{\alpha^1}^+$, $l^2\in Pre_{\alpha^2}^-$ with $l^1=_{\cM} l^2$ such that, for every action schema $\alpha$, for every $\cT$-class $L$ of literals of $\alpha$ for which $\alpha_{L}$ is irrelevant,  and for every matching $\tilde\cM$ between $\alpha^1$ and $\alpha$ for which $\tilde\cM\supseteq \cM_{L^1, L}$, we have that $l^1\not\in_{\tilde\cM} Eff^-_{\alpha}$;
\item there exist $l^1\in \Gamma_{\alpha^1}^-$, $l^2\in Pre_{\alpha^2}^+$ with $l^1=_{\cM} l^2$ such that, for every action schema $\alpha$, for every $\cT$-class $L$ of literals of $\alpha$ for which $\alpha_{L}$ is irrelevant,  and for every matching $\tilde\cM$ between $\alpha^1$ and $\alpha$ for which $\tilde\cM\supseteq \cM_{L^1, L}$, we have that $l^1\not\in_{\tilde\cM} Eff^+_{\alpha}$;
\item $w(Pre^+_{\alpha^1_{L^{1}}}\cup_{\cM} (Pre^+_{\alpha^2_{L^{2}}}\setminus_{\cM} Eff^+_{\alpha^1_{L^{1}}}))>1$.
\end{enumerate}
\end{definition}

\begin{proposition}\label{prop: relevant right isolated schemas} 
Consider a template $\cT$,  a pair of action schemas $\alpha^1,\alpha^2\in\cA$ and relative groundings $gr^1$ and $gr^2$. Put $a^i=gr(\alpha^i)$ and consider an instance $\gamma$. Let $L^i$ be the $\cT$-class of literals of $\alpha^i$ on which $gr^i$ and $\gamma$ are coherent.
If $(\alpha^1,\alpha^2)$ is strongly $(L^1,L^2)${-irrelevant unreachable}, then $(a^1,a^2)$ is strongly $\gamma${-irrelevant unreachable}.
\end{proposition}

Denote by $\cA^d(\cT)$ the durative action schemas that are not strongly safe with respect to the template $\cT$ and with $\cA^{st}(\cT)$, $\cA^{end}(\cT)$ the corresponding start and ending fragments.

\begin{proposition}\label{prop: relevant non inter lifted} Consider a template $\cT$.  The set $\cG\cA^d$ is relevant non intertwining if the following conditions are satisfied:
\begin{enumerate}[(i)]
\item  $\cA^d$ is relevant left isolated;
\item for every $D\alpha^1\in \cA^{d}$ and $\cT$-class $L^1$, and for every $\alpha^2\in \cA$ and $\cT$-class $L^2$ such that $\alpha^{1st}_{L^1}\in \cA^{st}(\cT)$, $\alpha^2_{L^2} \not\in \cA^{end}(\cT)$ and is not irrelevant or $\alpha^2_{L^2}\in \cA^{st}(\cT)$, we have that $(\alpha^{1st}, \alpha^2)$ is strongly $(L^1, L^2)$-irrelevant unreachable.
\end{enumerate}
\end{proposition}

%

We are now ready to propose the lifted versions of our main Theorems \ref{theo: *safety}, \ref{theo: non intertwining}. Proofs are straightforward consequences of our previous definitions and results.

\begin{corollary}
\label{cor: *safety schema} 
Consider a template $\cT$ and suppose that the set of instantaneous action schemas $\cA^i$ and that of durative action schemas $\cA^d$ satisfy the following properties:
\begin{enumerate}[(i)]
\item every $\alpha\in\cA^i$ is strongly safe;
\item for every $D\alpha\in\cA^d$ and every $\cT$-class $L$ such that $D\alpha_L\in \cA^d(\cT)$, $D\alpha_{*L}$ is reachable and strongly safe;
\item $\cA^d$ is relevant right isolated.
\end{enumerate}
Then, $\cT$ is invariant.
\end{corollary}

\begin{corollary} 
\label{cor: non intertwining schemas} 
Consider a template $\cT$ and suppose that the set of instantaneous action schemas $\cA^i$ and that of durative action schemas $\cA^d$ satisfy the following properties:
\begin{enumerate}[(i)]
\item every $\alpha\in\cA^i$ is strongly safe;
\item for every $D\alpha\in\cA^d$, $D\alpha_*$ is safe;
\item  $\cA^d$ satisfies the conditions expressed in Proposition \ref{prop: relevant non inter lifted}.
\end{enumerate}
Then, $\cT$ is invariant.
\end{corollary}

Finally, it is useful to consider a lifted version of Corollary \ref{cor: type a}.

\begin{corollary} \label{cor: type a schemas} 
Consider a template $\cT$ and suppose that the set of instantaneous action schemas $\cA^i$ and that of durative action schemas $\cA^d$ satisfy the following properties:
\begin{enumerate}[(i)]
\item for every $D\alpha\in\cA^d$ and every $\cT$-class $L$, if $D\alpha_L\in \cA^d(\cT)$, then $D\alpha_{*L}$ is simply safe of type (a);
\item for every $\alpha\in\cA$ and every $\cT$-class $L$, if $\alpha_L\not\in \cA^{st}(\cT)\cup \cA^{end}(\cT)$, then, $\alpha_L$ is either irrelevant or relevant balanced.
\end{enumerate}
Then, $\cT$ is invariant.
\end{corollary}

We end this section by presenting two examples from the IPCs in which we apply Corollaries \ref{cor: type a schemas} and \ref{cor: *safety schema} to demonstrate the invariance of the templates under consideration. Corollary \ref{cor: non intertwining schemas} is the most general one and can be used in more complex cases.

\begin{example}[\sfs{Floortile} domain]
\label{ex:coro}
Consider our usual template: $$\cT_{ft} = (\{\langle \verb+robot-at+, 2, 0 \rangle, \langle \verb+painted+, 2, 1 \rangle,  \langle \verb+clear+, 1, 1 \rangle \}$$ 
The actions schemas in the domains are: $$\cA^d = \{{\tt {change-color, paint-up, paint-down, up, down, right, left}} \}$$
The schemas {\tt paint-up} and {\tt paint-down} are symmetrical and differ only on literals not in the components of $\cT_{ft}$. They have the same $\cT$-classes $L_1 = \{$\url{robot-at(r,x)}$\}$ and $L_2 = \{$\url{clear(y)},  \url{painted(y,c)}$\}$. As seen in Example \ref{ex:class}, the pure schemas {\tt paint-up}$_{*L_1}^{st}$ and {\tt paint-up}$_{*L_1}^{end}$ are irrelevant and {\tt paint-up}$_{*L_2}$ is simply safe of type (a). The same holds for {\tt paint-down}$_{*L_1}$ and {\tt paint-down}$_{*L_2}$. 

The schemas {\tt up, down, right, left} are also symmetrical and differ only on literals not in the components of $\cT_{ft}$. They have the same $\cT$-classes $L_3 = \{$\url{robot-at(r,x)}, \url{clear(x)}$\}$ and $L_4 = \{$\url{robot-at(r,y)}, \url{clear(y)}$\}$. The schemas {\tt up}$_{*L_i}$, {\tt down}$_{*L_i}$, {\tt right}$_{*L_i}$ and {\tt left}$_{*L_i}$, with $i=3,4$, are all simply safe of type (a).

The schema {\tt change-color} has no equivalence classes and its start and end fragments are both irrelevant. 

By Corollary \ref{cor: type a schemas}, the template template $\cT_{ft}$ is invariant.
\end{example}

\begin{example}[\sfs{Depot} domain]
\label{ex:depot}
Consider the domain \sfs{Depot} (see Appendix B) and the template: $$\cT_{dp} = (\{\langle \verb+lifting+, 2, 1 \rangle, \langle \verb+available+, 1, 1 \rangle \}$$
Invariants of this template mean that, given a hoist, it can be in two possibile states: lifting a crate or available.
The actions schemas in the domains are all durative: $$\cA^d = \{{\tt {drive, lift, drop, load, unload}} \}$$
We indicate them as $D\alpha^1, \ldots, D\alpha^5$ respectively and, given $D\alpha^i$, its arguments as $x_i, y_i, \ldots$.

To demonstrate that $\cT_{dp}$ is invariant, we want to apply Corollary \ref{cor: *safety schema}. We start with condition (ii) since $\cA^i$ is empty. 

The action $D\alpha^1 = ${\tt drive} has no literals that match the template so it is strongly safe. The other schemas have respectively $\cT$-classes $L_i = \{$\url{lifting}$(x_i,y_i)$, \url{available}$(x_i)\}$. There are only two fragments of the durative actions that are not strongly safe as they are relevant unbounded: $\alpha_{L_3}^{3end}$ and $\alpha_{L_4}^{4end}$. However, their auxiliary versions $\alpha_{*L_3}^{3end}$ and $\alpha_{*L_4}^{4end}$ are strongly safe since they are balanced (when the over all condition \url{lifting}$(x_3,y_3)$ is added to the end effects, it matches the delete effect \url{lifting}$(x_3,y_3)$ and balances the add effect \url{available}$(x_3)$; similar considerations hold for $D\alpha^4$). Reachability for $\alpha_{*L_3}^{3end}$ and $\alpha_{*L_4}^{4end}$ is a straightforward check.  In consequence, condition (ii) holds.

We now need to verify condition (iii) of Corollary \ref{cor: *safety schema}, i.e. $\cA^d$ is relevant right isolated. Under the re-writing $\cM_{L_3L_4}$, we have that $x_3=x_4$ and $y_3=y_4$ and therefore $Eff^+_{\alpha^{3end}_{L_3}}\cup_{\cM_{L_3L_4}} Eff^+_{\alpha^{4end}_{L_4}} = \{ \verb+available+(x_3) = \verb+available+(x_4)\}$. Hence condition (i) of Definition \ref{def: relevant right isolated schema} is satisfied.

We can conclude that $\cT_{dp}$ is an invariant template.
\end{example}


\section{Guess, Check and Repair Algorithm}
\label{sec:algorithm}
As with related techniques \citep{GereviniSchubert-00,Rintanen-00,Helmert-09}, our algorithm for finding invariants implements a \emph{guess, check and repair} approach. It starts by generating a set of initial simple templates. For each template $\cT$, it then applies the results stated in the previous sections to check its invariance. If $\cT$ is invariant, the algorithm outputs it. On the other hand, if the algorithm does not manage to prove the invariance of $\cT$, it discards it. Before rejection, however, the algorithm tries to fix the template by generating a set of new templates that are guaranteed not to fail for the same reasons as $\cT$. In turn, these new templates need to be checked against the invariance conditions as they might fail due to other reasons.  

\subsection{Guessing initial templates}
\label{sec:initialtemp}
When we create the set of initial templates, we ignore constant relations, i.e. relations whose ground atoms have the same truth value in all the states (for example, type predicates). In fact, they are trivially invariants and so are typically not interesting. 

For each modifiable relation $r$ with arity $a$, we generate $a+1$ initial templates. They all have one component and zero or one counted argument (which can be in any position from $0$ to $(a-1)$):  $\langle r, a, a \rangle$ (no counted argument) and  $\langle r, a, p \rangle$ with $p \in \{0, \ldots, (a-1) \}$. Since the templates have one component, there is only one possible admissible partition $\cF_{\cC}$, with $\cC = \{ c \}$. Hence, we construct the template $\cT = ( \cC, \cF_{\cC})$. 

\begin{example}[\sfs{Floortile} domain]
\label{ex:guess}
Consider the components $c_1 =  \langle \verb+robot-at+, 2, 1 \rangle$. Put $\cF_{\cC} = \{F_1\}$ where $F_1 = \{(c_1, 0)\}$. An initial template is $\cT_1 = (\{c_1\},  \{F_1\})$. Intuitively, invariants of $\cT_1$ mean that a robot can occupy only one position at any given time and our algorithm validates it as an invariant. Another initial template is built by considering the component $c_2 =  \langle \verb+robot-at+, 2, 0 \rangle$ and the partition $\cF_{\cC} = \{F_2\}$ where $F_2 = \{(c_2, 1)\}$. We have another initial template: $\cT_2 = (\{c_2\},  \{F_2\})$. Invariants of this template mean that a tile cannot be occupied by more than one robot, which is not true in general, and our algorithm correctly discards it. 
Finally, consider the component $c_3 =  \langle \verb+robot-at+, 2, 2 \rangle$ and the partition $\cF_{\cC} = \{F_3\}$ where $F_3 = \{(c_3, 0), (c_3, 1)\}$. Another initial template is $\cT_3 = (\{c_3\},  \{F_3\})$. This is also not an invariant and is rejected. 

If we repeat this process with every modifiable relation $r$ in the $\sfs{Floortile}$ domain, we obtain the full set of initial templates. 
\end{example}

\subsection{Checking conditions for invariance}
Given a template, we apply the results stated in the previous sections to check its invariance. In particular, we apply our most operative results: Corollary \ref{cor: all strongly safe} and Corollaries \ref{cor: *safety schema} - \ref{cor: type a schemas}. Remarkably, all these results work at the level of action schemas, not ground actions. 

We first need to verify if all the instantaneous action schemas $\cA$ in the domain, both the native ones and those obtained from the fragmentation of durative actions, respect the strong safety conditions. We then check safety conditions that only involve durative action schemas that are not strongly safe. Finally, we validate additional conditions that avoid the intertwinement of potentially dangerous durative actions. Given the different computational complexity of our results (see considerations below), our algorithm checks the applicability of them in the following order: first, Corollary \ref{cor: all strongly safe}, which involves only conditions for instantaneous schemas, then Corollary \ref{cor: type a schemas}, which considers safety conditions for individual action schemas, and finally Corollaries \ref{cor: *safety schema} and \ref{cor: non intertwining schemas}, which verify conditions involving pairs of durative action schemas. To implement this procedure, we apply the decision tree shown in Figure \ref{fig:dectree} to the set of action schemas $\cA$. The leaves labelled as \emph{Possibly Not Invariant} arise when our sufficient results do not apply. In this case, we cannot assert anything about the invariance of the template.

\begin{figure}[!htbp]
\centering\includegraphics[height=165mm, width=185mm, keepaspectratio]{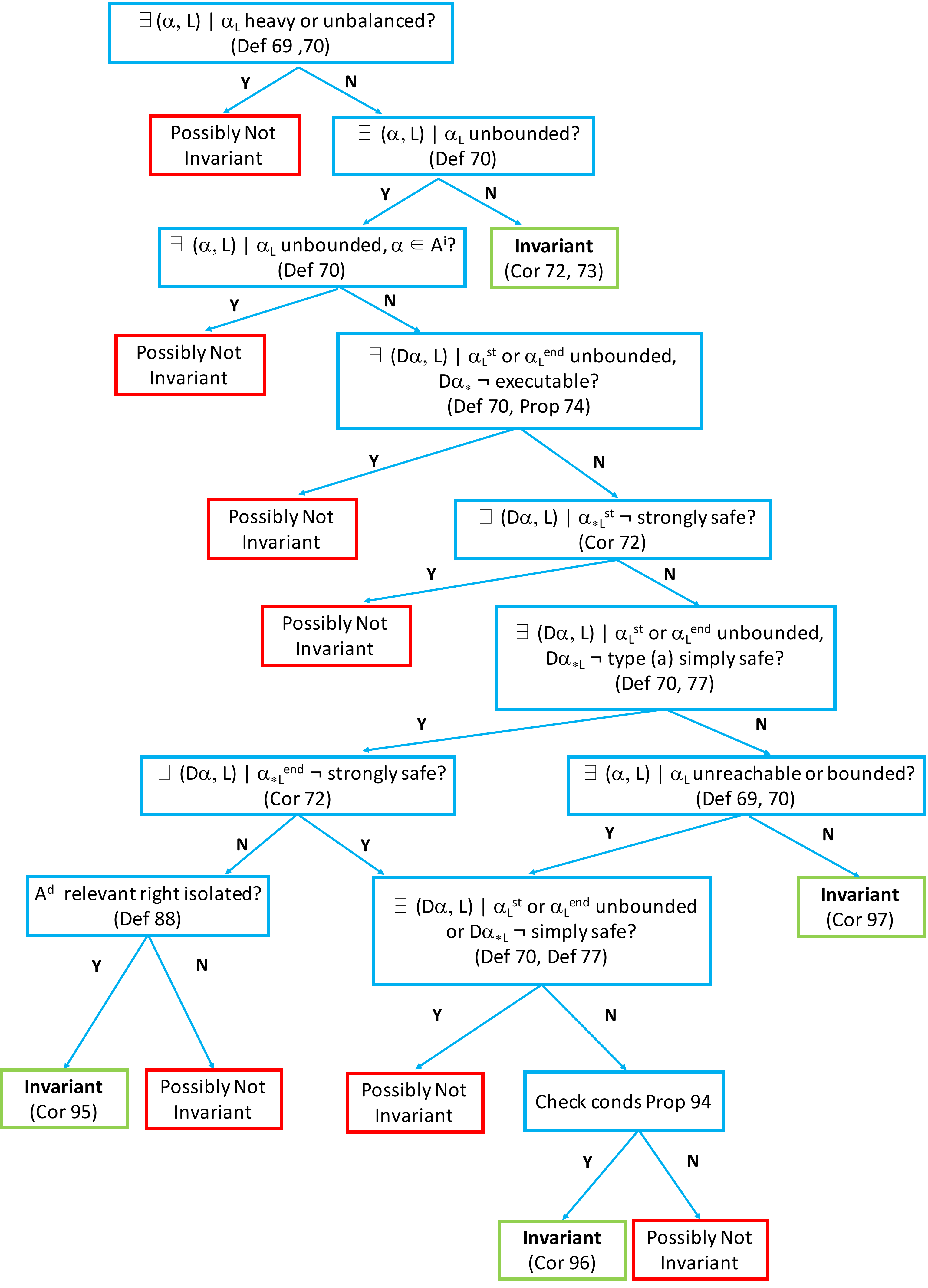}
\caption{Decision Tree for deciding the invariance of a template $\cT$.}
\label{fig:dectree}
\end{figure}

Our checks involve the analysis of all $\cT$-classes in each action schema $\alpha$ in the domain. Since the $\cT$-classes form a partition of the set of literals in the schema that match the template, the maximum number of $\cT$-classes is equal to the number of such literals. We can estimate this term with the product $\omega\cdot |\cC|$ where $\omega$ is the maximum number of literals in any schema that share the same relation and $|\cC|$ is the cardinality of the template's component set $\cC$.
We deduce that all safety checks for individual schemas (both the instantaneous and the durative ones) have a computational complexity of the order of $M \cdot |\cA| \cdot\omega\cdot |\cC|$, where $M$ is the maximum number of the literals appearing in any schema and $|\cA|$ is the total number of schemas. The check of right and left relevant isolated properties involve instead a pair of schemas and $\cT$-classes and, in consequence, the computational complexity is of the order of $M^2 \cdot |\cA|^2 \cdot \omega^2\cdot |\cC|^2$. The check that two schemas are strongly irrelevant unreachable, as required in Corollary \ref{cor: non intertwining schemas}, involves considering a third action schema and a family of matchings, $\tilde\cM$. The complexity relating to the check of the matchings can be shown to be reducible to a check of complexity of the order of the maximum number $N$ of arguments in any literal in the domain. Therefore, the total complexity of the check for strongly irrelevant unreachable schemas is of the order $N\cdot M^3 \cdot |\cA|^3 \cdot\omega^3\cdot |\cC|^3$. In our experiments, we have found no cases in which this check needs to be applied.

\begin{example}[\sfs{Floortile} domain]
\label{ex:guess}
The parameters for the computational complexity analysis are as follows:
$$|\cA|=14,\; M=4,\; \omega =1, |\cC|=3, N=2$$
\end{example}

\subsection{Repairing templates}
\label{sec:newtemp}
When, in analysing an action schema $\alpha$, we reach a failure node in our decision tree, we discard the template $\cT$ under consideration since we cannot prove its invariance. This might be because of two reasons: either $\cT$ is not an invariant or our sufficient conditions are not powerful enough to capture it. Given that we cannot assume that the template is not invariant with certainty, before discarding it, we try to fix it in such a way to obtain new templates for which it might be possible to prove invariance under our conditions. In particular, based on the schema $\alpha$, we enlarge the set of components of the template by adding suitable literals that appear in the preconditions and negative effects of $\alpha$ since they can be useful to prove that $\alpha$ is simple or strong safety.

More precisely, if the algorithm rejects $\cT$ because it finds an instantaneous schema that is heavy or unbalanced (first step in the decision tree), no fixes are possible for $\cT$. Since $\alpha$ leads to a weight greater or equal to two for at least an instance of $\cT$, enlarging the set of components of $\cT$ cannot help in repairing the template. Similarly, if there are durative schemas that are non executable or unreachable, no fixes are possible since these properties cannot be changed by adding components. However, when a failure node is reached in the presence of unbounded schemas, enlarging the set of components might prove useful in making them simply or strongly safe schemas. We operate as follows: for each unbounded schema $\alpha$, we try to turn it into a balanced action schema and, when $\alpha$ is the end fragment of a durative action $D\alpha$, we alternatively attempt to make $D\alpha$ a simply safe schema, as defined in Definition \ref{def:classdur}.

Given a template $\cT = ( \cC, \cF_{\cC})$ that has been rejected by the algorithm, put $k$ the number of fixed arguments for $\cT$ and $m$ the number of its components. Consider an unbounded schema $\alpha$ with relevant literal $l$. We look for another literal $l'$ in $\alpha$ with the following characteristics: 
\begin{enumerate}[(i)]
\item $\rela[l'] = \langle r', a' \rangle$, where $a' = k$ or $a' = (k+1)$; 
\item There exists a bijection $\beta$ from the set of free arguments of $l$ to the set of free arguments of $l'$ such that $Arg[i, l]  = Arg[\beta(i), l']$ for every $i\in I$;
\item $l' \in Pre^{+}_{\alpha} \cap Eff^{-}_{\alpha}$.
\end{enumerate}

If $\alpha$ is the end fragment of a durative action $D\alpha$, then condition (iii) can be substituted with one the alternative following conditions:
\begin{enumerate}[(iv)]
\item $l' \in Pre^{+}_{\alpha^{st}_*} \cap Eff^{-}_{\alpha^{st}_*}$
\item $l' \in Pre^{+}_{\alpha^{st}_*} \cap Eff^{-}_{\alpha^{end}_*}$
\end{enumerate}

For each literal $l'$ that satisfies conditions (i), (ii) and one between conditions (iii), (iv) and (v), we create a new component 
$c' = \langle r', a', p' \rangle$, where $p' \in \{0, \ldots, a' \}$, and one new template $\cT' = ( \cC', \cF'_{\cC})$, where $\cC' = \cC \cup \{ c' \}$ and $\cF'_{\cC}$ is an admissible partition of $F_{\cC'}$ such that for each $c_1, c_2 \in \cC$, we have that $(c_1, i) \sim_{F_{\cC'}} (c_2, j)$ if and only if $(c_1, i) \sim_{F_{\cC}}  (c_2, j)$ and $(c, i) \sim_{F_{\cC'}} (c', j)$ if and only if $Arg[i, l]  =  Arg[j, l']$ (or, equivalently, $j=\beta(i)$). 

If we find a literal $l'$ that satisfies condition (iii), the schema $\alpha$ is guaranteed to be balanced for $\cT'$; if the literal $l'$ satisfies condition (iv), $\alpha$ is guaranteed to be simply safe of type (a) for $\cT'$; finally, if the literal $l'$ satisfies condition (v), $\alpha$ is guaranteed to be simply safe of type (b) for $\cT'$.

\begin{example}[\sfs{Floortile} domain]
\label{ex:repair}
Consider the template  $\cT_2 = (\{c_2\},  \{F_2\})$ as indicated in Example \ref{ex:guess} and the action schema $\alpha=${\tt up}$^{end}$: $Pre_{\alpha}= \emptyset$, $Eff_{\alpha}^+=\{$\url{robot-at(r,y)}, \url{clear(x)}$\}$. The literal \url{robot-at(r,y)} matches the  $\cT_{2}$ and forms a $\cT$-class $L_1 = \{$\url{robot-at(r,y)}$\}$. The pure action schema $\alpha_{L_{1}}$ is relevant unbounded as well as the end parts of the other schemas that indicate movements. If we apply our decision tree to $\cT_2$ and the set of actions $\cA$, we cannot prove that $\cT_2$ is an invariant since the unbounded schemas are not simply safe. Before discarding $\cT_2$, we try to fix it. In particular, the literal \url{clear(y)} satisfies conditions (i), (ii) and (iv) above. If we add it to $\cT_{2}$, we obtain a new template $\cT_2' = (\{c_2, c_2'\}, \{F'_2\})$ where $c_2' =  \langle \verb+clear+, 1, 1 \rangle $ with $F_{c'_{2}} = \{(c_2', 0)\}$ and $F'_2 = \{(c_2, 1), (c_2', 0) \}$. If we apply our decision tree to this new template, we can prove that $\cT'_{2}$ is an invariant since Corollary \ref{cor: type a schemas} can be successfully applied (all schemas are either strongly sage or simply safe of type (a)). Intuitively, invariants of this template mean that, given a tile, either it is clear or it is occupied by a robot.
\end{example}


\section{Experimental Results}
\label{sec:experiments}
To evaluate the performance of our Temporal Invariant Synthesis, referred as \textsc{TIS} in what follows, we have performed a number of experiments on the IPC benchmarks. We implemented the \textsc{TIS} algorithm, reported in Section \ref{sec:algorithm}, in the Python language and conducted the experiments by using a 2.53 GHz Intel Core 2 Duo processor with a memory of 4 GB.

Currently, it is difficult to compare our technique for generating lifted temporal invariants to related techniques since they either handle non temporal domains only (STRIPS domains, in particular) \citep{FoxLong-98,GereviniSchubert-00,Rintanen-00,Rintanen-08,Helmert-09} or find ground temporal invariants \citep{Rintanen-14}. The approach that appears most similar to ours is the invariant synthesis implemented within the Temporal Fast Downward (TFD) planner \citep{Eyerich-09}. However, there is no formal account of such a technique and its soundness, and our knowledge of it is based on a manual inspection of the code\footnote{TFD-0.4 code available at http://gki.informatik.uni-freiburg.de/tools/tfd/index.html}. 

The TFD invariant synthesis is a simple extension of Helmert's original synthesis devised to deal with temporal and numeric domains. The algorithm analyses the temporal schemas directly, without slitting them into their start, overall and end fragments. As the original technique, only the weight of one is considered safe and only this weight is checked for assessing whether an action schema is safe or not. This implies that bounded action schemas (which are strongly safe) and simply safe schemas of type (c) are always considered unsafe. Only two types of relevant durative schemas are evaluated as safe:
\begin{inparaenum}[(i)]
\item those that add and delete relevant literals at start; and
\item those that check that one relevant literal is true at start, then delete this literal at start and finally add another relevant literal at end. 
\end{inparaenum}
Case (i) corresponds to a balanced schema in our classification. However, the TFD synthesis misses schemas that are balanced at end. Case (ii) corresponds to simply safe schemas of type (a). In all the other cases, the action schemas are labelled as unsafe and the candidate invariant is dismissed. 

The TFD invariant synthesis seems to have been carefully tailored to meet the particular requirements of the domains of the most recent competitions, the IPC6 in particular, in which almost all the action schemas fall in Cases (i) and (ii) above. In consequence, the TFD synthesis succeeds in finding useful invariants for these domains, but it fails in addressing the more general problem of finding invariants in generic PDDL temporal domains. While for the IPC6 domains the \textsc{TIS} and the TFD synthesis produce similar results, this is not true in general. Their output is different when the domains used offer a broader variety of action schema's types, as shown in Figure \ref{fig:TDFvsTIS}. This figure highlights the limited applicability of the TFD synthesis in domains not included in the IPC6. In these domains, almost all the ground atoms end up being translated as state variables with two values (true and false) and the performance of TFD suffers from this trivial encoding\footnote{We do not include experiments concerning the performance of TFD as this is outside the scope of our paper.}. 

 \begin{figure}[!htbp]
\centering\includegraphics[width=120mm]{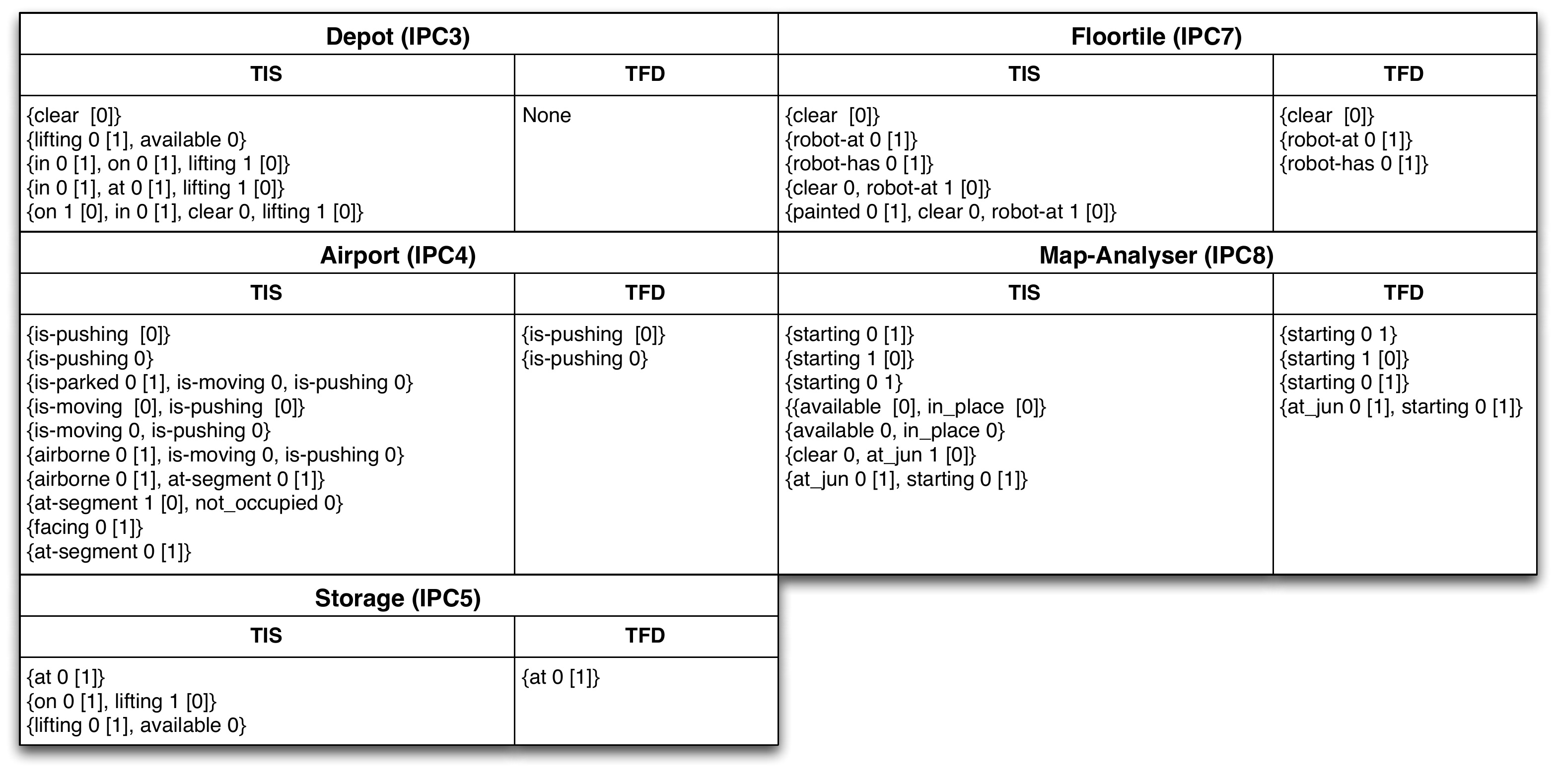}
\caption{\footnotesize Examples of the different output of the \textsc{TIS} and the TFD invariant syntheses for IPC domains.}\label{fig:TDFvsTIS}
\end{figure}

Given the limited scope of the TFD synthesis, in our experimental results, we propose a comparison between our \textsc{TIS} and another technique, which we call Simple Invariant Synthesis (\textsc{SIS}). We designed and used the \textsc{SIS} in order to analyse the impact on temporal planners' performance of different encodings, which result from synthesising state variables based on different sets of lifted invariants. The \textsc{SIS} is also a simple extension of Helmert's original technique to temporal domains, but it is general in its applicability and not devised around specific IPC domains. It adopts the simplest strategy to extend the synthesis of invariants from instantaneous to durative actions: it considers safe only changes in weight that happen at the same time point. More in depth, there are two main differences between the \textsc{TIS} and the \textsc{SIS}: 
\begin{inparaenum}[i)]
\item in the \textsc{SIS}, only a weight equal to one is considered to be legal when a template is checked, whereas in the \textsc{TIS} both zero and one are considered to be legal weights; and
\item in the \textsc{SIS}, all the potential interactions between concurrent action schemas that affect the weight of a template are considered to be problematic and so such schemas are labeled as unsafe.
\end{inparaenum}
As a consequence of these conservative choices, the \textsc{SIS} algorithm considers safe only irrelevant schemas and schemas that we define as balanced in our classification, and judges unsafe all the other action schemas. 

\subsection{Invariants, state variables and quality of the representation}
Figure \ref{fig:Invariants-IPC-2008} provides the readers with examples of the invariants that our algorithm finds when applied to IPC domains. Each set in Figure \ref{fig:Invariants-IPC-2008} corresponds to a set $\cC$ of components, which are separated by a comma and indicated with the relation name (arity is omitted here), the positions of the fixed arguments (not enclosed in square brackets) and the position of the counted variable (enclosed in square brackets). For example, \{at 0 [1], in 0 [1]\} indicates the invariant with the components $c_1 =  \langle \verb+at+, 2, 1 \rangle$ and $c_2 = \langle \verb+in+, 2, 1 \rangle$. For these domains, the only admissible partition is the trivial one and so it is not indicated. The only exception is the $\sfs{Parkprinter-strips}$ domain and the invariant \{hasimage 0 1 [2], notprintedwith 0 1 [2]\}, in which we choose an admissible partition that connects together the fixed arguments in position 0 and those in position 1.

\begin{figure}[!htbp]
\centering\includegraphics[width=120mm]{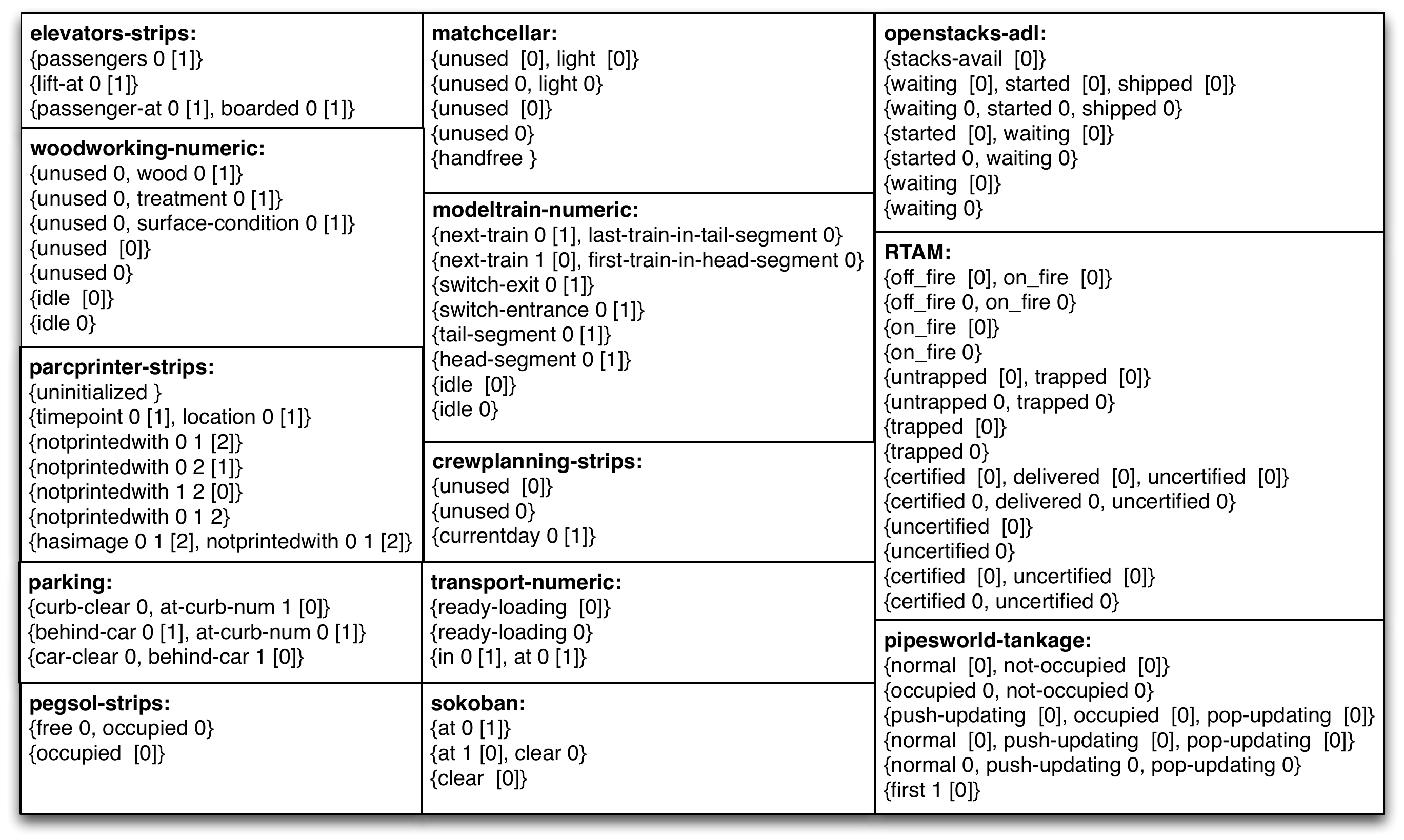}
\caption{\footnotesize Examples of invariants for the temporal domains of the IPCs. Each invariant is enclosed in braces where the predicate names indicate the components of the invariant, the numbers not enclosed in square brackets indicate the positions of the fixed variables in the list of arguments of the corresponding predicate and numbers enclosed in square brackets indicate the position of the counted variables. }\label{fig:Invariants-IPC-2008}
\end{figure}

Table \ref{fig:tableIS}-Left reports the number of invariants (\textsc{\# Inv}), number of invariants obtained by applying fixes (\textsc{\# Fix}) and run time (\textsc{RT}) for generating invariants for the temporal domains of the IPC-6, IPC-7, and IPC-8 when the \textsc{TIS} algorithm is applied. The first two columns of Table \ref{fig:tableIS}-Left also compares the number of invariants found by the \textsc{TIS} to those found by the \textsc{SIS}. This table shows that the computational time to compute invariants is negligible and that there is no significant delay associated with splitting each action schema in its initial and final parts and with checking a broad set of configurations in the schemas' conditions and effects. While these features of our algorithm do not impact the computational time, they allow us to find a more comprehensive set of invariants than related techniques.

Tables \ref{fig:tableIS}-Right and \ref{fig:TableSV-IPC8} show a comparison between the number of state variables obtained by instantiating invariants for the domains of the IPC-6, IPC-7, and IPC-8 obtained by applying our \textsc{TIS} and the \textsc{SIS} as well as by simply producing a state variable with two truth values (true and false) for each atom in the domain (Basic Invariant Synthesis, \textsc{BIS}). In many domains, the \textsc{TIS} produces a significant reduction in the number of state variables in comparison with the other two techniques. In several cases (see instances of $\sfs{Elevators}$, $\sfs{Sokoban}$, $\sfs{Transport}$, $\sfs{Drivelog}$, and $\sfs{Parking}$), the reduction is greater than an order of magnitude. In addition, Tables \ref{fig:tableIS}-Right and \ref{fig:TableSV-IPC8} report the mean (M) of the number of values in the domain of each state variable (when different from 2). In the \textsc{BIS}, the mean is always two as each state variable can only assume two values, true and false, and it is not indicated. In the \textsc{SIS}, the mean is often two with some exceptions, whereas in the \textsc{TIS} several state variables have larger domains with many values (see, for example, $\sfs{Elevators}$, $\sfs{Transport}$, $\sfs{Sokoban}$, $\sfs{Parking}$ and $\sfs{Drivelog}$). 

\begin{table}[!htbp]
\centering
{\tiny
\begin{tabular}{c@{\hspace{1pt}}|c@{\hspace{1pt}}|c@{\hspace{1pt}}|c@{\hspace{1pt}}|c@{\hspace{4pt}}|}\cline{1-5}
 \multicolumn{1}{|l@{\hspace{2pt}}}{Domains} &   
 \multicolumn{1}{|c@{\hspace{2pt}}}{\textsc{\# Inv SIS}} & 
 \multicolumn{1}{|c@{\hspace{2pt}}}{\textsc{\# Inv TIS}} &
 \multicolumn{1}{|c@{\hspace{2pt}}}{\textsc{\# Fix TIS}} & 
 \multicolumn{1}{|c@{\hspace{2pt}}|}{\textsc{RT TIS}} \\
 \hline\hline
  \multicolumn{1}{|l@{\hspace{2pt}}|}{IPC-6} &  &  &  &  \\
  \hline\hline
  \multicolumn{1}{|l@{\hspace{2pt}}|}{CrewPlanning} & 0 & 3 & 0 & 0.29 \\
  \multicolumn{1}{|l@{\hspace{2pt}}|}{Elevators-Num} & 0 & 2 & 1 & 0.02 \\ 
  \multicolumn{1}{|l@{\hspace{2pt}}|}{Elevators-Str} & 0 & 3 & 1 & 0.02 \\ 
  \multicolumn{1}{|l@{\hspace{2pt}}|}{Modeltrain-Num} & 3 & 8 & 2 & 0.23 \\
  \multicolumn{1}{|l@{\hspace{2pt}}|}{Openstacks-Adl} & 2 & 7 & 4 & 0.01 \\  
  \multicolumn{1}{|l@{\hspace{2pt}}|}{Openstacks-Num} & 4 & 10 & 6 & 0.06 \\ 
  \multicolumn{1}{|l@{\hspace{2pt}}|}{Openstacks-Num-Adl} & 2 & 6 & 4 & 0.01 \\ 
  \multicolumn{1}{|l@{\hspace{2pt}}|}{Openstacks-Str} & 4 & 11 & 6 & 0.09 \\ 
  \multicolumn{1}{|l@{\hspace{2pt}}|}{Parcprinter-Str} & 5 & 7 & 2 & 0.59 \\ 
  \multicolumn{1}{|l@{\hspace{2pt}}|}{Pegsol-Str} & 0 & 1 & 1 & 0.002 \\ 
  \multicolumn{1}{|l@{\hspace{2pt}}|}{Sokoban-Str} & 0 & 3 & 1 & 0.01 \\ 
  \multicolumn{1}{|l@{\hspace{2pt}}|}{Transport-Num} & 0 & 3 & 1 & 0.01 \\ 
  \multicolumn{1}{|l@{\hspace{2pt}}|}{Woodworking-Num} & 2 & 7 & 3 & 0.20 \\ 
  \hline\hline
  \multicolumn{1}{|l@{\hspace{2pt}}|}{IPC-7} &  &  &  &  \\ 
  \hline\hline 
  \multicolumn{1}{|l@{\hspace{2pt}}|}{CrewPlanning} & 0 & 3 & 0 & 0.31 \\ 
  \multicolumn{1}{|l@{\hspace{2pt}}|}{Elevators} & 0 & 3 & 1 & 0.02 \\ 
  \multicolumn{1}{|l@{\hspace{2pt}}|}{Floortile} & 0 & 5 & 2 & 0.05 \\ 
  \multicolumn{1}{|l@{\hspace{2pt}}|}{Matchcellar} & 4 & 5 & 2 & 0.003 \\
  \multicolumn{1}{|l@{\hspace{2pt}}|}{Openstacks} & 4 & 11 & 6 & 1.81 \\
  \multicolumn{1}{|l@{\hspace{2pt}}|}{Parcprinter} & 5 & 7 & 2 & 0.65 \\ 
  \multicolumn{1}{|l@{\hspace{2pt}}|}{Parking} & 0 & 3 & 3 & 0.03 \\  
  \multicolumn{1}{|l@{\hspace{2pt}}|}{Pegsol} & 0 & 1 & 1 & 0.002 \\ 
  \multicolumn{1}{|l@{\hspace{2pt}}|}{Sokoban} & 0 & 3 & 1 & 0.01 \\ 
  \multicolumn{1}{|l@{\hspace{2pt}}|}{Storage} & 0 & 3 & 2 & 0.05 \\
  \multicolumn{1}{|l@{\hspace{2pt}}|}{TMS} & 0 & 0 & 0 & 0.02 \\
  \multicolumn{1}{|l@{\hspace{2pt}}|}{TurnAndOpen} & 2 & 5 & 2 & 0.03 \\
  \hline\hline
  \multicolumn{1}{|l@{\hspace{2pt}}|}{IPC-8} &  &  &  &  \\
  \hline\hline  
  \multicolumn{1}{|l@{\hspace{2pt}}|}{Driverlog} & 0 & 2 & 2 & 0.03 \\  
  \multicolumn{1}{|l@{\hspace{2pt}}|}{Floortile} & 0 & 5 & 2 & 0.05 \\ 
  \multicolumn{1}{|l@{\hspace{2pt}}|}{Mapanalyser} & 3 & 7 & 4 & 0.04 \\ 
  \multicolumn{1}{|l@{\hspace{2pt}}|}{MatchCellar} & 4 & 5 & 2 & 0.003 \\
  \multicolumn{1}{|l@{\hspace{2pt}}|}{Parking} & 0 & 3 & 3 & 0.03\\
  \multicolumn{1}{|l@{\hspace{2pt}}|}{RTAM} & 6 & 14 & 8 & 0.20 \\  
  \multicolumn{1}{|l@{\hspace{2pt}}|}{Satellite} & 0 & 2 & 1 & 0.01 \\ 
  \multicolumn{1}{|l@{\hspace{2pt}}|}{Storage} & 0 & 3 & 2 & 0.05 \\ 
  \multicolumn{1}{|l@{\hspace{2pt}}|}{TMS} & 0 & 0 & 0 & 0.02 \\
  \multicolumn{1}{|l@{\hspace{2pt}}|}{TurnAndOpen} & 2 & 5 & 2 & 0.03 \\ 
\hline\hline
\end{tabular}
}
{\tiny
\begin{tabular}{p{1.3cm}@{\hspace{2pt}}|c@{\hspace{2pt}}|c@{\hspace{2pt}}|c@{\hspace{2pt}}|}\cline{2-4}
 \multicolumn{1}{c@{}}{Domains - IPC6} &   
 \multicolumn{1}{|c@{\hspace{2pt}}}{} & 
 \multicolumn{1}{c@{\hspace{2pt}}}{\textsc{\# SV}} &
 \multicolumn{1}{c@{\hspace{2pt}}|}{}\\
 \cline{2-4}
 \multicolumn{1}{c@{}}{} &  
 \multicolumn{1}{|c@{\hspace{2pt}}|}{\textsc{BIS}} & 
 \multicolumn{1}{c@{\hspace{2pt}}|}{\textsc{SIS}} &
 \multicolumn{1}{c@{\hspace{2pt}}|}{\textsc{TIS}} \\
\hline\hline
  \multicolumn{1}{|l@{\hspace{2pt}}|}{Crew Planning - p10} & 112 & 112 & 106 \\
  \multicolumn{1}{|l@{\hspace{2pt}}|}{Crew Planning - p20} & 270 & 270 & 261 \\ 
  \multicolumn{1}{|l@{\hspace{2pt}}|}{Crew Planning - p30} & 510 & 510 & 498 \\
  \multicolumn{1}{|l@{\hspace{2pt}}|}{Elevators-Num - p10} & 193 & 193 & 21 \\
  \multicolumn{1}{|l@{\hspace{2pt}}|}{Elevators-Num - p20} & 578 & 578 & 34 \\ 
  \multicolumn{1}{|l@{\hspace{2pt}}|}{Elevators-Num - p30} & 1216 & 1216 & 49 \\ 
  \multicolumn{1}{|l@{\hspace{2pt}}|}{Elevators-Str - p10} & 203 & 203 & 21 \\
  \multicolumn{1}{|l@{\hspace{2pt}}|}{Elevators-Str - p20} & 592 & 592 & 34 \\ 
  \multicolumn{1}{|l@{\hspace{2pt}}|}{Elevators-Str - p30} & 1240 & 1240 & 49 \\
  \multicolumn{1}{|l@{\hspace{2pt}}|}{Modeltrain-Num - p10} & 397 & 205 & 191 \\
  \multicolumn{1}{|l@{\hspace{2pt}}|}{Modeltrain-Num - p20} & 396 & 204 & 188 \\ 
  \multicolumn{1}{|l@{\hspace{2pt}}|}{Modeltrain-Num - p30} & 910 & 418 & 390 \\
  \multicolumn{1}{|l@{\hspace{2pt}}|}{Openstacks-Num - p10} & 71 & 71 & 29 \\
  \multicolumn{1}{|l@{\hspace{2pt}}|}{Openstacks-Num - p20} & 121 & 121 & 49 \\ 
  \multicolumn{1}{|l@{\hspace{2pt}}|}{Openstacks-Num - p30} & 171 & 171 & 69 \\  
  \multicolumn{1}{|l@{\hspace{2pt}}|}{Parcprinter-Str - p10} & 641 & 641 & 431\\
  \multicolumn{1}{|l@{\hspace{2pt}}|}{Parcprinter-Str - p20} & 1273 & 1273 & 673\\ 
  \multicolumn{1}{|l@{\hspace{2pt}}|}{Parcprinter-Str - p30} & 669 & 669 & 439 \\ 
  \multicolumn{1}{|l@{\hspace{2pt}}|}{Pegsol-Str - p10} & 66 & 66 & 33 \\
  \multicolumn{1}{|l@{\hspace{2pt}}|}{Pegsol-Str - p20} & 66 & 66 & 33 \\ 
  \multicolumn{1}{|l@{\hspace{2pt}}|}{Pegsol-Str - p30} & 66 & 66 & 33 \\ 
  \multicolumn{1}{|l@{\hspace{2pt}}|}{Sokoban-Str - p10} & 490 & 490 & 72\\
  \multicolumn{1}{|l@{\hspace{2pt}}|}{Sokoban-Str - p20} & 127 & 127 & 37 \\ 
  \multicolumn{1}{|l@{\hspace{2pt}}|}{Sokoban-Str - p30} & 1131 & 1131 & 75\\
  \multicolumn{1}{|l@{\hspace{2pt}}|}{Transport-Num - p10} & 1292 & 1292 & 36 \\
  \multicolumn{1}{|l@{\hspace{2pt}}|}{Transport-Num - p20} & 1292 & 1292 & 36 \\ 
  \multicolumn{1}{|l@{\hspace{2pt}}|}{Transport-Num - p30} & 1772 & 1772 & 64\\ 
  \multicolumn{1}{|l@{\hspace{2pt}}|}{Woodworking-Num - p10} & 143 & 143 & 95 \\
  \multicolumn{1}{|l@{\hspace{2pt}}|}{Woodworking-Num - p20} & 239 & 239 & 151 \\ 
  \multicolumn{1}{|l@{\hspace{2pt}}|}{Woodworking-Num - p30} & 251 & 251 & 158 \\ 
\hline\hline
\end{tabular}
}
{\tiny
\begin{tabular}{|c@{\hspace{2pt}}|c@{\hspace{2pt}}|}
 \cline{1-2}
 \multicolumn{1}{|c@{\hspace{2pt}}}{\textsc{M}} &
 \multicolumn{1}{c@{\hspace{2pt}}|}{}\\
 \cline{1-2}
 \multicolumn{1}{|c@{\hspace{2pt}}|}{\textsc{SIS}} &
 \multicolumn{1}{c@{\hspace{2pt}}|}{\textsc{TIS}} \\
\hline\hline
  \multicolumn{1}{|l@{\hspace{2pt}}|}  {2} & {2.05} \\
  \multicolumn{1}{|l@{\hspace{2pt}}|}  {2} & {2.03} \\
  \multicolumn{1}{|l@{\hspace{2pt}}|}  {2} & 2.02 \\
  \multicolumn{1}{|l@{\hspace{2pt}}|}  {2} & 10.19 \\
  \multicolumn{1}{|l@{\hspace{2pt}}|}  {2} & 18 \\
  \multicolumn{1}{|l@{\hspace{2pt}}|}  {2} & 25.81 \\ 
  \multicolumn{1}{|l@{\hspace{2pt}}|}  {2} & 10.67 \\
  \multicolumn{1}{|l@{\hspace{2pt}}|}  {2} & 18.41 \\ 
  \multicolumn{1}{|l@{\hspace{2pt}}|}  {2} & 26.30 \\
  \multicolumn{1}{|l@{\hspace{2pt}}|}  {2.93} & 3.08 \\
  \multicolumn{1}{|l@{\hspace{2pt}}|}  {2.94} & 3.11 \\ 
  \multicolumn{1}{|l@{\hspace{2pt}}|}  {3.17} & 3.33 \\ 
  \multicolumn{1}{|l@{\hspace{2pt}}|}  {2} & 3.45 \\
  \multicolumn{1}{|l@{\hspace{2pt}}|}  {2} & 3.47 \\ 
  \multicolumn{1}{|l@{\hspace{2pt}}|}  {2} & 2.5 \\ 
  \multicolumn{1}{|l@{\hspace{2pt}}|}  {2} & 2.49 \\
  \multicolumn{1}{|l@{\hspace{2pt}}|}  {2} & 2.89 \\ 
  \multicolumn{1}{|l@{\hspace{2pt}}|}  {2} & 2.52 \\ 
  \multicolumn{1}{|l@{\hspace{2pt}}|}  {2} & 3 \\
  \multicolumn{1}{|l@{\hspace{2pt}}|}  {2} & 3 \\ 
  \multicolumn{1}{|l@{\hspace{2pt}}|}  {2} & 3 \\ 
  \multicolumn{1}{|l@{\hspace{2pt}}|}  {2} & 7.8 \\
  \multicolumn{1}{|l@{\hspace{2pt}}|}  {2} & 4.43 \\ 
  \multicolumn{1}{|l@{\hspace{2pt}}|}  {2} & 16.08 \\
  \multicolumn{1}{|l@{\hspace{2pt}}|}  {2} & 36.89 \\
  \multicolumn{1}{|l@{\hspace{2pt}}|}  {2} & 36.89 \\ 
  \multicolumn{1}{|l@{\hspace{2pt}}|}  {2} & 28.69 \\ 
  \multicolumn{1}{|l@{\hspace{2pt}}|}  {2} & 2.55 \\
  \multicolumn{1}{|l@{\hspace{2pt}}|}  {2} & 2.62 \\ 
  \multicolumn{1}{|l@{\hspace{2pt}}|}  {2} & 2.62 \\ 
\hline\hline
\end{tabular}
}
\caption[]{{\bf Left:} Number of invariants (\textsc{\# Inv}), number of invariants obtained by repairing failed templates (\textsc{\# Fix}) and run time (\textsc{RT}) for generating invariants for the temporal domains of the IPCs by using the Temporal Invariant Synthesis (\textsc{TIS}) and the Simple Invariant Synthesis (\textsc{SIS}).

{\bf Right:} Number of state variables (\textsc{\# SV}) and mean (\textsc{M}) of the number of the values in the domain of each state variable (when different from 2) for the temporal domains of the IPC-8. The state variables are obtained by instantiating invariants obtained by applying: (1) Basic Invariant Synthesis (\textsc{BIS}); (2) Simple Invariant Synthesis (\textsc{SIS}); and (3) Temporal Invariant Synthesis (\textsc{TIS}).}
\label{fig:tableIS}
\end{table}

\begin{table*}[!htbp]
\centering
{\tiny
\begin{tabular}{p{1.3cm}@{\hspace{2pt}}|c@{\hspace{2pt}}|c@{\hspace{2pt}}|c@{\hspace{2pt}}|}\cline{2-4}
 \multicolumn{1}{c@{}}{Domains - IPC7} &   
 \multicolumn{1}{|c@{\hspace{2pt}}}{} & 
 \multicolumn{1}{c@{\hspace{2pt}}}{\textsc{\# SV}} &
 \multicolumn{1}{c@{\hspace{2pt}}|}{}\\
 \cline{2-4}
 \multicolumn{1}{c@{}}{} &  
 \multicolumn{1}{|c@{\hspace{2pt}}|}{\textsc{BIS}} & 
 \multicolumn{1}{c@{\hspace{2pt}}|}{\textsc{SIS}} &
 \multicolumn{1}{c@{\hspace{2pt}}|}{\textsc{TIS}} \\
\hline\hline
  \multicolumn{1}{|l@{\hspace{2pt}}|}{Crew Planning - p01} & 74 & 74 & 71 \\
  \multicolumn{1}{|l@{\hspace{2pt}}|}{Crew Planning - p10} & 240 & 240 & 231 \\ 
  \multicolumn{1}{|l@{\hspace{2pt}}|}{Crew Planning - p20} & 305 & 305 & 296 \\ 
  \multicolumn{1}{|l@{\hspace{2pt}}|}{Elevators - p01} & 592 & 592 & 21 \\
  \multicolumn{1}{|l@{\hspace{2pt}}|}{Elevators - p10} & 1240 & 1240 & 49 \\ 
  \multicolumn{1}{|l@{\hspace{2pt}}|}{Elevators - p20} & 3068 & 3068 & 76 \\
  \multicolumn{1}{|l@{\hspace{2pt}}|}{Floortile - p0} & 64 & 64 & 16 \\
  \multicolumn{1}{|l@{\hspace{2pt}}|}{Floortile - p10} & 126 & 126 & 26 \\ 
  \multicolumn{1}{|l@{\hspace{2pt}}|}{Floortile - p19} & 186 & 186 & 36 \\
  \multicolumn{1}{|l@{\hspace{2pt}}|}{Matchcellar - p0} & 13 & 10 & 10 \\
  \multicolumn{1}{|l@{\hspace{2pt}}|}{Matchcellar - p10} & 53 & 40 & 40 \\ 
  \multicolumn{1}{|l@{\hspace{2pt}}|}{Matchcellar - p19} & 89 & 67 & 67 \\ 
  \multicolumn{1}{|l@{\hspace{2pt}}|}{Openstacks - p01} & 142 & 142 & 49 \\
  \multicolumn{1}{|l@{\hspace{2pt}}|}{Openstacks - p10} & 201 & 201 & 69 \\ 
  \multicolumn{1}{|l@{\hspace{2pt}}|}{Openstacks - p20} & 260 & 260 & 89 \\ 
  \multicolumn{1}{|l@{\hspace{2pt}}|}{Parcprinter - p01} & 483 & 483 & 315\\
  \multicolumn{1}{|l@{\hspace{2pt}}|}{Parcprinter - p10} & 1149 & 1149 & 609 \\ 
  \multicolumn{1}{|l@{\hspace{2pt}}|}{Parcprinter - p20} & 801 & 801 & 548 \\ 
  \multicolumn{1}{|l@{\hspace{2pt}}|}{Parking - p0} & 227 & 227 & 40\\
  \multicolumn{1}{|l@{\hspace{2pt}}|}{Parking - p10} & 458 & 458 & 58\\ 
  \multicolumn{1}{|l@{\hspace{2pt}}|}{Parking - p19} & 827 & 827 & 79 \\   
  \multicolumn{1}{|l@{\hspace{2pt}}|}{Pegsol - p01} & 66 & 66 & 33 \\
  \multicolumn{1}{|l@{\hspace{2pt}}|}{Pegsol - p10} & 66 & 66 & 33 \\ 
  \multicolumn{1}{|l@{\hspace{2pt}}|}{Pegsol - p20} & 66 & 66 & 33 \\ 
  \multicolumn{1}{|l@{\hspace{2pt}}|}{Sokoban - p01} & 332 & 332 & 57\\
  \multicolumn{1}{|l@{\hspace{2pt}}|}{Sokoban - p10} & 841 & 841 & 85 \\ 
  \multicolumn{1}{|l@{\hspace{2pt}}|}{Sokoban - p20} & 284 & 284 & 64\\
  \multicolumn{1}{|l@{\hspace{2pt}}|}{Storage - p0} & 210 & 210 & 66 \\
  \multicolumn{1}{|l@{\hspace{2pt}}|}{Storage - p10} & 522 & 522 & 138 \\ 
  \multicolumn{1}{|l@{\hspace{2pt}}|}{Storage - p19} & 854 & 854 & 214\\ 
  \multicolumn{1}{|l@{\hspace{2pt}}|}{TMS - p0} & 5151 & 5151 & 5151 \\
  \multicolumn{1}{|l@{\hspace{2pt}}|}{TMS - p10} & 45451 & 45451 & 45451 \\ 
  \multicolumn{1}{|l@{\hspace{2pt}}|}{TMS - p19} & 115921 & 115921 & 115921 \\
  \multicolumn{1}{|l@{\hspace{2pt}}|}{TurnAndOpen - p0} & 182 & 182 & 161 \\
  \multicolumn{1}{|l@{\hspace{2pt}}|}{TurnAndOpen - p10} & 588 & 588 & 552 \\ 
  \multicolumn{1}{|l@{\hspace{2pt}}|}{TurnAndOpen - p20} & 1126 & 1126 & 1071 \\   
\hline\hline
\end{tabular}
}
{\tiny
\begin{tabular}{|c@{\hspace{2pt}}|c@{\hspace{2pt}}|c@{\hspace{2pt}}|}
 \cline{1-2}
 \multicolumn{1}{|c@{\hspace{2pt}}}{\textsc{M}} &
 \multicolumn{1}{c@{\hspace{2pt}}|}{}\\
 \cline{1-2}
 \multicolumn{1}{|c@{\hspace{2pt}}|}{\textsc{SIS}} &
 \multicolumn{1}{c@{\hspace{2pt}}|}{\textsc{TIS}} \\
\hline\hline
  \multicolumn{1}{|l@{\hspace{2pt}}|} {2} & {2.04} \\
  \multicolumn{1}{|l@{\hspace{2pt}}|} {2} & {2.03} \\
  \multicolumn{1}{|l@{\hspace{2pt}}|} {2} & 2.03 \\
  \multicolumn{1}{|l@{\hspace{2pt}}|} {2} & 18.41 \\
  \multicolumn{1}{|l@{\hspace{2pt}}|} {2} & 26.3 \\
  \multicolumn{1}{|l@{\hspace{2pt}}|} {2} & 41.37 \\ 
  \multicolumn{1}{|l@{\hspace{2pt}}|} {2} & 5 \\
  \multicolumn{1}{|l@{\hspace{2pt}}|} {2} & 5.85 \\ 
  \multicolumn{1}{|l@{\hspace{2pt}}|} {2} & 6.17 \\
  \multicolumn{1}{|l@{\hspace{2pt}}|} {2.93} & 2.3 \\
  \multicolumn{1}{|l@{\hspace{2pt}}|} {2.94} & 2.3 \\ 
  \multicolumn{1}{|l@{\hspace{2pt}}|} {3.17} & 2.3 \\ 
  \multicolumn{1}{|l@{\hspace{2pt}}|} {2} & 3.89 \\
  \multicolumn{1}{|l@{\hspace{2pt}}|} {2} & 3.91 \\ 
  \multicolumn{1}{|l@{\hspace{2pt}}|} {2} & 3.92 \\ 
  \multicolumn{1}{|l@{\hspace{2pt}}|} {2} & 2.53 \\
  \multicolumn{1}{|l@{\hspace{2pt}}|} {2} & 2.89 \\ 
  \multicolumn{1}{|l@{\hspace{2pt}}|} {2} & 2.46 \\
    \multicolumn{1}{|l@{\hspace{2pt}}|} {2} & 6.67 \\
  \multicolumn{1}{|l@{\hspace{2pt}}|} {2} & 8.9 \\ 
  \multicolumn{1}{|l@{\hspace{2pt}}|} {2} & 11.46 \\  
  \multicolumn{1}{|l@{\hspace{2pt}}|} {2} & 3 \\
  \multicolumn{1}{|l@{\hspace{2pt}}|} {2} & 3 \\ 
  \multicolumn{1}{|l@{\hspace{2pt}}|} {2} & 3 \\ 
  \multicolumn{1}{|l@{\hspace{2pt}}|} {2} & 6.82 \\
  \multicolumn{1}{|l@{\hspace{2pt}}|} {2} & 10.89 \\ 
  \multicolumn{1}{|l@{\hspace{2pt}}|} {2} & 5.43 \\
  \multicolumn{1}{|l@{\hspace{2pt}}|} {2} & 4.18 \\
  \multicolumn{1}{|l@{\hspace{2pt}}|} {2} & 4.78 \\ 
  \multicolumn{1}{|l@{\hspace{2pt}}|} {2} & 4.99 \\ 
  \multicolumn{1}{|l@{\hspace{2pt}}|} {2} & 2 \\
  \multicolumn{1}{|l@{\hspace{2pt}}|} {2} & 2 \\ 
  \multicolumn{1}{|l@{\hspace{2pt}}|} {2} & 2 \\ 
  \multicolumn{1}{|l@{\hspace{2pt}}|} {2} & 2.13 \\
  \multicolumn{1}{|l@{\hspace{2pt}}|} {2} & 2.06 \\ 
  \multicolumn{1}{|l@{\hspace{2pt}}|} {2} & 2.46 \\ 
\hline\hline
\end{tabular}
}
{\tiny
\begin{tabular}{p{1.3cm}@{\hspace{2pt}}|c@{\hspace{2pt}}|c@{\hspace{2pt}}|c@{\hspace{2pt}}|}\cline{2-4}
 \multicolumn{1}{c@{}}{Domains - IPC8} &   
 \multicolumn{1}{|c@{\hspace{2pt}}}{} & 
 \multicolumn{1}{c@{\hspace{2pt}}}{\textsc{\# SV}} &
 \multicolumn{1}{c@{\hspace{2pt}}|}{}\\
 \cline{2-4}
 \multicolumn{1}{c@{}}{} &  
 \multicolumn{1}{|c@{\hspace{2pt}}|}{\textsc{BIS}} & 
 \multicolumn{1}{c@{\hspace{2pt}}|}{\textsc{SIS}} &
 \multicolumn{1}{c@{\hspace{2pt}}|}{\textsc{TIS}} \\
\hline\hline
  \multicolumn{1}{|l@{\hspace{2pt}}|}{Driverlog - p01} & 535 & 535 & 25 \\
  \multicolumn{1}{|l@{\hspace{2pt}}|}{Driverlog - p10} & 8512 & 8512 & 79 \\ 
  \multicolumn{1}{|l@{\hspace{2pt}}|}{Driverlog - p15} & 3822 & 3822 & 54 \\ 
  \multicolumn{1}{|l@{\hspace{2pt}}|}{Floortile - p01} & 104 & 104  & 24 \\
  \multicolumn{1}{|l@{\hspace{2pt}}|}{Floortile - p10} & 126 & 126  & 26 \\ 
  \multicolumn{1}{|l@{\hspace{2pt}}|}{Floortile - p20} & 154 & 154 & 34 \\
  \multicolumn{1}{|l@{\hspace{2pt}}|}{Mapanalyser - p01} & 215 & 215 & 174 \\
  \multicolumn{1}{|l@{\hspace{2pt}}|}{Mapanalyser - p10} & 752  &  752 & 670 \\ 
  \multicolumn{1}{|l@{\hspace{2pt}}|}{Mapanalyser - p20} & 854 & 854 & 722 \\ 
  \multicolumn{1}{|l@{\hspace{2pt}}|}{Matchcellar - p01} & 50 & 35 & 35 \\
  \multicolumn{1}{|l@{\hspace{2pt}}|}{Matchcellar - p10} & 80 & 55 & 55 \\ 
  \multicolumn{1}{|l@{\hspace{2pt}}|}{Matchcellar - p20} & 107 & 73 & 73 \\ 
  \multicolumn{1}{|l@{\hspace{2pt}}|}{Parking - p01} & 584 & 584 & 66\\
  \multicolumn{1}{|l@{\hspace{2pt}}|}{Parking - p10} & 428 & 428 & 56\\ 
  \multicolumn{1}{|l@{\hspace{2pt}}|}{Parking - p15} & 584 & 584  & 66 \\
  \multicolumn{1}{|l@{\hspace{2pt}}|}{RTAM - p01} & 1279 & 1279 & 1133\\
  \multicolumn{1}{|l@{\hspace{2pt}}|}{RTAM - p10} & 1498 & 1498 & 1325 \\ 
  \multicolumn{1}{|l@{\hspace{2pt}}|}{RTAM - p20} & 3114 & 3114 & 2817\\
  \multicolumn{1}{|l@{\hspace{2pt}}|}{Satellite - p01} & 335 & 335 & 190 \\
  \multicolumn{1}{|l@{\hspace{2pt}}|}{Satellite - p10} & 1161 & 1161 & 926 \\ 
  \multicolumn{1}{|l@{\hspace{2pt}}|}{Satellite - p20} & 1565 & 1565 & 1166 \\ 
  \multicolumn{1}{|l@{\hspace{2pt}}|}{Storage - p01} & 244 & 244 & 76 \\
  \multicolumn{1}{|l@{\hspace{2pt}}|}{Storage - p10} & 522 & 522 & 138 \\ 
  \multicolumn{1}{|l@{\hspace{2pt}}|}{Storage - p20} &  522 & 522 & 138 \\ 
  \multicolumn{1}{|l@{\hspace{2pt}}|}{TMS - p01} & 20301 & 20301 & 20301 \\
  \multicolumn{1}{|l@{\hspace{2pt}}|}{TMS - p10} & 72771 & 72771 & 72771 \\ 
  \multicolumn{1}{|l@{\hspace{2pt}}|}{TMS - p20} & 169071  & 169071 & 169071 \\
  \multicolumn{1}{|l@{\hspace{2pt}}|}{TurnAndOpen - p0} & 206 & 206  & 185 \\
  \multicolumn{1}{|l@{\hspace{2pt}}|}{TurnAndOpen - p10} & 1166 & 1166 & 1111 \\ 
  \multicolumn{1}{|l@{\hspace{2pt}}|}{TurnAndOpen - p20} & 1676 & 1676 & 1598 \\  
\hline\hline
\end{tabular}
}
{\tiny
\begin{tabular}{|c@{\hspace{2pt}}|c@{\hspace{2pt}}|c@{\hspace{2pt}}|}
 \cline{1-2}
 \multicolumn{1}{|c@{\hspace{2pt}}}{\textsc{M}} &
 \multicolumn{1}{c@{\hspace{2pt}}|}{}\\
 \cline{1-2}
 \multicolumn{1}{|c@{\hspace{2pt}}|}{\textsc{SIS}} &
 \multicolumn{1}{c@{\hspace{2pt}}|}{\textsc{TIS}} \\
\hline\hline
  \multicolumn{1}{|l@{\hspace{2pt}}|} {2} & {22.4} \\
  \multicolumn{1}{|l@{\hspace{2pt}}|} {2} & {108.74} \\
  \multicolumn{1}{|l@{\hspace{2pt}}|} {2} & {71.78} \\
  \multicolumn{1}{|l@{\hspace{2pt}}|} {2} & {5.33} \\
  \multicolumn{1}{|l@{\hspace{2pt}}|} {2} & {5.84} \\
  \multicolumn{1}{|l@{\hspace{2pt}}|} {2} & {5.53} \\ 
  \multicolumn{1}{|l@{\hspace{2pt}}|} {2} & {2.24} \\
  \multicolumn{1}{|l@{\hspace{2pt}}|} {2} & {2.12} \\ 
  \multicolumn{1}{|l@{\hspace{2pt}}|} {2} & {2.18} \\
  \multicolumn{1}{|l@{\hspace{2pt}}|} {2.43} & {2.43} \\
  \multicolumn{1}{|l@{\hspace{2pt}}|} {2.45} & {2.45} \\ 
  \multicolumn{1}{|l@{\hspace{2pt}}|} {2.46} & {2.46} \\ 
  \multicolumn{1}{|l@{\hspace{2pt}}|} {2} & {9.84} \\
  \multicolumn{1}{|l@{\hspace{2pt}}|} {2} & {8.64} \\ 
  \multicolumn{1}{|l@{\hspace{2pt}}|} {2} & {9.85} \\ 
  \multicolumn{1}{|l@{\hspace{2pt}}|} {2} & {2.15} \\
  \multicolumn{1}{|l@{\hspace{2pt}}|} {2} & {2.15} \\ 
  \multicolumn{1}{|l@{\hspace{2pt}}|} {2} & {2.13} \\ 
  \multicolumn{1}{|l@{\hspace{2pt}}|} {2} & {2.76} \\
  \multicolumn{1}{|l@{\hspace{2pt}}|} {2} & {2.25} \\ 
  \multicolumn{1}{|l@{\hspace{2pt}}|} {2} & {2.34} \\ 
  \multicolumn{1}{|l@{\hspace{2pt}}|} {2} & {4.21} \\
  \multicolumn{1}{|l@{\hspace{2pt}}|} {2} & {4.78} \\ 
  \multicolumn{1}{|l@{\hspace{2pt}}|} {2} & {4.78} \\
  \multicolumn{1}{|l@{\hspace{2pt}}|} {2} & {2} \\
  \multicolumn{1}{|l@{\hspace{2pt}}|} {2} & {2} \\ 
  \multicolumn{1}{|l@{\hspace{2pt}}|} {2} & {2} \\ 
  \multicolumn{1}{|l@{\hspace{2pt}}|} {2} & {2.11} \\
  \multicolumn{1}{|l@{\hspace{2pt}}|} {2} & {2.05} \\ 
  \multicolumn{1}{|l@{\hspace{2pt}}|} {2} & {2.04} \\ 
\hline\hline
\end{tabular}
}
\caption[]{Number of state variables (\textsc{\# SV}) and mean (\textsc{M}) of the number of the values in the domain of each state variable (when different from 2) for the temporal domains of the IPC-7 and IPC-8. The state variables are obtained by instantiating invariants obtained by applying: (1) Basic Invariant Synthesis (\textsc{BIS}); (2) Simple Invariant Synthesis (\textsc{SIS}); and (3) Temporal Invariant Synthesis (\textsc{TIS}).}
\label{fig:TableSV-IPC8}
\end{table*}

\subsection{Performance in Temporal Planners}

We have performed a number of additional experiments in order to evaluate the impact of using the state variables generated by the \textsc{TIS} on the performance of those planners that use a variable/value representation. In particular, we focus here on the performance of two planners: Temporal Fast Downward (TFD) \citep{Eyerich-09} and POPF-SV, a version of POPF \citep{coles-10} that makes use of multi-valued state variables\footnote{This version of POPF is not documented, but it has been made available to us by their authors Andrew Coles and Amanda Coles}. 

TFD is a planning system for temporal and numeric problems based on Fast Downward (FD) \citep{helmert-jair-2006}, which is limited to non temporal and non numeric domains. TFD uses a multi-valued variable representation called ``Temporal Numeric SAS$^{+}$'' (TN-SAS$^{+}$), which is a direct extension of the ``Finite Domain Representation'' (FDR) used within FD to handle tasks with time and numeric fluents. TN-SAS$^{+}$ captures all the features of PDDL - Level 3 and represents planning tasks by using:
\begin{inparaenum}[i)]
\item a set of state variables, which are divided into logical and numeric state variables; 
\item a set of axioms, which are used to represent logical dependencies and arithmetic sub-terms; and
\item a set of durative actions, which comprise: a) a duration variable; b) start, persistent and end conditions; and c) start and end effects. 
\end{inparaenum}

TFD translates PDDL2.1 tasks into TN-SAS$^{+}$ tasks first and then performs a heuristic search in the space of time-stamped states by using a context-enhanced additive heuristic \citep{Helmert-08} extended to handle time and numeric fluents. The translation from PDDL2.1 to TN-SAS$^{+}$ works in four steps. First, the PDDL instance is normalised, i.e. types are removed and conditions and effects are simplified. Then, an instance where all the literals are ground is produced through a grounding step and the invariant synthesis is applied to generate invariants (the grounding and the invariant synthesis can be performed in parallel). Starting from the invariants provided by the invariant synthesis and the ground domain, a set of multi-valued state variables is generated. Finally, a set of actions is obtained starting from PDDL actions, which describe how the state variables change over time.

In our experiments, we modified TFD by substituting the original invariant synthesis with the three alternative versions: \textsc{TIS}, \textsc{BIS} and \textsc{SIS}. The first is our technique, which we want to evaluate. The \textsc{BIS} is used as a baseline for our experiments. The \textsc{SIS}, as we explained above, is a simple but general alternative to generate invariants\footnote{We do not offer a direct comparison between the original TFD and TFD integrated with our \textsc{TIS} because on IPC6 domains the two versions have comparable performance, while in the other domains TFD does not perform well. For big instances, it often produces no plans with a time bound of 10 minutes. As already mentioned, this behaviour can be linked to the fact that the TFD invariant synthesis works well for IPC6 domains, but it is not general enough to capture invariants in other domains.}.

Tables \ref{fig:TableTFD-2} display the search time (\textsc{ST}) in seconds for the planner TFD and the domains of the IPC-6, IPC-7 and IPC-8. The three different columns indicate the search time when state variables are obtained by applying the \textsc{BIS}, the \textsc{SIS} and the \textsc{TIS}. For each domain, the planner is run against the following problems: p01, p05, p10, p15, p20, p25 and p30. The dash symbol indicates that a plan is not found in 300 seconds. Problems for which a plan could not be found in 300 seconds by applying all three techniques do not appear in the table. 

The tables show that in several domains having fewer state variables with larger domains is beneficial to the search. In particular, in domains such as $\sfs{Elevators}$, $\sfs{Sokoban}$, $\sfs{Transport}$, $\sfs{Mapanalyser}$, and $\sfs{TurnAndOpen}$, the gain is high. If we analyse Table \ref{fig:TableTFD-2} in combination with Tables \ref{fig:tableIS} and \ref{fig:TableSV-IPC8}, we see that there is a strong correlation between the mean of the cardinality of the domains of the state variables and the impact of the state variables on performance. More specifically, state variables whose value domains have mean cardinality greater or equal to three seem to produce the strongest improvements on the search, whereas state variables whose value domains have mean cardinality around two do not yield significant differences in performance. This is not particularly surprising because the use of variables with only two values resolves in producing the same state space as it is obtained without applying any invariant synthesis. On the other hand, variables with three or more values provide an actual reduction in the number of states. 

We speculate that the number of the state variables and the cardinality of their domains has an impact on the calculation of the heuristic estimates. In particular, let us consider the additive heuristic plus context used within TFD in terms of its corresponding causal graph interpretation \cite{Helmert-08}. The causal graph heuristic gives an estimate of the number of actions needed to reach the goal from a state $s$ in terms of the estimated costs of changing the value of each state variable that appears in the goal from its value in $s$ to its value in the goal. In order to compute this estimate, the heuristic uses two structures: the domain transition graph, which describes the relations between the different values of a state variable, and the causal graph, which describes the dependencies between the different state variables. Both structures are highly influenced by the number of the state variables and the cardinality of their domains. Fewer state variables with larger domains result in a smaller number of domain transition graphs, each of which has a more complex structure. In addition, a smaller number of state variables gives rise to a much more compact and structured causal graph. We believe that, in a number of cases, these different configurations of the two types of graphs produce more effective heuristic estimates with the consequence of improving performance.

\begin{table*}[!htbp]
\centering
{\scriptsize
\begin{tabular}{p{1.3cm}@{\hspace{2pt}}|c@{\hspace{2pt}}|c@{\hspace{2pt}}|c@{\hspace{2pt}}|}\cline{2-4}
 \multicolumn{1}{c@{}}{Domains - IPC6} &   
 \multicolumn{1}{|c@{\hspace{2pt}}}{} & 
 \multicolumn{1}{c@{\hspace{2pt}}}{\textsc{TFD - ST}} &
 \multicolumn{1}{c@{\hspace{2pt}}|}{}\\
 \cline{2-4}
 \multicolumn{1}{c@{}}{} &  
 \multicolumn{1}{|c@{\hspace{2pt}}|}{\textsc{BIS}} & 
 \multicolumn{1}{c@{\hspace{2pt}}|}{\textsc{SIS}} &
 \multicolumn{1}{c@{\hspace{2pt}}|}{\textsc{TIS}} \\
\hline\hline
  \multicolumn{1}{|l@{\hspace{2pt}}|}{Crew Planning - p01} & 0 & 0 & 0 \\ 
  \multicolumn{1}{|l@{\hspace{2pt}}|}{Crew Planning - p05} & 0.02 & 0.02 & 0.02 \\
  \multicolumn{1}{|l@{\hspace{2pt}}|}{Crew Planning - p10} & 0.14 & 0.14 & 0.14 \\
  \multicolumn{1}{|l@{\hspace{2pt}}|}{Crew Planning - p15} & 0.02 & 0.02 & 0.02 \\
  \multicolumn{1}{|l@{\hspace{2pt}}|}{Crew Planning - p20} & 0.54 & 0.54 & 0.53 \\ 
  \multicolumn{1}{|l@{\hspace{2pt}}|}{Crew Planning - p30} & 1.46 & 1.46 & 1.46 \\
  \multicolumn{1}{|l@{\hspace{2pt}}|}{Elevators-Num - p01} & 1.57 & 1.57 & 0.02 \\ 
  \multicolumn{1}{|l@{\hspace{2pt}}|}{Elevators-Num - p05} & 0.32 & 0.32 & 0.05 \\ 
  \multicolumn{1}{|l@{\hspace{2pt}}|}{Elevators-Num - p10} & - & - & 57.56 \\
  \multicolumn{1}{|l@{\hspace{2pt}}|}{Elevators-Num - p15} & - & - & 11.74 \\ 
  \multicolumn{1}{|l@{\hspace{2pt}}|}{Elevators-Num - p25} & - & - & 32.85 \\ 
  \multicolumn{1}{|l@{\hspace{2pt}}|}{Elevators-Str - p01} & 0.34 & 0.33 & 0.03 \\
  \multicolumn{1}{|l@{\hspace{2pt}}|}{Elevators-Str - p05} & 0.27 & 0.28 & 0.08 \\
  \multicolumn{1}{|l@{\hspace{2pt}}|}{Elevators-Str - p10} & - & - & 23.04 \\
  \multicolumn{1}{|l@{\hspace{2pt}}|}{Elevators-Str - p15} & - & - & 23.98 \\
  \multicolumn{1}{|l@{\hspace{2pt}}|}{Openstacks-Adl - p01} & 0 & 0 & 0 \\
  \multicolumn{1}{|l@{\hspace{2pt}}|}{Openstacks-Adl - p05} & 0.01 & 0.01 & 0.01 \\
  \multicolumn{1}{|l@{\hspace{2pt}}|}{Openstacks-Adl - p10} & 0.04 & 0.04 & 0.03 \\
  \multicolumn{1}{|l@{\hspace{2pt}}|}{Openstacks-Adl - p15} & 0.09 & 0.09 & 0.07 \\
  \multicolumn{1}{|l@{\hspace{2pt}}|}{Openstacks-Adl - p20} & 0.2 & 0.2 & 0.14 \\
  \multicolumn{1}{|l@{\hspace{2pt}}|}{Openstacks-Adl - p25} & 0.34 & 0.37 & 0.25 \\ 
  \multicolumn{1}{|l@{\hspace{2pt}}|}{Openstacks-Adl - p30} & 0.62 & 0.74 & 0.46 \\ 
  \multicolumn{1}{|l@{\hspace{2pt}}|}{Openstacks-Num - p01} & 0 & 0 & 0 \\
  \multicolumn{1}{|l@{\hspace{2pt}}|}{Openstacks-Num - p05} & 0 & 0 & 0 \\
  \multicolumn{1}{|l@{\hspace{2pt}}|}{Openstacks-Num - p10} & 0.02 & 0.02 & 0.01 \\
  \multicolumn{1}{|l@{\hspace{2pt}}|}{Openstacks-Num - p15} & 0.03 & 0.03 & 0.03 \\
  \multicolumn{1}{|l@{\hspace{2pt}}|}{Openstacks-Num - p20} & 0.07 & 0.07 & 0.06 \\
  \multicolumn{1}{|l@{\hspace{2pt}}|}{Openstacks-Num - p25} & 0.1 & 0.11 & 0.1 \\ 
  \multicolumn{1}{|l@{\hspace{2pt}}|}{Openstacks-Num - p30} & 0.19 & 0.2 & 0.18 \\ 
  \multicolumn{1}{|l@{\hspace{2pt}}|}{Parcprinter-Str - p01} & 0 & 0 & 0 \\
  \multicolumn{1}{|l@{\hspace{2pt}}|}{Parcprinter-Str - p05} & 1.41 & 1.44 & - \\
  \multicolumn{1}{|l@{\hspace{2pt}}|}{Parcprinter-Str - p30} & 21.51 & 21.7 & - \\
  \multicolumn{1}{|l@{\hspace{2pt}}|}{Pegsol-Str - p01} & 0 & 0 & 0 \\
  \multicolumn{1}{|l@{\hspace{2pt}}|}{Pegsol-Str - p05} & 0 & 0 & 0 \\    
  \multicolumn{1}{|l@{\hspace{2pt}}|}{Pegsol-Str - p10} & 0.51 & 0.52 & 0.41 \\ 
  \multicolumn{1}{|l@{\hspace{2pt}}|}{Pegsol-Str - p15} & 1.27 & 1.3 & 1.13 \\    
  \multicolumn{1}{|l@{\hspace{2pt}}|}{Pegsol-Str - p20} & 0.11 & 0.11 & 0.15 \\
  \multicolumn{1}{|l@{\hspace{2pt}}|}{Pegsol-Str - p25} & 1.42 & 1.45 & 1.11 \\     
  \multicolumn{1}{|l@{\hspace{2pt}}|}{Sokoban-Str - p01} & 1.4 & 1.41 & 0.22 \\
  \multicolumn{1}{|l@{\hspace{2pt}}|}{Sokoban-Str - p05} & 94.55 & 95.1 & 10.73 \\    
  \multicolumn{1}{|l@{\hspace{2pt}}|}{Sokoban-Str - p15} & 31.6 & 32.13 & 11.89 \\    
  \multicolumn{1}{|l@{\hspace{2pt}}|}{Sokoban-Str - p20} & 3.96 & 4.02 & 0.04  \\
  \multicolumn{1}{|l@{\hspace{2pt}}|}{Transport-Num - p01} & 0 & 0 & 0 \\
  \multicolumn{1}{|l@{\hspace{2pt}}|}{Transport-Num - p05} & - & - & 13.64 \\    
  \multicolumn{1}{|l@{\hspace{2pt}}|}{Woodworking-Num - p01} & 0 & 0 & 0 \\
  \multicolumn{1}{|l@{\hspace{2pt}}|}{Woodworking-Num - p05} & 0.01 & 0.01 & 0.2 \\      
  \multicolumn{1}{|l@{\hspace{2pt}}|}{Woodworking-Num - p10} & 0.02 & 0.02 & 0.08 \\
  \multicolumn{1}{|l@{\hspace{2pt}}|}{Woodworking-Num - p15} & 0.03 & 0.03 & 0.01 \\    
  \multicolumn{1}{|l@{\hspace{2pt}}|}{Woodworking-Num - p20} & 0.08 & 0.08 & 0.06 \\
  \multicolumn{1}{|l@{\hspace{2pt}}|}{Woodworking-Num - p25} & 28.86 & 29.47 & 19.8 \\     
  \multicolumn{1}{|l@{\hspace{2pt}}|}{Woodworking-Num - p30} & 28.45 & 29.17 & 0.18 \\  
\hline\hline
\end{tabular}
}
{\scriptsize
\begin{tabular}{p{1.3cm}@{\hspace{2pt}}|c@{\hspace{2pt}}|c@{\hspace{2pt}}|c@{\hspace{2pt}}|}\cline{2-4}
 \multicolumn{1}{c@{}}{Domains - IPC7} &   
 \multicolumn{1}{|c@{\hspace{2pt}}}{} & 
 \multicolumn{1}{c@{\hspace{2pt}}}{\textsc{ST}} &
 \multicolumn{1}{c@{\hspace{2pt}}|}{}\\
 \cline{2-4}
 \multicolumn{1}{c@{}}{} &  
 \multicolumn{1}{|c@{\hspace{2pt}}|}{\textsc{BIS}} & 
 \multicolumn{1}{c@{\hspace{2pt}}|}{\textsc{SIS}} &
 \multicolumn{1}{c@{\hspace{2pt}}|}{\textsc{TIS}} \\
\hline\hline
  \multicolumn{1}{|l@{\hspace{2pt}}|}{Crew Planning - p01} & 0.02 & 0.02 & 0.02 \\
  \multicolumn{1}{|l@{\hspace{2pt}}|}{Crew Planning - p05} & 0.05 & 0.05 & 0.05 \\
  \multicolumn{1}{|l@{\hspace{2pt}}|}{Crew Planning - p10} & 0.52 & 0.52 & 0.52 \\ 
  \multicolumn{1}{|l@{\hspace{2pt}}|}{Crew Planning - p15} & 1.52 & 1.51 & 1.5 \\
  \multicolumn{1}{|l@{\hspace{2pt}}|}{Crew Planning - p20} & 0.71 & 0.7 & 0.7 \\ 
  \multicolumn{1}{|l@{\hspace{2pt}}|}{Floortile - p0} & 0.09 & 0.09 & 0.09 \\
  \multicolumn{1}{|l@{\hspace{2pt}}|}{Matchcellar - p0} & 0.01 & 0.01 & 0.01 \\
  \multicolumn{1}{|l@{\hspace{2pt}}|}{Matchcellar - p05} & 0.93 & 0.87 & 0.87 \\ 
  \multicolumn{1}{|l@{\hspace{2pt}}|}{Matchcellar - p10} & 8.64 & 8.09 & 8.06 \\
  \multicolumn{1}{|l@{\hspace{2pt}}|}{Matchcellar - p15} & 40.81 & 37.86 & 38.68 \\  
  \multicolumn{1}{|l@{\hspace{2pt}}|}{Matchcellar - p19} & 108.75 & 100.68 & 101.8 \\
  \multicolumn{1}{|l@{\hspace{2pt}}|}{Openstacks - p01} & 0.21 & 0.21 & 0.15 \\
  \multicolumn{1}{|l@{\hspace{2pt}}|}{Openstacks - p05} & 0.36 & 0.37 & 0.27 \\ 
  \multicolumn{1}{|l@{\hspace{2pt}}|}{Openstacks - p10} & 0.68 & 0.68 & 0.5 \\ 
  \multicolumn{1}{|l@{\hspace{2pt}}|}{Openstacks - p15} & 1.15 & 1.11 & 0.82 \\ 
  \multicolumn{1}{|l@{\hspace{2pt}}|}{Openstacks - p20} & 1.77 & 1.73 & 1.2 \\ 
  \multicolumn{1}{|l@{\hspace{2pt}}|}{Parking - p0} & 76.79 & 65.12 & 35.09 \\
  \multicolumn{1}{|l@{\hspace{2pt}}|}{Parking - p05} & 148.8 & 129.79 & 78.3 \\    
  \multicolumn{1}{|l@{\hspace{2pt}}|}{Parking - p10} & 72.52 & 63.01 & 87.92 \\
  \multicolumn{1}{|l@{\hspace{2pt}}|}{Pegsol - p01} & 0.09 & 0.08 & 0.08 \\
  \multicolumn{1}{|l@{\hspace{2pt}}|}{Pegsol - p05} & 0.01 & 0.01 & 0 \\  
  \multicolumn{1}{|l@{\hspace{2pt}}|}{Pegsol - p10} & 0.18 & 0.18 & 0.11 \\
  \multicolumn{1}{|l@{\hspace{2pt}}|}{Pegsol - p15} & 0.91 & 0.86 & 0.7 \\   
  \multicolumn{1}{|l@{\hspace{2pt}}|}{TurnAndOpen - p0} & 0.21 & 0.21 & 0.21 \\
  \multicolumn{1}{|l@{\hspace{2pt}}|}{TurnAndOpen - p05} & 6.36 & 6.41 & 5.67 \\  
  \multicolumn{1}{|l@{\hspace{2pt}}|}{TurnAndOpen - p10} & 4.3 & 4.37 & 3.01 \\
  \multicolumn{1}{|l@{\hspace{2pt}}|}{TurnAndOpen - p15} & 5.1 & 5.27 & 13.95 \\      
  \multicolumn{1}{|l@{\hspace{2pt}}|}{TurnAndOpen - p20} & 22.62 & 23.55 & 38.28 \\  
\hline\hline
 \multicolumn{1}{c@{}}{Domains - IPC8} &   
 \multicolumn{1}{|c@{\hspace{2pt}}}{} & 
 \multicolumn{1}{c@{\hspace{2pt}}}{\textsc{ST}} &
 \multicolumn{1}{c@{\hspace{2pt}}|}{}\\
 \cline{2-4}
 \multicolumn{1}{c@{}}{} &  
 \multicolumn{1}{|c@{\hspace{2pt}}|}{\textsc{BIS}} & 
 \multicolumn{1}{c@{\hspace{2pt}}|}{\textsc{SIS}} &
 \multicolumn{1}{c@{\hspace{2pt}}|}{\textsc{TIS}} \\
\hline\hline
  \multicolumn{1}{|l@{\hspace{2pt}}|}{Mapanalyser - p01} & 119.24 & 122.14 & 0.14 \\
  \multicolumn{1}{|l@{\hspace{2pt}}|}{Mapanalyser - p05} & 5.71 & 5.88 & 0.96 \\
  \multicolumn{1}{|l@{\hspace{2pt}}|}{Mapanalyser - p10} & 173.54 & 176.14 & 2.16 \\
  \multicolumn{1}{|l@{\hspace{2pt}}|}{Mapanalyser - p15} & - & - & 3.26 \\
  \multicolumn{1}{|l@{\hspace{2pt}}|}{Mapanalyser - p20} & - & - & 55.19 \\
  \multicolumn{1}{|l@{\hspace{2pt}}|}{Matchcellar - p01} & 3.17 & 3.06 & 3.05 \\
  \multicolumn{1}{|l@{\hspace{2pt}}|}{Matchcellar - p05} & 9.71 & 9.35 & 9.37 \\
  \multicolumn{1}{|l@{\hspace{2pt}}|}{Matchcellar - p10} & 24.49 & 23.65 & 23.64 \\
  \multicolumn{1}{|l@{\hspace{2pt}}|}{Matchcellar - p15} & 54.32 & 52.32 & 52.32 \\ 
  \multicolumn{1}{|l@{\hspace{2pt}}|}{Matchcellar - p20} & 95.16 & 90.57 & 90.64 \\ 
  \multicolumn{1}{|l@{\hspace{2pt}}|}{Parking - p01} & 1.01 & 1.04 & 0.5 \\
  \multicolumn{1}{|l@{\hspace{2pt}}|}{Parking - p05} & 0.17 & 0.18 & 0.09 \\
  \multicolumn{1}{|l@{\hspace{2pt}}|}{Parking - p10} & 0.3 & 0.31 & 0.17 \\
  \multicolumn{1}{|l@{\hspace{2pt}}|}{Parking - p15} & 1.12 & 1.14 & 0.55 \\
  \multicolumn{1}{|l@{\hspace{2pt}}|}{Satellite - p01} & 0.55 & 0.56 & 1.6 \\
  \multicolumn{1}{|l@{\hspace{2pt}}|}{Satellite - p05} & 13.57 & 13.87 & 2.93 \\ 
  \multicolumn{1}{|l@{\hspace{2pt}}|}{Satellite - p15} & 2.59 & 2.67 & 13 \\ 
  \multicolumn{1}{|l@{\hspace{2pt}}|}{TurnAndOpen - p0} & 0.21 & 0.22 & 0.22 \\
  \multicolumn{1}{|l@{\hspace{2pt}}|}{TurnAndOpen - p05} & 6.3 & 6.41 & 5.19 \\  
  \multicolumn{1}{|l@{\hspace{2pt}}|}{TurnAndOpen - p10} & 54.78 & 55.99 & 23.8 \\
  \multicolumn{1}{|l@{\hspace{2pt}}|}{TurnAndOpen - p15} & 65.83 & 67.64 & 55.88 \\ 
  \multicolumn{1}{|l@{\hspace{2pt}}|}{TurnAndOpen - p20} & 136.9 & 142.32 & 81.35 \\  
\hline\hline
\end{tabular}
}
\caption[]{Search time (\textsc{ST}) in seconds for the planner TFD and the domains of the IPC-6 -- IPC-8 (ST does not include the translation from PDDL2.1 to temporal SAS+). Invariants are obtained by applying: (1) a Basic Invariant Synthesis (\textsc{BIS}); (2) a Simple Invariant Synthesis (\textsc{SIS}); and (3) our Temporal Invariant Synthesis (\textsc{TIS}). For each domain, the planner is run against problems p01, p05, p10, p15, p20, p25 and p30. The dash symbol indicates that a plan has not been found in 300 seconds. Problems for which all the techniques do not find a plan in 300 seconds do not appear in the table.}
\label{fig:TableTFD-2}
\end{table*}

POPF-SV is a version of the forwards-chaining temporal planner POPF \cite{coles-10} that is capable of reading a variable/value representation and making use of it to perform some inference in a pre-processing step and also to reduce the size of states during search. In particular, POPF-SV reads a standard PDDL task along with its corresponding TN-SAS$^{+}$ translation and reasons with both representations. The multi-valued state variable representation of the task in not used in the heuristic computation, but it is used for two different purposes. An inference step based on the state variables is performed to support temporal preferences. This step extracts rules that are then used during search (for example, is it possible to have action $a$ within 10 time units of action $b$). The second use of the invariant analysis aims to make the state representation more efficient. Only one proposition from each mutex group needs to be stored within a state since if one is true then the others must necessarily be false. This property results in massive savings in memory. This is particularly beneficial for POPF as memory is generally what causes the planner to fail (rather than time).  

Tables \ref{fig:TablePOPF-2} display the search time (\textsc{ST}) in seconds for the planner POPF-SV and the domains of the IPC-6 and IPC-7 (we do not show the domains of the IPC-8 as the planner has not been maintained since IPC-7 and so does not perform well on those domains). The three different columns indicate the search time when state variables are obtained by applying the \textsc{BIS}, the \textsc{SIS} and the \textsc{TIS}. For each domain, the planner is run against the following problems: p01, p05, p10, p15, p20, p25 and p30. Problems for which a plan could not be found in 300 seconds do not appear in the table. 

The tables show that in several domains having fewer state variables with larger domains is beneficial to the search. In particular, in domains such as $\sfs{Elevators}$, $\sfs{Sokoban}$, $\sfs{Parking}$, $\sfs{Openstacks}$ and $\sfs{TurnAndOpen}$, the gain is significant. If we analyse Table \ref{fig:TablePOPF-2} in combination with Tables \ref{fig:tableIS} and \ref{fig:TableSV-IPC8}, we again see that there is a correlation between the number of state variables and the performance of the planner. When the reduction in the number of variables is significant, the planner works significantly better than with the \textsc{SIS} or the \textsc{BIS}. In addition, we observe a correlation between the mean of the cardinality of the domains of the state variables and the impact of the state variables on performance. 

\begin{table*}[!htbp]
\centering
{\scriptsize
\begin{tabular}{p{1.3cm}@{\hspace{2pt}}|c@{\hspace{2pt}}|c@{\hspace{2pt}}|c@{\hspace{2pt}}|}\cline{2-4}
 \multicolumn{1}{c@{}}{Domains - IPC6} &   
 \multicolumn{1}{|c@{\hspace{2pt}}}{} & 
 \multicolumn{1}{c@{\hspace{2pt}}}{\textsc{ST}} &
 \multicolumn{1}{c@{\hspace{2pt}}|}{}\\
 \cline{2-4}
 \multicolumn{1}{c@{}}{} &  
 \multicolumn{1}{|c@{\hspace{2pt}}|}{\textsc{BIS}} & 
 \multicolumn{1}{c@{\hspace{2pt}}|}{\textsc{SIS}} &
 \multicolumn{1}{c@{\hspace{2pt}}|}{\textsc{TIS}} \\
\hline\hline
  \multicolumn{1}{|l@{\hspace{2pt}}|}{Crew Planning - p01} & 0 & 0 & 0 \\ 
  \multicolumn{1}{|l@{\hspace{2pt}}|}{Crew Planning - p05} & 0 & 0 & 0.02 \\
  \multicolumn{1}{|l@{\hspace{2pt}}|}{Crew Planning - p10} & 0.04 & 0.04 & 0.04 \\
  \multicolumn{1}{|l@{\hspace{2pt}}|}{Crew Planning - p15} & 0.02 & 0.02 & 0.02 \\
  \multicolumn{1}{|l@{\hspace{2pt}}|}{Crew Planning - p20} & 0.26 & 0.30 & 0.26 \\ 
  \multicolumn{1}{|l@{\hspace{2pt}}|}{Crew Planning - p25} & 0.24 & 0.26 & 0.22 \\
  \multicolumn{1}{|l@{\hspace{2pt}}|}{Crew Planning - p30} & 0.84 & 0.82 & 0.80 \\
  \multicolumn{1}{|l@{\hspace{2pt}}|}{Elevators-Num - p01} & 0.20 & 0.18 & 0.18 \\ 
  \multicolumn{1}{|l@{\hspace{2pt}}|}{Elevators-Num - p05} & 0.24 & 0.22 & 0.24 \\ 
  \multicolumn{1}{|l@{\hspace{2pt}}|}{Elevators-Num - p10} & 7.36 & 7.64 & 7.14 \\
  \multicolumn{1}{|l@{\hspace{2pt}}|}{Elevators-Num - p15} & - & 26.30 & 23.28 \\ 
  \multicolumn{1}{|l@{\hspace{2pt}}|}{Elevators-Num - p20} & - & 276.72 & 251.54 \\ 
  \multicolumn{1}{|l@{\hspace{2pt}}|}{Elevators-Num - p25} & - & - & 32.85 \\ 
  \multicolumn{1}{|l@{\hspace{2pt}}|}{Elevators-Str - p01} & 1.16 & 1.22 & 0.94 \\
  \multicolumn{1}{|l@{\hspace{2pt}}|}{Elevators-Str - p05} & 0.62 & 0.70 & 0.48 \\
  \multicolumn{1}{|l@{\hspace{2pt}}|}{Elevators-Str - p10} & - & 49.48 & 30.22 \\
  \multicolumn{1}{|l@{\hspace{2pt}}|}{Elevators-Str - p15} & - & - & - \\
  \multicolumn{1}{|l@{\hspace{2pt}}|}{Elevators-Str - p20} & - & - & 585.21 \\ 
  \multicolumn{1}{|l@{\hspace{2pt}}|}{Elevators-Str - p25} & - & - & 468.16 \\
  \multicolumn{1}{|l@{\hspace{2pt}}|}{Openstacks-Str - p01} & 0 & 0 & 0 \\
  \multicolumn{1}{|l@{\hspace{2pt}}|}{Openstacks-Str - p05} & 0.02 & 0.02 & 0.02 \\
  \multicolumn{1}{|l@{\hspace{2pt}}|}{Openstacks-Str - p10} & 0.08 & 0.08 & 0.06 \\
  \multicolumn{1}{|l@{\hspace{2pt}}|}{Openstacks-Str - p15} & 0.18 & 0.18 & 0.14 \\
  \multicolumn{1}{|l@{\hspace{2pt}}|}{Openstacks-Str - p20} & 0.32 & 0.34 & 0.24 \\
  \multicolumn{1}{|l@{\hspace{2pt}}|}{Openstacks-Str - p25} & 0.56 & 0.58 & 0.44 \\
  \multicolumn{1}{|l@{\hspace{2pt}}|}{Openstacks-Str - p30} & 0.84 & 0.86 & 0.60 \\
  \multicolumn{1}{|l@{\hspace{2pt}}|}{Parcprinter-Str - p01} & 0.02 & 0.02 & 0.02 \\
  \multicolumn{1}{|l@{\hspace{2pt}}|}{Parcprinter-Str - p25} & - & - & 21.36 \\
  \multicolumn{1}{|l@{\hspace{2pt}}|}{Pegsol-Str - p01} & 0 & 0 & 0 \\
  \multicolumn{1}{|l@{\hspace{2pt}}|}{Pegsol-Str - p05} & 0.02 & 0 & 0 \\    
  \multicolumn{1}{|l@{\hspace{2pt}}|}{Pegsol-Str - p10} & 0.10 & 0.10 & 0.10 \\ 
  \multicolumn{1}{|l@{\hspace{2pt}}|}{Pegsol-Str - p15} & 0.20 & 0.20 & 0.18 \\    
  \multicolumn{1}{|l@{\hspace{2pt}}|}{Pegsol-Str - p20} & 0.06 & 0.06 & 0.06 \\
  \multicolumn{1}{|l@{\hspace{2pt}}|}{Pegsol-Str - p25} & 0.28 & 0.28 & 0.24 \\     
  \multicolumn{1}{|l@{\hspace{2pt}}|}{Sokoban-Str - p01} & 1.40 & 1.50 & 1.10 \\
  \multicolumn{1}{|l@{\hspace{2pt}}|}{Sokoban-Str - p05} & - & 58.98 & 30.80 \\    
  \multicolumn{1}{|l@{\hspace{2pt}}|}{Sokoban-Str - p15} & - & 11.08 & 10.60 \\    
  \multicolumn{1}{|l@{\hspace{2pt}}|}{Sokoban-Str - p20} & 2.92 & 3.22 & 2.26  \\
  \multicolumn{1}{|l@{\hspace{2pt}}|}{Transport-Num - p01} & 0 & 0 & 0 \\
  \multicolumn{1}{|l@{\hspace{2pt}}|}{Woodworking-Num - p01} & 0 & 0 & 0 \\
  \multicolumn{1}{|l@{\hspace{2pt}}|}{Woodworking-Num - p05} & 0.04 & 0.04 & 0.04 \\      
  \multicolumn{1}{|l@{\hspace{2pt}}|}{Woodworking-Num - p15} & 0.10 & 0.08 & 0.10 \\    
  \multicolumn{1}{|l@{\hspace{2pt}}|}{Woodworking-Num - p20} & 3.02 & 3.48 & 3.22 \\
  \multicolumn{1}{|l@{\hspace{2pt}}|}{Woodworking-Num - p25} & 0.10 & 0.10 & 0.08 \\     
  \multicolumn{1}{|l@{\hspace{2pt}}|}{Woodworking-Num - p30} & - & 2.09 & 2.02 \\  
\hline\hline
\end{tabular}
}
{\scriptsize
\begin{tabular}{p{1.3cm}@{\hspace{2pt}}|c@{\hspace{2pt}}|c@{\hspace{2pt}}|c@{\hspace{2pt}}|}\cline{2-4}
 \multicolumn{1}{c@{}}{Domains - IPC7} &   
 \multicolumn{1}{|c@{\hspace{2pt}}}{} & 
 \multicolumn{1}{c@{\hspace{2pt}}}{\textsc{ST}} &
 \multicolumn{1}{c@{\hspace{2pt}}|}{}\\
 \cline{2-4}
 \multicolumn{1}{c@{}}{} &  
 \multicolumn{1}{|c@{\hspace{2pt}}|}{\textsc{BIS}} & 
 \multicolumn{1}{c@{\hspace{2pt}}|}{\textsc{SIS}} &
 \multicolumn{1}{c@{\hspace{2pt}}|}{\textsc{TIS}} \\
\hline\hline
  \multicolumn{1}{|l@{\hspace{2pt}}|}{Crew Planning - p01} & 0.02 & 0.02 & 0.02 \\
  \multicolumn{1}{|l@{\hspace{2pt}}|}{Crew Planning - p05} & 0.02 & 0.02 & 0.02 \\
  \multicolumn{1}{|l@{\hspace{2pt}}|}{Crew Planning - p10} & 0.22 & 0.22 & 0.22 \\ 
  \multicolumn{1}{|l@{\hspace{2pt}}|}{Crew Planning - p15} & 0.76 & 0.78 & 0.76 \\
  \multicolumn{1}{|l@{\hspace{2pt}}|}{Crew Planning - p20} & 0.36 & 0.34 & 0.34 \\ 
  \multicolumn{1}{|l@{\hspace{2pt}}|}{Matchcellar - p0} & 0.00 & 0.00 & 0.00 \\
  \multicolumn{1}{|l@{\hspace{2pt}}|}{Matchcellar - p05} & 0.04 & 0.04 & 0.04 \\ 
  \multicolumn{1}{|l@{\hspace{2pt}}|}{Matchcellar - p10} & 0.14 & 0.14 & 0.14 \\
  \multicolumn{1}{|l@{\hspace{2pt}}|}{Matchcellar - p15} & 0.30 & 0.30 & 0.30 \\  
  \multicolumn{1}{|l@{\hspace{2pt}}|}{Matchcellar - p19} & 0.54 & 0.56 & 0.54 \\
  \multicolumn{1}{|l@{\hspace{2pt}}|}{Openstacks - p01} & 0.38 & 0.34 & 0.24 \\
  \multicolumn{1}{|l@{\hspace{2pt}}|}{Openstacks - p05} & 0.68 & 0.60 & 0.42 \\ 
  \multicolumn{1}{|l@{\hspace{2pt}}|}{Openstacks - p10} & 1.05 & 0.86 & 0.62 \\ 
  \multicolumn{1}{|l@{\hspace{2pt}}|}{Openstacks - p15} & 1.64 & 1.40 & 1.00 \\ 
  \multicolumn{1}{|l@{\hspace{2pt}}|}{Openstacks - p20} & 2.58 & 2.24 & 1.62 \\ 
  \multicolumn{1}{|l@{\hspace{2pt}}|}{Parking - p0} & 0.84 & 0.94 & 0.72 \\
  \multicolumn{1}{|l@{\hspace{2pt}}|}{Parking - p05} & 23.80 & 26.12 & 17.26 \\    
  \multicolumn{1}{|l@{\hspace{2pt}}|}{Parking - p10} & 37.46 & 40.86 & 27.34 \\
  \multicolumn{1}{|l@{\hspace{2pt}}|}{Parking - p15} & 68.28 & 70.18 & 55.74 \\     
  \multicolumn{1}{|l@{\hspace{2pt}}|}{Parking - p19} & 335.96 & 346.34 & 243.58 \\    
  \multicolumn{1}{|l@{\hspace{2pt}}|}{Pegsol - p01} & 0.16 & 0.16 & 0.14 \\
  \multicolumn{1}{|l@{\hspace{2pt}}|}{Pegsol - p05} & 0.02 & 0.02 & 0.02 \\  
  \multicolumn{1}{|l@{\hspace{2pt}}|}{Pegsol - p10} & 0.26 & 0.26 & 0.24 \\
  \multicolumn{1}{|l@{\hspace{2pt}}|}{Pegsol - p15} & 1.72 & 1.80 & 1.60 \\   
  \multicolumn{1}{|l@{\hspace{2pt}}|}{Sokoban - p01} & 129.96 & 130.02 & 88.82 \\ 
  \multicolumn{1}{|l@{\hspace{2pt}}|}{TMS - p0} & 48.0 & 51.06 & 50.28 \\ 
  \multicolumn{1}{|l@{\hspace{2pt}}|}{TurnAndOpen - p0} & 1.98 & 2.00 & 2.00 \\
  \multicolumn{1}{|l@{\hspace{2pt}}|}{TurnAndOpen - p05} & 13.92 & 15.34 & 13.24 \\  
\hline\hline
\end{tabular}
}
\caption[]{Search time (\textsc{ST}) in seconds for the planner POPF-SV and the domains of the IPC-6 and IPC-7 (ST does not include the translation from PDDL2.1 to temporal SAS+). Invariants are obtained by applying: (1) a Basic Invariant Synthesis (\textsc{BIS}); (2) a Simple Invariant Synthesis (\textsc{SIS}); and (3) our Temporal Invariant Synthesis (\textsc{TIS}). For each domain, the planner is run against problems p01, p05, p10, p15, p20, p25 and p30. The dash symbol indicates that a plan has not been found in 300 seconds. Problems for which all the techniques do not find a plan in 300 seconds do not appear in the table.}
\label{fig:TablePOPF-2}
\end{table*}


\section{Related Work}
\label{sec:related}
The invariant synthesis presented in this paper builds on early work described in \cite{Bernardini-11-sara}. However, the theory behind the invariant synthesis presented here is significantly more comprehensive and the implementation of the technique, in reflecting this extended theory, is capable of finding more invariants and more complex ones. We work here with instantaneous action schemas instead of durative ones and we resort to them only when necessary. From a theoretical point of view, this makes our presentation cleaner and facilitates the exhibition of sound results. From a practical point of view, this makes our algorithm more efficient. The entire classification of action schemas is different from \cite{Bernardini-11-sara} and, in consequence of this more sophisticated classification, our new approach can handle complex interactions between schemas that our original simplistic technique cannot. In fact, our original approach works well only in simple domains with balanced  schemas, while it fails in more complex domains because the conditions imposed on potentially interfering schemas are too conservative. 

Several other approaches to invariant synthesis are available in the literature. In what follows, we present these approaches more in depth by highlighting differences and similarities with our technique.

\subsection{Fast Downward and Temporal Fast Downward}

In \cite{Helmert-09}, Helmert present a translation from a subset of PDDL2.2 into FDR (Finite Domain Representation), a multi-valued planning task formalism used within the planner Fast Downward \citep{helmert-jair-2006}. In particular, the translation only handles non-temporal and non-numeric PDDL2.2 domains, the so-called ``PDDL Level 1'' (equivalent to STRIPS \citep{FikesNilsson-71} with the extensions known as ADL \citep{Pednault-86}). One of the step of this translation is the identification of mutual exclusion invariants and it is an extension of the technique presented in \cite{edelkamp-helmert-ecp-1999} developed for STRIPS. 

When considering non-temporal domains, the invariant synthesis presented in this paper works similarly to Helmert's one. In particular, both work at the lifted level, while all the other related techniques discussed below work at the ground level. Both techniques start from simple invariant candidates, check them against conditions that ensure invariance by analysing the structure of the action schemas in the domain. When a candidate is rejected, they both try to refine it to create a new stronger candidate, which is then checked from scratch. 

However, in contrast with our technique, Helmert's method considers safe only the weight one instead of both the weights of one and zero. This simplified analysis results in the identification of a smaller set of invariants with respect to our technique. For example, Helmert's invariant synthesis labels as unsafe all the action schemas that add a relevant literal without deleting that literal or another relevant one, even when the preconditions impose that the weight is zero when the action schema is applied. In this way, Helmert's invariant synthesis misses invariants that our technique is able to find.

\cite{CHEN2009} builds on Helmert's invariant synthesis and his multi-valued domain formulation to synthesise \emph{long-distance} mutual exclusions (londex), which capture constraints over actions and facts not only at the same time step but also across multiple steps. The londex has been successfully used in SAT-based planners to improve their performance. In future work, we will explore how the concept of londex can be extended to temporal domains.

Within the context of Temporal Fast Downward (TFD) \citep{Eyerich-09}, a simple extension of Helmert's invariant synthesis is used to deal with temporal and numeric domains of the ICPs. See Section \ref{sec:experiments} for a description of such a technique.

\subsection{Rintanen's Invariant Synthesis}
An algorithm for inferring invariants in propositional STRIPS domains is proposed by \cite{Rintanen-00,Rintanen-08}. It synthesises not only mutual-exclusion invariants, but also other types of invariants. The algorithm works on a ground representation of the domain and, starting from an inductive definition of invariants as formulae that are true in the initial state and are preserved by the application of every action, the algorithm is based on an iterative computation of a fix-point, which is useful for reasoning about all the invariants of a domain at the same time rather than inferring some invariants first and then use them for inferring others. 

Rintanen's algorithm uses a guess, check and repair approach but, unlike our technique, it starts from stronger invariant candidates and then progressively weaken them if they are not preserved by the actions. Thus, the \emph{repair} phase consists in considering a less general invariant instead of a more general one. For example, let us consider the schema $\sigma = x \neq y \rightarrow P(x, y) \vee Q(y, z)$ as a potential invariant (all the invariants considered have this implicative form). One of the weakening operation consists of identifying two variables. In this case, if $z$ is set equal to $x$, the weaker candidate $\sigma = x \neq y \rightarrow P(x, y) \vee Q(y, x)$ is obtained and checked. 

This technique has been successfully used within both Graphplan based planners \citep{BluFur97}, where it helps to identify unreachable subgoals, and SAT-based planners \citep{KautzSelman-99}, where it can be useful to reduce the amount of search needed. However, although its implementation is limited to invariants involving two literals at the most, it incurs a high performance penalty on large instances.  

In \cite{Rintanen-14}, Rintanen extends the original algorithm presented in \cite{Rintanen-00, Rintanen-08} in order to handle temporal domains. As the original algorithm, the temporal one works on ground domains, not using a lifted representation at any stage. The format of the invariants found is $l_1 V(r) \ l_2$, where $l_1$ and $l_2$ are positive or negative ground facts, $r$ is a floating point number, and the formula says that either $l_1$ is true or $l_2$ is true over the interval [0..r] relative to the current time point. If $r=inf$, the formula means that if $l_1$ is false, then $l_2$ will remain true forever. 

Since Rintanen's invariant synthesis exploits the initial conditions and the ground representation of the domain, it usually finds a broader range of invariants than our technique. However, this makes the invariant synthesis suffer high computational cost. Reachability analysis on a ground representation of the planning instances is computationally very expensive and, while our algorithm takes a few seconds to run, Rintanen's synthesis employs tens of minutes to find invariants in several domains (see Table 1 in \cite{Rintanen-14}). 

We do not directly compare our technique against Rintanen's algorithm in Section \ref{sec:experiments} because the two techniques aim to find different types of invariants (our focuses on mutual exclusion invariants, while Rintanen's tackles a broad range of invariant types) and they work on different representations of the problem (lifted versus ground). However, in what follows, we give examples of the output of Rintanen's technique for completeness.

Let us consider the $\sfs{Crewplanning}$ domain (IPC6 and IPC7). For each crew member $c_i$, Rintanen's algorithm finds ground invariants of the type: $$not \ current\_day-c_i-d_j \ V(inf) \ not \ current\_day-c_i-d_k$$ which means that if it is day $d_j$ for the crew member $c_i$, it cannot be day $d_k$ at the same time. All these invariants correspond to the lifted invariant ${current\_day \ 0 \ [1]}$ that is found by our invariant synthesis. For the same domain, however, Rintanen's algorithm finds additional invariants that express temporal relations between atoms. Our technique does not aim to find this type of invariants. For example, Rintanen's method finds temporal invariants of the form: $$done\_sleep-c_i-d_k \ V(255) \ not \ done\_meal c_i-d_{(k+2)}$$ which means that, for the crew member $c_i$, the atom $done\_meal$ in day $k+1$ becomes true 255 time units after the atom $done_sleep$ was true in day $k$. In fact, in day $k$, $done\_sleep$ is made true by the end effects of the action $sleep$. From this time point, in order to make $done\_meal$ true the day after $k+1$, two actions need to be executed: $post\_sleep$, with duration 195, and $have\_meal$ with duration 60, for a total time separation of 255 time units. For the $\sfs{Crewplanning}$ domain, the run time of our algorithm is $0.29$ seconds, while Rintanen's one has a runtime of 1 minute and $23.24$ seconds for hard instances. This is actually one of the best run times, since for problems such as \sfs{Parcprinter}, \sfs{Elevators}, \sfs{Sokoban}, \sfs{Transport-numeric} and other the algorithm has a run time of more than 4 hours. Given these run times, it does not seem plausible to use Rintanen's algorithm as a pre-processing steps to improve search in planning, which is one of the most important cases of use of invariant synthesis algorithms.

\subsection{DISCOPLAN}

DISCOPLAN (\emph{DIScovering State COnstraints for PLANning}) \citep{Gerevini-98} is a technique for generating invariants from PDDL Level 1 tasks. DISCOPLAN discovers not only mutual exclusion invariants, but also other types of invariants: static predicates, simple implicative, (strict) single valuedness and n-valuedness, anti-simmetry, OR and XOR invariants.

Considering mutual exclusion invariants, DISCOPLAN uses a guess, check and repair approach similar to our approach: an hypothetical invariant is generated by analysing simultaneously the preconditions and the effects of each action to see whether an instantiation of a literal is deleted whenever another instantiation of the same literal is added. Then, this candidate is checked against all the other actions and the initial conditions. If the hypothetical invariant is not found to be valid, then all the unsafe actions are collected together and a set of possible refinements are generated. However, whereas our technique tries to refine a candidate as soon as an unsafe action is found, DISCOPLAN tries to address all the causes of unsafety at the same time while generating refinements. This approach leads to more informed choices on how to refine hypothetical invariants and can result in the identification of more invariants. However, it is more expensive from a computational point of view, which is 
why DISCOPLAN is often inefficient on big instances. 

DISCOPLAN can be used not only for finding invariants, but also for inferring action-parameter domains. An action-parameter domain is a set including all the objects that can be used to instantiate the parameters of an action. Such sets of possible tuples of arguments are found by forward propagation of ground atoms from the initial state. This technique is related to the reachability analysis performed by Graphplan \citep{BluFur97}, but does not implement mutual exclusion calculation.

DISCOPLAN is usually used in combination with SAT encodings of planning problems. In particular, a pre-processing step is performed over the domain under consideration in order to find invariants and parameter domains, then the domain as well as the invariants and the parameter domains are translated into SAT. Finally, a SAT-based planner is used to solve the resulting translated domain. SAT-based planners \citep{KautzSelman-99,Huang-10} show significant speed-up when invariants and action-parameter domains are used. 

\subsection{Type Inference Module}
TIM (\emph{Type Inference Module}) \citep{FoxLong-98} uses a different approach for finding invariants in PDDL Level 1 domains. More precisely, TIM is a pre-preprocessing technique for inferring object types on the basis of the actions and the initial state. Data obtained from this computation are then used for inferring invariants. TIM recognises four kinds of invariants (invariants of type 2 correspond to mutual exclusion invariants):
\begin{enumerate}\itemsep=0pt
\item \emph{Identity invariants} (for example, considering the domain $\sfs{Blockworld}$, two objects cannot be at the same place at the same time);
\item \emph{Unique state invariants} (for example, every object must be in at most one place at any time point);
\item \emph{State membership invariants} (for example, every object must be in at least one place at any time point); and
\item \emph{Resource invariants} (for example, in a 3-blocks world, there are 4 surfaces).
\end{enumerate}

The invariants found by TIM have been exploited for improving performance within the planner STAN \citep{fox2011}.

\subsection{Knowledge representation and engineering}

In addition to works that address the creation of invariants directly, there are works in the literature that highlight the importance of multi-valued state variables for debugging domain descriptions and for assisting the domain designer in building correctly encoded domains \citep{FoxLong-98,Bernardini-11-iwpss,Cushing-07}. In particular, \cite{Cushing-07} analyse well-studied IPC temporal and numeric domains and reveal several modelling errors that affect such domains. This analysis lead the authors to suggest better ways of describing temporal domains. They identify as a central feature to do so the direct specification of multi-valued state variables and show how this can help domain experts to write correct models.

Other works in the literature use the creation of invariants and state variables as an intermediate step in the translation from PDDL to other languages. In particular, \cite{Huang-10} introduce SASE, a novel SAT encoding scheme based on the SAS+ formalism \cite{Helmert-09}. The state variables (extracted from invariants) used by SASE play a key role in achieving its efficiency. Since our technique generates a broader set of invariants than related techniques, it gives rise to SAS+ tasks with smaller sets of state variables. We speculate that this would in turn reverberate positively on SAT-based planners that use a SASE encoding. Testing of this hypothesis is part of our future work. 

\section{Conclusions and Future Work}
\label{sec:conclusions}
In this paper, we present a technique for automatically finding lifted mutual exclusion invariants in temporal planning domains expressed in PDDL2.1. Our technique builds on Helmert's invariant synthesis \citep{Helmert-09}, but generalised it and extends it to temporal domains. Synthesising invariants for temporal tasks is much more complex than for tasks with instantaneous actions only because actions can occur simultaneously or concurrently and interfere with each other. For this reason, a simple generalisation of Helmert's approach does not work in temporal settings. In extending the theory to capture the temporal case, we have had to formulate invariance conditions that take into account the entire structure of the actions as well as the possible interactions between them. As a result, we have constructed a technique that is significantly more comprehensive than the related ones. Our technique is presented here formally and proofs are offered that support its soundness.

Since our technique, differently from related approaches, works at the lifted level of the representation, it is very efficient. The experimental results show that its run time is negligible, while it allows us to find a wider set of invariants, which in turn results in synthesising a smaller number of state variables to represent a domain. The experiments also indicate that the temporal planners that use state variables to represent the world usually benefit from dealing with a smaller number of state variables. 

Our approach to finding invariants can be incorporated in any translation from PDDL2.1 to a language based on multi-valued state variables. For example, we have used (a simplified version of) the temporal invariant synthesis described in this paper in our translator from PDDL2.1 to NDDL, which is the domain specification language of the planner EUROPA2 \citep{Bernardini-08-trans}. EUROPA2 has been the core planning technology for several NASA space mission operations. It uses a language based on multi-valued state variables that departs from PDDL2.1 in several ways. The use of our translator from PDDL2.1 to NDDL has facilitated the testing of EUROPA2 against domains of the IPCs originally expressed in PDDL2.1 \citep{Bernardini-07,Bernardini-08}. This has originally motivated our work on temporal invariant synthesis.

In future work, we plan to extend our experimental evaluation by incorporating our invariant synthesis in other planners that use a multi-valued variable representation and that are not currently publicly available. This will allow us to assess more exhaustively the impact that handling fewer state variables with broader domains has on the performance of temporal planners. In addition, we plan to exploit the metric information encoded in planning domains to find a broader range of invariants. Invariants for domains with metric fluents are interesting and challenging. We envisage that there are two kinds of situations to be considered: those in which it can be shown that a linear combination of fluents is invariant (relevant to domains with linear effects on variables) and those in which metric fluents interact with propositional fluents in a more complex structure. For example, one might think of a domain encoding the act of juggling in which the number of balls in the air plus the number in the hands is a constant, but the balls in the hand might be encoded propositionally (for example, by a literal {\tt holding\_left} and so on), while those in the air as a count. Finding the invariant in this case is a challenging problem since it crosses the propositional and metric fluent spaces.

\newpage
\section*{Appendix A: PDDL2.1 Specification of the \sfs{Floortile} Domain}

\lstset{language=Lisp}         
\begin{lstlisting}[frame=single,basicstyle=\footnotesize]

(define (domain floor-tile)
 (:requirements :typing :durative-actions)
 (:types robot tile color - object)

(:predicates 	
  (robot-at ?r - robot ?x - tile)
  (up ?x - tile ?y - tile)
  (down ?x - tile ?y - tile)
  (right ?x - tile ?y - tile)
  (left ?x - tile ?y - tile)		
  (clear ?x - tile)
  (painted ?x - tile ?c - color)
  (robot-has ?r - robot ?c - color)
  (available-color ?c - color)
  (free-color ?r - robot))

(:durative-action change-color
  :parameters (?r - robot ?c - color ?c2 - color)
  :duration (= ?duration 5)
  :condition (and (at start (robot-has ?r ?c))
                  (over all (available-color ?c2)))
  :effect (and (at start (not (robot-has ?r ?c))) 
               (at end (robot-has ?r ?c2))))

(:durative-action paint-up
  :parameters (?r - robot ?y - tile ?x - tile ?c - color)
  :duration (= ?duration 2)
  :condition (and (over all (robot-has ?r ?c))
                  (at start (robot-at ?r ?x))
                  (over all (up ?y ?x))
                  (at start (clear ?y)))
  :effect (and (at start (not (clear ?y)))
               (at end (painted ?y ?c))))

(:durative-action paint-down
  :parameters (?r - robot ?y - tile ?x - tile ?c - color)
  :duration (= ?duration 2)
  :condition (and (over all (robot-has ?r ?c))
  		  (at start (robot-at ?r ?x))
		  (over all (down ?y ?x))
		  (at start (clear ?y)))
  :effect (and (at start (not (clear ?y)))
  	       (at end (painted ?y ?c))))

(:durative-action up 
  :parameters (?r - robot ?x - tile ?y - tile)
  :duration (= ?duration 3)
  :condition (and (at start (robot-at ?r ?x)) 
  		  (over all (up ?y ?x)) 
		  (at start (clear ?y)))
  :effect (and 
  	       (at start (not (robot-at ?r ?x)))
	       (at end (robot-at ?r ?y))
	       (at start (not (clear ?y)))
               (at end (clear ?x))))

(:durative-action down 
  :parameters (?r - robot ?x - tile ?y - tile)
  :duration (= ?duration 1)
  :condition (and (at start (robot-at ?r ?x))
  		  (over all (down ?y ?x)) 
		  (at start (clear ?y)))
  :effect (and (at start (not (robot-at ?r ?x)))
  	       (at end (robot-at ?r ?y))
	       (at start (not (clear ?y)))
               (at end (clear ?x))))

(:durative-action right 
  :parameters (?r - robot ?x - tile ?y - tile)
  :duration (= ?duration 1)
  :condition (and (at start (robot-at ?r ?x))
  		  (over all (right ?y ?x))
		  (at start (clear ?y)))
  :effect (and (at start (not (robot-at ?r ?x)))
  	       (at end (robot-at ?r ?y))
	       (at start (not (clear ?y)))
               (at end (clear ?x))))

(:durative-action left 
  :parameters (?r - robot ?x - tile ?y - tile)
  :duration (= ?duration 1)
  :condition (and (at start (robot-at ?r ?x)) 
  		  (over all (left ?y ?x)) 
		  (at start (clear ?y)))
  :effect (and (at start (not (robot-at ?r ?x)))
  	       (at end (robot-at ?r ?y))
	       (at start (not (clear ?y)))
               (at end (clear ?x))))
)
\end{lstlisting}

\section*{Appendix B: PDDL2.1 Specification of the \sfs{Depot} Domain}
\lstset{language=Lisp}         
\begin{lstlisting}[frame=single,basicstyle=\footnotesize]
(define (domain Depot)
(:requirements :typing :durative-actions)
(:types place locatable - object
	depot distributor - place
        truck hoist surface - locatable
        pallet crate - surface)

(:predicates (at ?x - locatable ?y - place) 
             (on ?x - crate ?y - surface)
             (in ?x - crate ?y - truck)
             (lifting ?x - hoist ?y - crate)
             (available ?x - hoist)
             (clear ?x - surface))
	
(:durative-action Drive
:parameters (?x - truck ?y - place ?z - place) 
:duration (= ?duration 10)
:condition (and (at start (at ?x ?y)))
:effect (and (at start (not (at ?x ?y))) (at end (at ?x ?z))))

(:durative-action Lift
:parameters (?x - hoist ?y - crate ?z - surface ?p - place)
:duration (= ?duration 1)
:condition (and (over all (at ?x ?p)) (at start (available ?x)) 
                (at start (at ?y ?p)) (at start (on ?y ?z)) 
                (at start (clear ?y)))
:effect (and (at start (not (at ?y ?p))) (at start (lifting ?x ?y)) 
             (at start (not (clear ?y)))(at start (not (available ?x))) 
             (at start (clear ?z)) (at start (not (on ?y ?z)))))

(:durative-action Drop 
:parameters (?x - hoist ?y - crate ?z - surface ?p - place)
:duration (= ?duration 1)
:condition (and (over all (at ?x ?p)) (over all (at ?z ?p)) 
		(over all (clear ?z)) (over all (lifting ?x ?y)))
:effect (and (at end (available ?x)) (at end (not (lifting ?x ?y))) 
             (at end (at ?y ?p)) (at end (not (clear ?z))) 
             (at end (clear ?y))(at end (on ?y ?z))))

(:durative-action Load
:parameters (?x - hoist ?y - crate ?z - truck ?p - place)
:duration (= ?duration 3)
:condition (and (over all (at ?x ?p)) (over all (at ?z ?p)) 
		(over all (lifting ?x ?y)))
:effect (and (at end (not (lifting ?x ?y))) (at end (in ?y ?z)) 
             (at end (available ?x))))

(:durative-action Unload 
:parameters (?x - hoist ?y - crate ?z - truck ?p - place)
:duration (= ?duration 4)
:condition (and (over all (at ?x ?p)) (over all (at ?z ?p)) 
                (at start (available ?x)) (at start (in ?y ?z)))
:effect (and (at start (not (in ?y ?z))) (at start (not (available ?x))) 
             (at start (lifting ?x ?y))))

)

\end{lstlisting}

\section*{Appendix C: Proofs}

\begin{proof}[Proof of Proposition \ref{prop: ser}]
The action $a_1$ is applicable in $s_0$ by definition. Assuming that $a_j$ is applicable in $s_{j-1}$ for $j=1, \ldots, k$, we now show that $a_{k+1}$ is applicable in $s_k$. Note that from the definition of transition function $\xi$ for single actions $s_k = (s \setminus \bigcup\limits_{j=1}^{k} Eff^-_{a_j}) \cup \bigcup\limits_{j=1}^{k}Eff^+_{a_j}$. Since $Pre^+_{a_{k+1}} \subseteq s$ and $Pre^-_{a_{k+1}}  \cap s = \emptyset$ by assumption and $a_{k+1}$ is not interfering with $a_1$, $a_2, \ldots, a_k$, we have that $Pre^+_{a_{k+1}} \subseteq s_k$ and $Pre^-_{a_{k+1}}  \cap s_k = \emptyset$. In addition, note that: $s_n = (s \setminus \bigcup\limits_{j=1}^{n} Eff^-_{a_j}) \cup \bigcup\limits_{j=1}^{n}Eff^+_{a_j} = \xi(s, A)$.
\end{proof}


\begin{proof}[Proof of Proposition \ref{prop: equiv safe}]
(ii)$\Rightarrow$(iii) is trivial and (iii)$\Rightarrow$(i) is an immediate consequence of (\ref{eq: state split}) and of the fact that $w(\cT,\gamma, s)=w(\cT,\gamma, s_{\gamma})$ and $w(\cT,\gamma, s')=w(\cT,\gamma, s'_{\gamma})$.

Finally, (i)$\Rightarrow$(ii) follows from the following argument. Given any $s\in \cS_{A_{\gamma}}$ such that $w(\cT,\gamma, s)\leq 1$, consider $s^*:=s_\gamma\cup Pre^+_{A_{\neg\gamma}}$. Since $s^*_{\gamma}=s_{\gamma}\in \cS_{A_{\gamma}}$ and 
$s^*_{\neg\gamma}=Pre^+_{A_{\neg\gamma}}\in \cS_{A_{\neg\gamma}}$, it follows that $s^*\in\cS_A$. If we consider the successor states $s' = \xi(s, A_{\gamma})$ and $s'^{*}= \xi(s^*, A_{\gamma})$, it follows from (\ref{eq: state split}) that 
$$s' _{\gamma}= \xi(s_{\gamma}, A_{\gamma})=\xi(s^*_{\gamma}, A_{\gamma})=s'^{*}_{\gamma}$$
Therefore, 
$$w(\cT,\gamma, s' )=w(\cT,\gamma, s' _{\gamma})=w(\cT,\gamma, s'^{*} _{\gamma})=w(\cT,\gamma, s'^{*} )\leq 1$$
where the last equality follows from the assumption that $A$ is strongly $\gamma$-safe. 
\end{proof}

\begin{proof}[Proof of Theorem \ref{prop:weight class}]

If $A$ is $\gamma$-unreachable and $A$ is applicable in the state $s$, It follows that $Pre^+_{A}\subseteq s$ and thus
$w(\cT, \gamma, s)\geq |Pre^+_{A_{\gamma}})|\geq 2$.
This shows that the condition $w(\cT, \gamma, s)\leq 1$ is never verified and thus $A$ is strongly $\gamma$-safe.

If $A$ is $\gamma$-irrelevant and $A$ is applicable in the state $s$, we have that 
the successor state $s' = \xi(s, a)\subseteq s$. This yields $w(\cT, \gamma, s')\leq w(\cT, \gamma, s)$. This implies that $A$ is strongly $\gamma$-safe.

Suppose $A$ is $\gamma$-heavy and consider the state
$s=Pre^+_{A}$. $A$ is applicable in $s$ and
$w(\cT, \gamma, s)=|Pre^+_{A_{\gamma}}|\leq 1$.
After applying $A$ in $s$,  the successor state $s' = \xi(s, A)$ is such that
$s'\supseteq Eff^+_{A_{\gamma}}$
This yields $w(\cT, \gamma, s')\geq |Eff^+_{\alpha}|\geq 2$ and proves that $A$ is not strongly $\gamma$-safe.
\end{proof}

\begin{proof}[Proof of Theorem \ref{prop:relevant actions}]
We will prove the corresponding property for $A_{\gamma}$ making use of Condition (iii) of Proposition \ref{prop: equiv safe}.

We first analyse the case when $A$ is balanced or unbalanced. 
Let $Pre^+_{A_{\gamma}}=\{q_1\}$ and $Eff^+_{A_{\gamma}}=\{q_2\}$.
Suppose now that $A$ is balanced and fix a state $s \in \gamma(\cT)$ such that $w(\cT, \gamma, s) \leq 1$ and $A_{\gamma}$ is applicable in $s$.  Clearly, $q_1 \subseteq s$ so, necessarily, $s=\{q_1\}$ and $w(\cT, \gamma, s) =1$. Consider the subsequent state $s' = \xi(s, A_{\gamma})$. If $q_1=q_2$, we have that $s'= s$ so that $w(\cT, \gamma, s') = 1$. If instead $q_1\in Eff^-_{A_{\gamma}}$, we have that $s'\subseteq (s\cup \{q_2\})\setminus \{q_1\}=\{q_2\}$ and thus $w(\cT, \gamma, s') =1$.

Suppose that $A$ is unbalanced and consider the state $s=\{q_1\}$. The subsequent state $s' = \xi(s, A_{\gamma})=\{q_1, q_2\}$ so that 
$w(\cT, \gamma, s') =2$.

We now consider the remaining two cases. Let  $Eff^+_{A_{\gamma}}=\{q_2\}$. Suppose now that $A$ is bounded and fix a state $s \in \gamma(\cT)$ such that $w(\cT, \gamma, s) \leq 1$ and $A_{\gamma}$ is applicable in $s$. Since $A$ is $\gamma$-relevant, the subsequent state $s' = \xi(s, A_{\gamma})$ is such that $w(\cT, \gamma, s')\leq  w(\cT, \gamma, s)+1$. The only case we need to consider is thus when $w(\cT, \gamma, s)=1$. Suppose that $s=\{q_1\}$. Since, by assumption $Pre_{A_{\gamma}}\cup Eff_{A_{\gamma}}=\cT(\gamma)$, it follows that $q_1\in Pre_{A_{\gamma}}\cup Eff_{A_{\gamma}}$. Clearly $q_1\not\in Pre^-_{A_{\gamma}}$ (otherwise $A_{\gamma}$ would not be applicable on the state $s$). Therefore, necessarily, either $q_1\in Eff^+_{A_{\gamma}}$ or $q_1\in  Eff^-_{A_{\gamma}}$. In the first case, we have that $q_1=q_2$ and thus $s'=s=\{q_1\}$. In the second case, $s'=\{q_2\}$. In both cases, $w(\cT, \gamma, s') =1$.

Finally, if $A$ is unbounded, we consider any ground atom $q_1\in \gamma(\cT)\setminus (Pre_{A_{\gamma}}\cup Eff_{A_{\gamma}})$ and we put $s=\{q_1\}$. Clearly $A_{\gamma}$ is applicable in $q_1$ since $Pre^+_{A_{\gamma}}=\emptyset$ and $q_1\not\in Pre^-_{A_{\gamma}}$, and $w(\cT, \gamma, s) =1$. Since it also holds that $q_1\not\in Eff^-_{A_{\gamma}}$, we have that the subsequent state $s' = \xi(s, A_{\gamma})=\{q_1, q_2\}$ and $w(\cT, \gamma, s') =2$.
\end{proof}

\begin{proof}[Proof of Proposition \ref{rem: strongly-safe-action}]
Write $A = \{a_1, \ldots, a_n\}$ and let $s$ be state such that $A$ is applicable in $s$. Note that from Proposition \ref{prop: ser}, the actions in $A$ can be serialised and the successor state $s'=\xi(s,A)$ can be recursively obtained as $s_0 = s$, $s_k = \xi(s_{k-1}, a_k)$, $k=2, \ldots, n$ and $s'=s_n$. By the assumption, it follows that $w(\cT, \gamma,s_i) \leq 1$ for every $i$. In particular, $w(\cT, \gamma,s') \leq 1$.
\end{proof}


\begin{proof}[Proof of Proposition \ref{prop: invariance safety}]
Given any instance $\gamma$ and any valid induced simple plan $\pi$ having $trace(\pi)=\{S_i=(t_i,s_i)_{i=0,\dots , \bar k}\}$ with happening sequence ${\bf A}_{\pi}$, we have that the state sequence $(s_0, \dots , s_{\bar k})\in {\bf S}_{{\bf A}_{\pi}}$. Therefore, since $w(\cT,\gamma, s_0)\leq 1$ (recall that $s_0=Init$ and  $w(\cT,\gamma, Init)\leq 1$ for every $\gamma$), the individual $\gamma$-safety of ${\bf A}_{\pi}$ implies that $w(\cT,\gamma, s_j)$ $\leq 1$ for every $j=1,\dots , \bar k$. Since this holds for every $\gamma$ and every valid plan, invariance of $\cT$ follows.
\end{proof}

\begin{proof}[Proof of Proposition \ref{prop: ind safe}]
(i): If $(s^0, s^1,\dots , s^n)\in{\bf S}_{{\bf A}}$, we have that 
$$(s^0, s^1,\dots , s^k)\in{\bf S}_{{\bf A}_1^k},\quad (s^{h-1}, s^1,\dots , s^n)\in{\bf S}_{{\bf A}_h^n}\,.$$ Therefore, if $w(\cT, \gamma, s^0)\leq 1$, from the fact that 
${\bf A}_1^k$ is individually $\gamma$-safe, it follows that $w(\cT, \gamma, s^j)\leq 1$ for every $j=1,\dots , k$. In particular, being $k\geq h-1$, we have that $w(\cT, \gamma, s^{h-1})\leq 1$. From the fact that ${\bf A}_h^n$ is also individually $\gamma$-safe, it now follows that $w(\cT, \gamma, s^j)\leq 1$ for every $j=h,\dots , n$. This implies that $w(\cT, \gamma, s^j)\leq 1$ for every $j=1,\dots , n$ and proves the thesis.

(ii): Suppose $(s^0, s^1,\dots ,s^{k-1}, s^{k+1}, \dots , s^n)\in{\bf S}_{{\bf A}'}$ where $s^{k+1}=\xi(A^k\cup A^{k+1}, s^{k-1})$. Put $s^k=\xi(A^k, s^{k-1})$ and note that, by serialisability (see Proposition \ref{prop: ser}), $s^{k+1}=\xi(A^{k+1}, s^{k})$, and therefore $(s^0, s^1,\dots ,s^{k-1}, s^k, s^{k+1}, \dots , s^n)\in{\bf S}_{{\bf A}}$. This implies that $w(\cT, \gamma, s^j)\leq 1$ for every $j=1,\dots , n$ and proves the thesis.

(iii): If $(s^0, s'^0, s^1,\dots , s'^{n-1}, s^n)\in{\bf S}_{{\bf A}'}$, then, $s'^{k-1}=s^k$ for every $k=1,\dots , n$ and $(s^0, s^1,\dots , s^n)\in{\bf S}_{{\bf A}}$. Individual $\gamma$-safety of $\bf A$ now yields the thesis.
Regarding ${\bf A}''$ thesis follows from the fact that ${\bf A}'$ is individually $\gamma$-safe and previous item (ii).
\end{proof}

\begin{proof}[Proof of Proposition \ref{prop: exec two}]

(i)$\Rightarrow$(ii): Note that if $(s^0, s^1, s^2)\in {\bf S}_{\bf A}$, it follows that $Pre^+_{A^1}\subseteq s^0$. Since $s^1=(s^0\setminus Eff^-_{A^1})\cup Eff^+_{A^1}$ it follows that $\Gamma^+_{A^1}\subseteq s^1$. Analogously, using the fact that $(Pre^-_{A^1})^c\supseteq s^0$, it follows that $(\Gamma^-_{A^1})\supseteq s^1$. Since $A^2$ must be applicable on $s^1$ conditions (ii) immediately follow.

(ii)$\Rightarrow$(i): Consider $s^0=Pre^+_{A^1}\cup (Pre^+_{A^2}\setminus Eff^+_{A^1})$. Straightforward set theoretic computation, using conditions (ii), show that $A^1$ can be applied on $s^0$ and that $A^2$ can be applied on $s^1=\xi(A^1, s^0)$. This proves (i).
\end{proof}

\begin{proof}[Proof of Proposition \ref{prop: reach two}]
(ii)$\Rightarrow$(i): It follows from the proof of (ii)$\Rightarrow$(i) in Proposition \ref{prop: exec two} that there exists $(s^0, s^1, s^2)\in {\bf S}_{\bf A}$ with $s^0=Pre^+_{A^1}\cup (Pre^+_{A^2}\setminus Eff^+_{A^1})$. 
By the assumption made $w(\cT,\gamma, s^0)\leq 1$ and this proves (i).

(i)$\Rightarrow$(ii): it follows from the fact that if $(s^0, s^1, s^2)\in {\bf S}_{\bf A}$, necessarily $Pre^+_{A^1}\cup (Pre^+_{A^2}\setminus Eff^+_{A^1})\subseteq s^0$.
\end{proof}

\begin{proof}[Proof of Proposition \ref{prop: exec reach}]
(i): If $(s^0, s^1,\dots , s^{n-1}, s^n)\in{\bf S}_{{\bf A}}$, we have that 
$(s^0, s^1,\dots , s^{n-1}, s'^n)\in{\bf S}_{{\bf A}'}$ for a suitable state $s'^n$. Result then follows from the definition of executability and $\gamma$-reachability.

(ii): This follows immediately from serialisability (see Proposition \ref{prop: ser}).
\end{proof}

\begin{proof}[Proof of Proposition \ref{prop: heavy presence}] Let $(s^0, \dots , s^n)\in {\bf S}_{\bf A}(\gamma)$ and suppose that $A^j$ is either $\gamma$-heavy or $\gamma$-relevant unbalanced. Then, necessarily, $w(\cT,\gamma, s^{j})\geq 2$.
\end{proof}

\begin{proof}[Proof of Proposition \ref{prop: safe split}]
1.: It follows from (\ref{eq: state split}) that, given any sequence of states $(s^0,\dots ,s^n)\in {\bf S}^{n+1}$, we have that
\begin{equation}\label{eq: sequence split}(s^0,\dots ,s^n)\in {\bf S}_{\bf A} \Leftrightarrow \left\{\begin{array}{rcl}(s^0_{\gamma},\dots ,s^n_{\gamma})&\in& {\bf S}_{{\bf A}_{\gamma}}\\
(s^0_{\neg\gamma},\dots ,s^n_{\neg\gamma})&\in& {\bf S}_{{\bf A}_{\neg\gamma}}\end{array}\right.\end{equation}
This immediately proves the 'only if' implication. On the other hand, if $s'\in {\bf S}_{{\bf A}_{\gamma}}$ and $s''\in {\bf S}_{{\bf A}_{\neg\gamma}}$, we have that $s'_{\gamma}\in {\bf S}_{{\bf A}_{\gamma}}$ and $s''_{\neg\gamma}\in {\bf S}_{{\bf A}_{\neg\gamma}}$ and thus $s=s'_{\gamma}\cup s''_{\neg\gamma}\in  {\bf S}_{\bf A}$ by (\ref{eq: sequence split}).

2. can be proven analogously to 1. and 3. follows by a straightforward extension of the arguments used to prove Proposition \ref{prop: equiv safe}.
Finally, 4. follows from the definition of strong $\gamma$-safety and previous items 1. and 2..
\end{proof}

\begin{proof}[Proof of Theorem \ref{theo: safe sequences}]
Consider the sequences restricted on the instantiation $\gamma(\cT)$ and its complement: ${\bf A}_{\gamma}$, ${\bf A}_{\neg \gamma}$ and, respectively, $\tilde{\bf A}_{\gamma}$, $\tilde{\bf A}_{\neg\gamma}$. 
By virtue of Proposition \ref{prop: safe split}, we have that ${\bf A}_{ \gamma}$ is $\gamma$-safe and to prove the result it is sufficient to show that $\tilde{\bf A}_{\gamma}$  is either non executable or $\gamma$-safe.

Assume that $\tilde{\bf A}_{\gamma}$ is executable and let $(s^0, s^1,s^2,\dots ,  s^{n+1}, s^{n+2})\in{\bf S}_{\tilde{\bf A}_{\gamma}}$ be such that $w(\cT, \gamma, s^0)\leq 1$. Since $(s^0, s^1)\in {\bf S}_{A^1_{\gamma}}$ and $A^1_\gamma$ is strongly safe, it follows that $w(\cT, \gamma, s^1)\leq 1$. Note now that $s^j=s^{j-1}\setminus Eff^-_{B^{j-1}_{\gamma}}$ for $j=2,\dots ,n+1$ and this immediately implies that 
$$w(\cT, \gamma, s^{n+1} )\leq w(\cT, \gamma, s^{n} )\leq \cdots \leq w(\cT, \gamma, s^{1} )\leq1$$
What remains to be shown is that also $w(\cT, \gamma, s^{n+2} )\leq 1$. To this aim, we introduce the following sets:
$$\Omega'=(\cup_{i=1}^n Eff^-_{B_{\gamma}^i})\cap Pre^-_{A_{\gamma}^2},\quad \Omega''=(\cup_{i=1}^n Eff^-_{B_{\gamma}^i})\setminus Pre^-_{A_{\gamma}^2}$$
Note that since ${\bf A}$ is executable, we have that $Pre^-_{A_{\gamma}^2}\cap Eff^+_{A^1}=\emptyset$. Consequently, also
$\Omega'\cap Eff^+_{A^1_{\gamma}}=\emptyset$. Therefore, $(s^0\setminus\Omega', s^1\setminus\Omega')\in {\bf S}_{A^1_{\gamma}}$. On the other hand, we also have that $A^2_{\gamma}$ is applicable on the state $s^{n+1}\cup \Omega ''$. This implies that there exists $\tilde s\in\cS$ such that $(s^{n+1}\cup \Omega '',\tilde s)\in {\bf S}_{A^2_{\gamma}}$. Note that $w(\cT,\gamma, \tilde s)\geq w(\cT,\gamma,  s^{n+2})$
Since $s^{n+1}\cup \Omega ''=(s^1\setminus
\cup_{i=1}^n Eff^-_{B_{\gamma}^i})\cup\Omega ''=s^1\setminus \Omega'$, we deduce that $(s^0\setminus\Omega', s^1\setminus\Omega', \tilde s)\in{\bf S}_{\bf A_{\gamma}}$. Since $w(\cT,\gamma, s^0\setminus \Omega ')\leq w(\cT,\gamma, s^0)\leq 1$, the fact that $\bf A_{\gamma}$ is $\gamma$-safe implies that $w(\cT,\gamma, \tilde s)\leq 1$. This also implies that $w(\cT,\gamma,  s^{n+2})\leq 1$ and the proof is complete.
\end{proof}

\begin{proof}[Proof of Proposition \ref{prop: aux durative 1}]
(i): Suppose $(s^0, s^1, s^2)\in {\bf S}_{(a^{st}, a^{inv})}$. Since $a^{inv}$ only contains preconditions, we have that $s^1=s^2$. Note now that $s^1=\xi(a^{st}, s^0)=(s^0\cup Eff^+_{a^{st}})\setminus Eff^-_{a^{st}}$ must satisfy the conditions $Pre^+_{a^{inv}}\subseteq s^1\subseteq (Pre^-_{a^{inv}})^c$.
This yields $Pre^+_{a^{inv}}\subseteq s^0\cup Eff^+_{a^{st}})$ and thus $Pre^+_{a^{inv}}\setminus Eff^+_{a^{st}} \subseteq s^0$. Similarly, from $s^0\setminus Eff^-_{a^{st}}\subseteq (Pre^-_{a^{inv}})^c$, we obtain that 
$s^0\subseteq (Pre^-_{a^{inv}}\setminus Eff^-_{a^{st}})^c$. This implies that also $a^{st}_*$ is applicable on $s^0$ and $s^1=\xi(a^{st}_*, s^0)$ since $a^{st}$ and $a^{st}_*$ have the same effects. 
If instead $(s^0, s^1)\in {\bf S}_{a^{st}_*}$, we have that $a^{st}$ is applicable on $s^0$ (since the preconditions of $a^{st}$ are also preconditions of $a^{st}_*$) and $s^1=\xi(a^{st}_*, s^0)=\xi(a^{st}, s^0)$.
(ii) is proven similarly to (i). (iii) follows from (i) and (ii) and, finally, (iv), (v), and (vi) follow, respectively, from (i), (ii), and (iii).
\end{proof}

\begin{proof}[Proof of Proposition \ref{prop: aux durative 2}]
Since ${\bf A}_*$ differs from ${\bf A}$ only for having more preconditions, 
it holds that ${\bf S}_{\bf A}\supseteq {\bf S}_{\bf A_*}$. Conversely, suppose $(s^0, \dots ,s^{n})\in {\bf S}_{\bf A}$. Then, $(s^0,s^1, s^1)\in {\bf S}_{(a^{st}, a^{inv})}$. Therefore, by (i) of Proposition \ref{prop: aux durative 1}, we have that $(s^0, s^1)\in {\bf S}_{a^{st}_*}$. Similarly, using (ii) of Proposition \ref{prop: aux durative 1}, we obtain that $(s^{n-1}, s^{n})\in {\bf S}_{a^{end}_*}$. These two facts together with  $(s^1, s^2,\dots , s^{n-1})\in {\bf S}_{(A^2,  \dots, A^{n-1})}$, yield $(s^0, \dots ,s^{n})\in {\bf S}_{\bf A_*}$.
\end{proof}

\begin{proof}[Proof of Proposition \ref{prop: aux simple safe}]
Note that $Da_*$, being $\gamma$-reachable and $a^{st}_*$ strongly $\gamma$-safe, is simply $\gamma$-safe if and only if $Da_{*_{\gamma}}$ is individually $\gamma$-safe. This last fact is equivalent to show that, given any state sequence $(s^0,s^1,s^2)\in {\bf S}_{Da_{*_{\gamma}}}$ such that $s^0\in\gamma(\cT)$ and $w(\cT, \gamma, s^0)\leq 1$, it holds that $w(\cT, \gamma, s^i)\leq 1$ for $i=2$ (since for $i=1$ follows from the strong safety of $a^{st}_*$).
Put
$$\cW_\gamma:=\{s^1\in\gamma(\cT)\,|\, \exists s^0, s^2\in\gamma(\cT), \, w(\cT, \gamma, s^0)\leq 1,\, (s^0, s^1, s^2)\in {\bf S}_{Da_{*\gamma}}\}$$
We need to show that, for every $s^1\in \cW_\gamma$, we have that $w(\cT, \gamma, s^2)\leq 1$, where 
$$s^2=\xi(a^{end}_{*\gamma}, s^1)=s^1\cup Eff^+_{a^{end}_{*\gamma}}\setminus Eff^-_{a^{end}_{*\gamma}}$$
Since $a^{end}_ {*}$ is $\gamma$-relevant unbounded, the condition $w(\cT, \gamma, s^2)\leq 1$ is clearly equivalent to
\begin{equation}\label{eq: safe condition}s^1\subseteq Eff^+_{a^{end}_{*\gamma}}\cup Eff^-_{a^{end}_{*\gamma}}\end{equation}

Since $a^{st}_*$ is $\gamma$-reachable and strongly $\gamma$-safe, it follows from Theorem \ref{prop:weight class} that is either $\gamma$-irrelevant or $\gamma$-relevant.
If $a^{st}_*$ is $\gamma$-irrelevant and $|Pre^+_{a^{st}_{*\gamma}}|=1$, we have that $\cW_\gamma=\{ Pre^+_{a^{st}_{*\gamma}}\setminus Eff^-_{a^{st}_{*\gamma}}\}$. Combining with (\ref{eq: safe condition}), we thus have that in this case $Da_*$ is $\gamma$-safe if and only if 
\begin{equation}\label{eq: safe condition 2}Pre^+_{a^{st}_{*\gamma}}\setminus Eff^+_{a^{st}_{*\gamma}}
\subseteq Eff^+_{a^{end}_{*\gamma}}\cup Eff^-_{a^{end}_{*\gamma}}\end{equation}
This leads to the two possible cases (a) and (b).

Suppose now that $a^{st}_*$ is $\gamma$-irrelevant and $|Pre^+_{a^{st}_{*\gamma}}|=0$. In this case,
$$\cW_\gamma=\{s^1\subseteq \gamma(\cT)\,|\,w(\cT,\gamma, s^1)\leq 1,\; s^1\cap (Pre^-_{a^{st}_{*\gamma}}\cup Eff^-_{a^{st}_{*\gamma}})=\emptyset\}$$
Combining with (\ref{eq: safe condition}), we thus have that in this case $Da_*$ is $\gamma$-safe if and only if  
\begin{equation}\label{eq: safe condition 3} Pre^-_{a^{st}_{*\gamma}}\cup Eff^-_{a^{st}_{*\gamma}}\cup Eff^+_{a^{end}_{*\gamma}}\cup Eff^-_{a^{end}_{*\gamma}}=\gamma(\cT)
\end{equation}
This leads to case (c).

Finally, if $\alpha^{st}_{*\gamma}$ is relevant we have that $\cW_\gamma=\{Eff^+_{a^{st}_{*\gamma}}\}$. Combining again with (\ref{eq: safe condition}), we obtain that in this case $Da_{*}$ is $\gamma$-safe if and only if condition (d) is verified.
\end{proof}

\begin{proof}[Proof of Proposition \ref{prop: start  * safe}] Since $Da$ is $\gamma$-reachable, it follows from Proposition \ref{prop: heavy presence}, that $a^{st}$ must necessarily be $\gamma$-relevant unbounded. In particular, this yields $Pre^+_{a_{\gamma}^{st}}=\emptyset$. Therefore, $Pre^+_{a^{st}_{*\gamma}}=Pre^+_{a_{\gamma}^{inv}}\setminus Eff^+_{a_{\gamma}^{st}}$ cannot have any intersection with $Eff_{a_{\gamma}^{st}}$. This says that $a_{*}^{st}$ cannot be $\gamma$-relevant balanced. Since it can neither be $\gamma$-unreachable (since $Da$ is $\gamma$-reachable), it follows from Corollary \ref{cor: strongly safe} that $a^{st}_*$ must be $\gamma$-relevant bounded. This proves (i).

Suppose now that the sequence $(\{a^{st}\}\cup A^1, a^{inv})$ is executable and let $q\in Eff^+_{A^1_{\gamma}}$. By (i), it follows that $q\in  Eff^+_{a^{st}_{\gamma}}\cup Eff^-_{a^{st}_{\gamma}}\cup Pre^-_{a^{st}_{*\gamma}}$. Note that $q$ cannot either belong to $Eff^-_{a^{st}_{\gamma}}$ or $Pre^-_{a^{st}_{\gamma}}$ since $a^{st}$ and the actions in $A^1$ must be non-interfering. On the other hand, $q$ cannot belong to $Pre^-_{a^{inv}_{\gamma}}$ otherwise the sequence would not be executable. Therefore the only possibility is that $q\in Eff^+_{a^{st}_{\gamma}}$. Therefore we have that $Eff^+_{A^1_{\gamma}}\subseteq Eff^+_{a^{st}_{\gamma}}$. Consider now $\tilde A^1$ the action set obtained from $A^1$ by eliminating all positive effects belonging to $\gamma(\cT)$. Clearly,  $\{a^{st}\}\cup A^1=\{a^{st}\}\cup \tilde A^1$. Consider now the sequence
$(\tilde A^1, a^{st}, a^{inv})$ and note that $\tilde A^1$ is $\gamma$-irrelevant, and $(a^{st}, a^{inv})$ is $\gamma$-individually safe because of (iv) of Proposition \ref{prop: aux durative 1}. Therefore, by Proposition \ref{prop: ind safe}, also $(\tilde A^1, a^{st}, a^{inv})$ is individually $\gamma$-safe, and thus also $(\{a^{st}\}\cup A^1, a^{inv})$.
\end{proof}


\begin{proof}[Proof of Theorem \ref{theo: *safety}]

Fix any valid simple plan $\pi$ 
with happening sequence ${\bf A}_{\pi}=(A_{t_0}, \dots , A_{t_{\bar k}})$ and any instance $\gamma$. We prove that ${\bf A}_{\pi}$ is individually $\gamma$-safe.

We split happenings as follows: $A_{t_i}=A_{t_i}^{st}\cup A_{t_i}^s\cup A_{t_i}^{end}$ where
\begin{itemize} 
\item $A_{t_i}^{st}$ is either empty or consists in the start fragments of durative actions in $\cG\cA^d(\gamma)$;
\item $A_{t_i}^{end}$ is either empty or consists in the ending fragments of durative actions in $\cG\cA^d(\gamma)$;
\item $A_{t_i}^s=A_t\setminus (A_{t_i}^{st}\cup A_{t_i}^{end})$ consists of strongly $\gamma$-safe actions (either instantaneous or possibly the starting and ending of durative ones in $\cG\cA^d\setminus \cG\cA^d(\gamma)$).
\end{itemize}
Note that if $A_{t_i}^{st}\neq \emptyset$, it either consists of all strongly $\gamma$-safe actions and is thus strongly $\gamma$-safe, or there exists a durative action $Da\in \cG\cA^d(\gamma)$ such that $a^{st}$ is not strongly safe and $a^{st}\in A_{t_i}^{st}$. 
Note that $A_{t_{i+1}}$ simply consists of $\{a^{inv}\}$ possibly together with other overall fragments of durative actions. Consequently, since $(A_{t_i}^{st}, A_{t_{i+1}})$ is executable, it is also executable $(A_{t_i}^{st}, a^{inv})$ (see (i) of Proposition \ref{prop: exec reach}).
By hypothesis, $Da$ is $\gamma$-reachable and $a^{st}_*$ is strongly $\gamma$-safe, we can thus apply Proposition \ref{prop: start  * safe} and conclude that $(A_{t_i}^{st}, a^{inv})$ is individually $\gamma$-safe. Using (iii) of Proposition \ref{prop: ind safe}, we obtain that $(A_{t_i}^{st}, A_{t_{i+1}})$ is individually $\gamma$-safe. Therefore, in any case, if $A_{t_i}^{st}\neq \emptyset$, $(A_{t_i}^{st}, A_{t_{i+1}})$ is individually $\gamma$-safe. 

Similarly, if $A_{t_i}^{end}\neq \emptyset$, it either consists of all strongly $\gamma$-safe actions and is thus strongly $\gamma$-safe, or there exists a durative action $Da\in \cG\cA^d(\gamma)$ such that $a^{end}$ is not strongly safe and $a^{end}\in A_{t_i}^{end}$. 
Suppose that it exists another durative action $Da'\in\cG\cA^d(\gamma)$ such that $a'^{end}\in A^{end}_{t_i}$ and $\{a^{end}, a'^{end}\}$ is $\gamma$-heavy. Then, since $\cG\cA^d$ is right relevant isolated and the two pairs $\{a^{inv}, a'^{inv}\}$, $\{a^{end}, a'^{end}\}$ are both non-interfering, the sequence $(\{a^{inv}, a'^{inv}\}, \{a^{end}, a'^{end}\})$ is $\gamma$-unreachable. Since $A_{t_{i-1}}$ only consists of actions with no effects, it then follows from Proposition \ref{prop: exec reach} that also the sequence $(A_{t_{i-1}}, A_{t_i}^{end})$ is $\gamma$-unreachable. The other possibility is that $Eff^+_{A^{end}_{t\gamma}}=Eff^+_{a^{end}_{\gamma}}$. 
Consider in this case $\tilde A^{end}_{t_i}$ to be the action set obtained from $A^{end}_{t_i}\setminus\{a^{end}\}$ by eliminating all positive effects belonging to $\gamma(\cT)$. Clearly,  $A_{t_i}^{end}=\{a^{end}\}\cup \tilde A^{end}_{t_i}$.
Note now that $(a^{inv}, a^{end})$ is individually $\gamma$-safe by (v) of Proposition \ref{prop: aux durative 1}. Considering that $A_{t_{i-1}}\setminus\{a^{inv}\}$ only contains preconditions and $\tilde A^{end}_{t_i}$ is strongly $\gamma$-safe, a repeated application of the different items of Proposition \ref{prop: ind safe} implies that $(A_{t_{i-1}}, A_{t_i}^{end})$ is individually $\gamma$-safe.

Note that, given each happening time $t_i$, there are four possibilities:
\begin{itemize}
\item $A_{t_i}^{st}=\emptyset$, $A_{t_i}^{end}=\emptyset$: in this case $A_{t_i}=A_{t_i}^s$ is strongly $\gamma$-safe by definition;
\item $A_{t_i}^{st}\neq \emptyset$, $A_{t_i}^{end}=\emptyset$: in this case, since $A_{t_i}^s$ and $(A_{t_i}^{st}, A_{t_{i+1}})$ are individually $\gamma$-safe, using (i) and (ii) of Proposition \ref{prop: ind safe}, we obtain that also $(A_{t_i}^s, A_{t_i}^{st}, A_{t_{i+1}})$ and $(A_{t_i}, A_{t_{i+1}})=(A_{t_i}^s\cup A_{t_i}^{st}, A_{t_{i+1}})$ are individually $\gamma$-safe.  
\item $A_{t_i}^{st}= \emptyset$, $A_{t_i}^{end}\neq\emptyset$: arguing analogously to the case above we obtain that $(A_{t_{i-1}}, A_{t_i})$ is individually $\gamma$-safe.
\item $A_{t_i}^{st}\neq\emptyset$, $A_{t_i}^{end}\neq\emptyset$: arguing analogously to the case above we obtain that $(A_{t_{i-1}}, A_{t_i}, A_{t_{i+1}})$ is individually $\gamma$-safe.
\end{itemize}
Using Corollary \ref{cor: ind safe} we obtain that ${\bf A}_{\pi}$ is individually $\gamma$-safe.
\end{proof}

\begin{proof}[Proof of Theorem \ref{theo: non intertwining} ] Fix any valid (possibly induced) simple plan $\pi$ 
with happening sequence ${\bf A}_{\pi}=(A_{t_0}, \dots , A_{t_{\bar k}})$ and any instance $\gamma$. We prove that ${\bf A}_{\pi}$ is individually $\gamma$-safe.

Suppose that we can prove that if $Da\in \cG\cA^d(\gamma)$ appears in $\pi$ on the time window $[t_h, t_k]$ (namely, $a^{st}\in A_{t_h}$ and 
$a^{end}\in A_{t_k}$), the corresponding action sequence ${\bf A}=(A_{t_h},\dots , A_{t_k})$ satisfies the following conditions:
\begin{itemize}
\item[(a)] for every $i\in (h, k)$, $A_{t_i}$ consists exclusively of $\gamma$-irrelevant actions;
\item[(b)] for every $i\in [h, k)$, $A_{t_i}$ does not contain actions in $\cG\cA^{st}(\gamma)$.
\end{itemize}
Note that if (b) holds true for every $Da\in \cG\cA^d(\gamma)$, we also have automatically that, 
\begin{itemize}
\item[(c)] for every $i\in (h, k]$, $A_{t_i}$ does not contain actions in $\cG\cA^{end}(\gamma)$.
\end{itemize}
Assuming this to hold, we now proceed as in the proof of Theorem \ref{theo: *safety} and we split each happening $A_{t_i}$ in the following way. We put $A_{t_i}=A_{t_i}^{st}\cup A_{t_i}^s\cup A_{t_i}^{end}$ where:
\begin{itemize} 
\item $A_{t_i}^{st}$ is either empty or consists in a start fragment in $\cG\cA^{st}(\gamma)$;
\item $A_{t_i}^{end}$ is either empty or consists in an ending fragment in $\cG\cA^{end}(\gamma)$;
\item $A_{t_i}^s=A_t\setminus (A_{t_i}^{st}\cup A_{t_i}^{end})$. 
\end{itemize}
We now consider the new plan $\tilde\pi$ given by
$$\tilde\pi=\{(t,a)\in\pi\,|\, a\in A_t^s\}\cup \{(t-\epsilon,a)\in\pi\,|\, a\in A_t^{end}\}\cup \{(t+\epsilon,a)\in\pi\,|\, a\in A_t^{st}\}$$
where $\epsilon >0$ is chosen in such a way that $\epsilon <t_{i+1}-t_i$ for every $i=0,\dots , \bar k-1$.
\begin{figure}[h]
\centering\includegraphics[width=80mm]{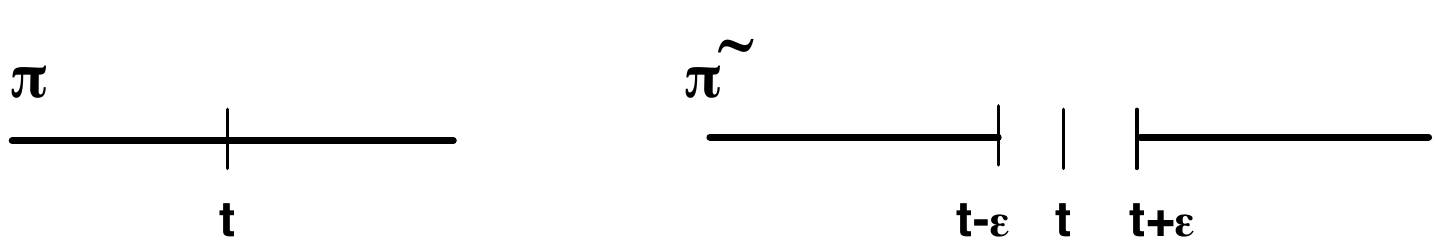}
\label{fig:plan}
\end{figure}

It follows from Proposition \ref{prop: ser} on serializability that plan $\tilde\pi$ is also valid. We denote its happening sequence as ${A}_{\tilde\pi}=(\tilde{A}_{t_0},\dots , \tilde{A}_{t_{\bar k}})$. For the sake of notation simplicity, happening times are denoted as those in $\pi$ even if in general they differ and form a larger set.
Note now that the happening times in $\tilde\pi$ can be split into singletons $ t_i$ such that $\tilde A_{t_i}$ only consists of strongly $\gamma$-safe actions, and intervals $[t_{i+1}, t_{j}]$ such that there exists a durative action $Da\in\cG\cA^d(\gamma)$ happening in that interval. In this case we have that the subsequence ${\bf A}=(\tilde A_{  t_{i+1}}, \dots \tilde A_{ t_{j}})$ is $Da$-admissible.
Put ${\bf A}_*=(\tilde A_{  t_{i}}\cup\{a^{inv}\}, \dots A_{ t_{j}}\cup\{a^{inv}\})$. Note that, since $\bf A$ is executable (as it appears in a valid plan), also $\bf A_*$ is executable by Proposition \ref{prop: aux durative 2}. Since, by assumption (ii), $Da_*$ is $\gamma$-safe, it follows from Theorem \ref{theo: safe sequences}, that  ${\bf A}_*$ is also $\gamma$-safe. Using again Proposition \ref{prop: aux durative 2} we finally obtain that $\bf A$ is individually $\gamma$-safe. 

We have thus proven that each happening time $ t_i$ in the new plan $\tilde\pi$ stays inside an individually $\gamma$-safe sequence (possibly of length $1$). By Corollary \ref{cor: ind safe}  this implies that $A_{\tilde\pi}$ is individually $\gamma$-safe. A repetitive use of (ii) of Proposition \ref{prop: ind safe} now yields that $A_{\pi}$ is also individually $\gamma$-safe.

We are thus left with proving that every durative action $Da\in\cG\cA^d(\gamma)$ happening in $\pi$ satisfies properties (a) and (b) stated above. Suppose this is not true and let $Da$ be the first (as starting time) to happen in $\pi$ (in the time window $[t_h, t_k]$) and to violate either condition (a) or (b). Note that all durative actions in $\cG\cA^d(\gamma)$ happening in $\pi$ and starting strictly before time $t_h$, will necessarily end at a time $t\leq t_h$ by the way $Da$ has been chosen. Moreover, all such durative actions will satisfy properties (a) and (b). We can then proceed as before and consider the splitting $A_{t_i}=A_{t_i}^{st}\cup A_{t_i}^s\cup A_{t_i}^{end}$ for every $i\leq h$ (note that in $t_h$ there could be, in principle, more than one starting actions in $A_{t_h}^{st}$). Consider now the auxiliary plan $\tilde\pi$ constructed exactly like before for $t\leq t_h$ and coinciding with $\pi$ for $t>t_h$. As before, we denote its happening sequence as ${A}_{\tilde\pi}=(\tilde{A}_{t_0},\dots , \tilde{A}_{t_{\bar k}})$ using the same notation for the happening times as in $\pi$ and we assume that $\tilde A_{t_h}=A^{st}_{t_h}$ (this is for simplicity of notation considering that it would be instead $\tilde A_{t_h+\epsilon}=A^{st}_{t_h}$).
Arguing as above, we obtain that $(\tilde A_{ t_0},\dots , \tilde A_{ t_{ h-1}})$ is individually $\gamma$-safe. If we take any $(\tilde s_0,\dots ,\tilde s_{\bar k})\in {\bf S}_{\tilde A_{\pi}}(\gamma)$, we thus have that 
$w(\cT, \gamma, \tilde s_{\tilde h-1})\leq 1$. Consider now ${\bf A}=(\tilde A_{ t_{ h}}, \tilde A_{t_{h}+1}\dots , \tilde A_{ t_{ k}})=(A_{t_h}^{st}, A_{t_{h+1}}, \dots A_{t_k})$
and note that, $(\tilde s_{  h-1}, \dots , \tilde s_{k})\in{\bf S}_{\tilde{\bf A}}(\gamma)$ so that $\tilde\bf A$ is $\gamma$-reachable. It then follows from the relevant non intertwining property ((i) of Definition \ref{def: non intertwining}) that $A^{st}_{t_{ h}}=\{a^{st}\}$. Suppose now that property (b) stated above is not satisfied and let $  l\in (  h,   k)$ be the first index such that $ A_{ t_{ l}}\cap \cG\cA^{st}(\gamma)\neq\emptyset$. By (ii) of Definition \ref{def: non intertwining}, it follows there must exists a durative action $Da'\in\cG\cA^d(\gamma)$ such that $a'^{end} \in  A_{ t_{ l'}}$ for some $ l'\in (h, l]$. Note that such durative action cannot start, in the plan $\pi$ and thus also in the plan $\tilde\pi$, before time $ t_{ h}$ for the way $Da$ was chosen, it cannot either start at time $ t_{ h}$ by previous considerations and neither in the interval $( t_{ h},  t_{l'})$ by the way $ l'$ has been chosen. This proves that property (b) must be satisfied. Note that this also shows that $ A_{ t_{ i}}$ does not contain actions in $\cG\cA^{end}(\gamma)$ for any $ i\in ( h, k)$ (as the corresponding start fragment cannot happen neither before or after time $ t_{ h}$). Finally, $A_{t_i}$ is $\gamma$-irrelevant for every $i\in (h,k)$ because
of (iii) of Definition \ref{def: non intertwining}. Therefore $\bf A$ satisfies properties (a) and (b) contrarily to the assumptions made on $Da$. Proof is now complete.
\end{proof}

\begin{proof}[Proof of Proposition \ref{prop: relevant non inter}]
Consider $Da\in\cG\cA^d(\gamma)$ and a $\gamma$-reachable $Da$-admissible sequence $${\bf A}=(\{a^{st}\}\cup A^1, A^2, \dots , A^{n-1}, \{a^{end}\}\cup A^{n})\,.$$ 

If $a'^{st}\in \cG\cA^{st}(\gamma)\cap A^1$, necessarily, $a^{st}$ and $a'^{st}$ are non-interfering, and, by Proposition \ref{prop: exec reach}, the sequence $(\{a^{st}, a'^{st}\}, \{a^{inv}, a'^{inv}\})$ is $\gamma$-reachable contradicting assumption (i). Therefore $\cG\cA^{st}(\gamma)\cap A^1=\emptyset$. This proves (i) in Definition \ref{def: non intertwining}. 

Suppose now that $A^1=\emptyset$ and suppose that (ii) in Definition \ref{def: non intertwining} does not hold true for $\bf A$. Let $j>1$ be the first index for which (ii) is violated. Let $a'^{st}\in A^j\cap \cG\cA^{st}(\gamma)$. Since $(\{a^{st}\}, A^2, \dots , A^{j-1}, \{a'^{st}\})$ is $\gamma$-reachable and for sure the pair $\{a^{inv}, a'^{st}\}$ is non-interfering, it follows from assumption (ii) that there must exist $0<j'<j$ such that $A^{j'}$ is not $\gamma$-irrelevant. Let $j'$ be the first index for which this happens and let $b\in A^{j'}$ be an action which is not $\gamma$-irrelevant . Note that $b\not\in \cG\cA^{end}(\gamma)$ (otherwise $a'^{st}$ would not violate (ii)). This however contradicts assumption (ii). Therefore this proves (ii) in Definition \ref{def: non intertwining}.

Suppose now that $A^1=\emptyset$ and $A^j\cap (\cG\cA^{st}(\gamma)\cup \cG\cA^{end}(\gamma))=\emptyset$ for every $j=2,\dots ,{n-1}$. If (iii) in Definition \ref{def: non intertwining} does not hold true for $\bf A$, consider $j>1$ to be the first index for which (iii) is violated, namely $A^j$ is not $\gamma$-irrelevant, and let $b\in A^j$ be any action which is not $\gamma$-irrelevant. Since $(\{a^{st}\}, A^2, \dots , A^{j-1}, \{b\})$ is $\gamma$-reachable, it follows from assumption (ii) that there must exist $0<j'<j$ such that $A^{j'}$ is $\gamma$-relevant but this contradicts the choice of $j$. Proof is thus complete.
\end{proof}

\begin{proof}[Proof of Proposition \ref{prop: irr unr}] Assume that, by contradiction, there exists a $\gamma$-reachable sequence
$${\bf A}=(\{a\}, A^2, \dots , A^{n-1}, \{a'\})$$
such that $A^2,\dots ,A^{n-1}$ are $\gamma$-irrelevant set of actions. Consider $(s^0,\dots , s^n)\in {\bf S}_{\bf A}(\gamma)$.

Suppose condition (i) is satisfied. Clearly, $q\in s^1$ and, because of the assumption made, it follows that $q\not\in Eff^-_{A^j}$ for every $j=2,\dots , n-1$. Therefore, $q\in s^{n-1}$. Since $q\in Pre_{a'}^-$, this is a contradiction.

A similar arguments can be used if instead condition (ii) is satisfied.

Finally, assume that condition (iii) is satisfied. Note that, since $A^2,\dots ,A^{n-1}$ are $\gamma$-irrelevant, $Pre^+_{a_{\gamma}}\cup (Pre^+_{a'_{\gamma}}\setminus Eff^+_{a_\gamma})\subseteq s^0$ and this contradicts the fact that $(s^0,\dots , s^n)\in {\bf S}_{\bf A}(\gamma)$.
\end{proof}

\begin{proof}[Proof of Corollary \ref{cor: type a}] It is clear that condition (i) and (ii) of Theorem \ref{theo: non intertwining} are satisfied. In order to check that $\cG\cA^d$ is relevant non intertwining, we 
show that the properties (i) and (ii) of Proposition \ref{prop: relevant non inter} are satisfied. Fix any instance $\gamma$.

Consider $Da^1, Da^2\in\cG\cA^d(\gamma)$. It follows from the fact that $Da^1_*$ and  $Da^2_*$ are both simply $\gamma$-safe of type (a) (see Remark \ref{rem: type a}) that
$$Pre^+_{a^{i\,st}_{\gamma}}=\{q^i\}\subseteq Eff^-_{a^{i\,st}_{\gamma}},\;i=1,2$$
If $a^{1\,st}$ and $a^{2\,st}$ are non-interfering, it follows that $q^1\neq q^2$ and, in this case, $\{a^{1\,st}, a^{2\,st}\}$ is $\gamma$-unreachable. This proves (i).

Consider now $Da\in\cG\cA^d(\gamma)$ and $a'\in \cG\cA\setminus \cG\cA^{end}(\gamma)$ that is $\gamma$-relevant or $a'\in \cG\cA^{st}(\gamma)$. Then, by assumptions (i) and (ii) we have that 
$$Pre^+_{a^{st}_{\gamma}}=\{q\}\subseteq Eff^-_{a^{st}_{\gamma}},\; Pre^+_{a'_{\gamma}}=\{q'\}$$
If $q=q'$, we have that $q\in \Gamma_{a^{st}}^-\cap Pre_{a'}^+$ and, since $q\in\gamma(\cT)$, for sure $q\not\in Eff^+_{a''}$ for any $a''$ which is $\gamma$-irrelevant. This implies that condition (ii) of Definition \ref{def: strongly del inter} is satisfied. If instead $q\neq q'$, we have that the condition (iii) is instead satisfied. In any case this says that the pair $(a^{st},a')$ is strongly $\gamma$-irrelevant unreachable and thus also, because of Proposition \ref{prop: irr unr}, $\gamma$-irrelevant unreachable.
\end{proof}


\begin{proof}[Proof of Lemma \ref{lemma:g-i}]
We use the formalism introduced in Remarks \ref{rem:matching} and \ref{rem:coherent}.
Assume that, for $i=1,2$, $l_i$ matches $\cT$ via the component $c_i$ whose corresponding relation has the form  $r_i(x^i_1,\dots x^i_k, v)$ so that $l_i=r(a_1^i,\dots ,a_k^i, a_{k+1}^i)$ or 
$l_i=\forall v:\; r(a_1^i,\dots ,a_k^i, v)$ for free arguments $a_1^i,\dots ,a_k^i, a_{k+1}^i$. 
We have that
\begin{equation}\label{eq coh1}gr(a^1_j)=\gamma(x^1_j)=\gamma(x^2_l),\;\forall j=1,\dots ,k\end{equation}
where the first equality follows from the assumption of coherence over $l_1$, while the second follows from the definition of an instance. 
Now, if $l_2\sim_{\cT}l_1$, we have that $a_j^1=a_j^2$ for every $j$. It thus follows from (\ref{eq coh1}) that 
\begin{equation}\label{eq coh2}gr(a^2_j)=\gamma(x^2_l),\;\forall j=1,\dots ,k\end{equation}
which says that  $gr$ and $\gamma$ are coherent over $l_2$. On the other hand, if $l_2\not\sim_{\cT}l_1$, it follows that $a_j^1\neq a_j^2$ for some $j$ and, since $gr$ is injective, we also have $gr(a^2_j)\neq gr(a^1_j)=\gamma(x^2_l)$ which says that $gr$ and $\gamma$ are not coherent over $l_2$.
\end{proof}

\begin{proof}[Proof of Corollary \ref{cor: strongly safe final}] Suppose that $a=gr(\alpha)$ for some $gr$ and let $\gamma$ be an instance. Then, $a_{\gamma}=gr(\alpha_L)$ where $L$ is the $\cT$-class on which $gr$ and $\gamma$ are coherent. Result is now a straightforward consequence of Proposition \ref{prop: pure schema} and Corollary \ref{cor: strongly safe}.

\end{proof}


\begin{proof}[Proof of Proposition \ref{prop: unreachable schemas} ]
Note first of all that each of the conditions (i), (ii), (iii) expressed in Definition \ref{schemas non int}, if true for $\cM=\cM_{L^1, L^2}$ is also true for any matching $\cM\supseteq \cM_{L^1, L^2}$: this is evident for properties (i) and (ii) (see Remark \ref{rem monotonic}) while for (iii) follows from the following argument. Condition (iii), for $\cM= \cM_{L^1, L^2}$, holds true if, either, $w(Pre^+_{\alpha^{i\,inv}_{L^i\,}}\cup Pre^+_{\alpha^{1end}_{L^i}})\geq 2$ for $i=1$ or $2$ (and this does not depend on $\cM$), or if there exist two unquantified literals $l^i\in Pre^+_{\alpha^{i\,inv}_{L^i\,}}\cup Pre^+_{\alpha^{1end}_{L^i}}$ for $i=1,2$ such that $l^1\neq_{\cM_{L^1, L^2}} l^2$. This implies that, necessarily, $l^1, l^2$ match $\cT$ through components $c^1=<r^1, a^1, p^1>$ and $c^2=<r^2, a^2, p^2>$ with $r^1\neq r^2$. This yields $Rel[l^1]\neq Rel[l^2]$ and, as a consequence, $l^1\neq_{\cM} l^2$ with respect to any possible matching $\cM$.

Let $\cM$ be the matching such that $gr^1$ and $gr^2$ are $\cM$-adapted. It follows from Proposition \ref{prop: equiv adapted} that $\cM_{L^1, L^2}\subseteq \cM$. Consequently we know that at least one of the conditions
(i), (ii), (iii) expressed in Definition \ref{schemas non int} holds true for such $\cM$. It then follows from (\ref{eq:gr}) that at least one of the following conditions hold
\begin{enumerate}[(ib)] 
\item $Pre^+_{a^{1inv}}\cap Pre^-_{a^{2end}}\neq\emptyset$;
\item $Pre^-_{a^{1inv}}\cap Pre^+_{a^{2end}}\neq\emptyset$;
\item $|Pre^+_{a^{1inv}_{\gamma}}\cup Pre^+_{a^{1end}_{\gamma}}\cup Pre^+_{a^{2inv}_{\gamma}} \cup Pre^+_{a^{2end}_{\gamma}}|\geq 2$.
\end{enumerate}
By virtue of Propositions \ref{prop: exec two} and \ref{prop: reach two} this implies that $(\{a^{1inv}, a^{2inv}\}, \{a^{1end}, a^{2end}\})$ is $\gamma$-unreachable.
\end{proof}

\begin{proof}[Proof of Proposition \ref{prop: relevant right isolated schemas}]
Fix any instance $\gamma$ and consider $Da^1, Da^2\in\cG\cA^d(\gamma)$. Let $D\alpha^i$ and $gr^i$, for $i=1,2$, durative schemas and groundings such that $Da^i=gr(D\alpha^i)$. Let $L^i$ be the $\cT$-class  of literals of each schema $D\alpha^i$ such that $gr^i$ and $\gamma$ are coherent over $L^i$ for $i=1,2$. Therefore, $D\alpha^1$ and $D\alpha^2$ must satisfy one of the conditions (i) to (iii) in the Definition \ref{def: relevant right isolated schema}. Let $\cM$ be the matching respect to which $gr^1$ and $gr^2$ are adapted (in the sense of Remark \ref{rem: ground adapted}). We know from Proposition \ref{prop: equiv adapted} that $\cM\supseteq \cM_{L^1,L2}$. Note now that if condition (i) holds true, it also holds true for such larger $\cM$ (Remark \ref{rem monotonic}) and this yields condition (i)
of Definition \ref{def: relevant right isolated}. Similarly, condition (ii) yields the same condition with this new $\cM$ (Remark \ref{rem monotonic}) from which condition (ii) in Definition \ref{def: relevant right isolated schema} 
follows using Proposition \ref{prop: mutex1}. Finally, if condition (iii) holds true, then condition (iii) in Definition \ref{def: relevant right isolated schema} follows by using Proposition \ref{prop: unreachable schemas}.
Therefore, by Definition \ref{def: relevant right isolated schema} we have that the two durative action schemas $D\alpha^1$ and $D\alpha^2$ must satisfy one of the conditions (i) to (iii) in the definition. From the fact that $gr^1$ and $gr^2$ are $\cM$-adapted, it follows that conditions (i) of Definition \ref{def: relevant right isolated schema} yields condition (i)
of Definition \ref{def: relevant right isolated}. Condition (ii) and (iii) in Definition \ref{def: relevant right isolated} finally follow conditions (ii) and (iii) in Definition \ref{def: relevant right isolated schema} using Propositions \ref{prop: mutex1} and \ref{prop: unreachable schemas}.

\end{proof}

\begin{proof}[Proof of Proposition \ref{prop: relevant right isolated schemas} ]
Let $\cM$ be the matching respect to which the two groundings $gr^1$ and $gr^2$ are $\cM$-adapted. By Proposition \ref{prop: equiv adapted}, we have that $\cM\supseteq \cM_{L^1, L^2}$. Arguing like in the proof of Proposition \ref{prop: unreachable schemas} we obtain that one of the conditions (i)) to (iii) of Definition \ref{def: strongly irr unreach schemas} must hold true for such a matching $\cM$. Suppose (i) holds and put $q=gr^1(l^1)=gr^2(l^2)\in \Gamma_a^{1+}\cap Pre_{a}^{2-}$. Consider now any ground action $a$ which is $\gamma$-irrelevant and let $\alpha$ be an action schema such that $a=gr(\alpha)$ for some grounding $gr$. Let $L$ be the $\cT$-class of literals of $\alpha$ on which $gr$ and $\gamma$ are coherent. It follows that $\alpha_{L}$ is irrelevant. Consider now the matching $\tilde \cM$ between $\alpha^1$ and $\alpha$ respect to which $gr^1$ and $gr$ are $\tilde\cM$-adapted. We have that $\tilde\cM\supseteq \cM_{L^1, L}$. Then, by (i) we have that $l^1\not\in_{\tilde\cM} Eff^+_{\alpha}$ which implies that $q\not\in Eff^-_{a}$. This shows that condition (i) of Definition \ref{def: strongly del inter} is satisfied. Similarly, one can prove that condition (ii) of Definition \ref{def: strongly irr unreach schemas} yields 
condition (ii) of Definition \ref{def: strongly del inter}. Finally the fact that (iii) of Definition \ref{def: strongly irr unreach schemas} yields 
condition (iii) of Definition \ref{def: strongly del inter} follows from a repeated application of relation (\ref{eq:gr}).
\end{proof}

\begin{proof}[Proof of Proposition \ref{prop: relevant non inter lifted}]
We simply have to prove that the assumptions of Proposition \ref{prop: relevant non inter} are satisfied. Note that (i) implies (i) of Proposition \ref{prop: relevant non inter} because of Proposition \ref{prop: relevant left isolated schemas}. We now prove (ii) of Proposition \ref{prop: relevant non inter}.
 
To this aim, fix an instance $\gamma$ and $a^{1st}\in \cG\cA^{st}(\gamma)$ and $a^2\in \cG\cA\setminus \cG\cA^{end}(\gamma)$ that is not $\gamma$-irrelevant. Let $D\alpha^1\in\cA^d$ and $\alpha^2\in\cA$, $gr^1$ and $gr^2$ be action schemas and groundings such that $gr^1(D\alpha^1)=Da^1$ and $a^2=gr^2(\alpha^2)$. Let $L^i$, for $i=1,2$, be $\cT$-classes of literals of $D\alpha^1$ and $\alpha^2$, respectively, such that $gr^i$ and $\gamma$ are coherent over $L^i$.
We have that $\alpha^{1st}_{L^1}\in \cG\cA^{st}(\cT)$, $\alpha^2_{L^2} \not\in \cG\cA^{end}(\cT)$ and is not irrelevant. By assumption (ii) it then follows that $\alpha^{1st}, \alpha^2$ is strongly $(L^1, L^2)$-irrelevant unreachable and thus, by Proposition \ref{prop: relevant right isolated schemas}, $(a^{1st},a^2)$ is strongly $\gamma$-irrelevant unreachable. In the case when instead $a^2\in \cG\cA^{st}(\gamma)$ proof is analogous.
\end{proof}

\section*{Acknowledgements}
We thank Malte Helmert, Gabriele Roger and Jussi Rintanen for making their code for synthesising invariants available, William Cushing for helpful discussions about the configurations of temporal actions and Maria Fox and Derek Long for insightful discussions on the semantics of PDDL2.1. This work has been supported by Royal Holloway University of London, Politecnico di Torino and the NASA Exploration Systems Program.

\section*{References}

\bibliographystyle{elsarticle-harv}
\bibliography{biblio}

\end{document}